\documentclass[letterpaper,11pt]{article}

\usepackage{microtype}
\usepackage{booktabs} 
\usepackage{graphicx}
\usepackage[font=small]{caption}
\usepackage{subcaption}
\usepackage{amsfonts,amsbsy,amssymb,amsmath,amsthm}
\usepackage{paralist,booktabs}
\usepackage{nicefrac,microtype,mathtools} 
\usepackage{thmtools,thm-restate}
\usepackage{array}
\usepackage{url} 
\usepackage{physics}
\usepackage{thm-restate}
\usepackage{arydshln}
\usepackage[margin=1in]{geometry}

\usepackage{natbib}

\DeclareMathOperator{\diag}{\text{diag}}
\DeclarePairedDelimiter{\innerprod}{\langle}{\rangle}
\DeclarePairedDelimiter{\crl}{\{}{\}}
\DeclarePairedDelimiter{\prn}{(}{)}

\DeclarePairedDelimiter{\brk}{[}{]}

\usepackage{hyperref}

\usepackage{xcolor}

\renewcommand{\hat}{\widehat}

\newcommand{\conj}[1]{\overline{#1}}
\newcommand{\Dim}{D}
\newcommand{\KerSize}{K}
\newcommand{\CIn}{R}
\newcommand{\COut}{C}
\newcommand{\cin}{r}
\newcommand{\cout}{c}
\newcommand{\SmallDim}{P}

\newcommand{\InputMatrix}{\Vector{X}}
\newcommand{\InputVec}{\Vector{x}}
\newcommand{\ParVec}{\Vector{w}}
\newcommand{\ParMatrix}{\Matrix{W}}
\newcommand{\ParVecFn}{w}
\newcommand{\ParMatrixFn}{{W}}
\newcommand{\Fourier}{\Matrix{F}}
\newcommand{\WgtOneMatrix}{\Matrix{U}}
\newcommand{\WgtOneTensor}{\Matrix{\mathcal{U}}}
\newcommand{\WgtTwoMatrix}{\Matrix{V}}

\newcommand{\WgtTwoVec}{\Vector{v}}

\newcommand{\DualRealCin}[1]{\Vector{\lambda}^{\text{real}}_{#1}}
\newcommand{\DualImgCin}[1]{\Vector{\lambda}^{\text{img}}_{#1}}

\newcommand{\SDPVar}{\Matrix{Z}}
\newcommand{\SDPConsVar}{\Matrix{Q}}
\newcommand{\DualVec}{\Vector{\lambda}}

\newcommand{\R}{{\mathcal{R}}}

\newcommand{\Rk}[2]{{\R_{#2}(#1)}}
\newcommand{\RkCout}[3]{{\R_{#2, #3}(#1)}}
\newcommand{\RkCoutHat}[3]{{\hat{\R}_{#2, #3}(#1)}}
\newcommand{\RkCoutOp}[2]{{\R_{#1, #2}}}
\newcommand{\RkCoutDim}[4]{{\R^{(#1)}_{#3, #4 }(#2)}}
\newcommand{\RkCinCout}[4]{{\R_{#2,#3,#4}(#1)}}
\newcommand{\RkCinCoutOp}[3]{{\R_{#1,#2,#3}}}
\newcommand{\RkCinCoutHat}[4]{{\hat{\R}_{#2, #3, #4}(#1)}}

\newcommand{\RSDPk}[2]{{\R^\textrm{SDP}_{#2}(#1)}}
\newcommand{\RSDPkCin}[3]{{\R^\textrm{SDP}_{#2,#3}(#1)}}
\newcommand{\RSDPDualk}[2]{{\R^\textrm{SDP}_{#2}(#1)}}

\newcommand{\Entrywise}{\odot}
\newcommand{\Conv}{\star}
\newcommand{\Vector}[1]{{\boldsymbol{\mathbf{#1}}}}
\newcommand{\Matrix}[1]{{\boldsymbol{\mathbf{#1}}}}
\newcommand{\Index}[2]{{#1{[{#2}]}}}
\newcommand{\e}{\mathrm{e}}
\newcommand{\Identity}{\Matrix{I}}

\newcommand{\floor}[1]{\left\lfloor #1 \right\rfloor}    
\newcommand{\ceil}[1]{\left\lceil #1 \right\rceil}

\newcommand{\ie}{\textit{i.e., }}
\newcommand{\eg}{\textit{e.g., }}
\newcommand{\st}{\text{ s.t., }}
\newcommand{\mybox}{\hfill\(\Box\)}

\newcommand{\bR}{\mathbb{R}}
\newcommand{\bC}{\mathbb{C}}
\renewcommand{\b}[1]{\Vector{#1}}

\newcommand{\X}{\InputMatrix}
\newcommand{\x}{\InputVec}
\newcommand{\w}{\ParVec}
\newcommand{\W}{\ParMatrix}
\newcommand{\U}{\WgtOneMatrix}
\newcommand{\V}{\WgtTwoMatrix}
\newcommand{\UU}{\WgtOneTensor}

\newcommand{\vv}{\WgtTwoVec}
\newcommand{\F}{\Fourier}

\newcommand{\Q}{\SDPConsVar}

\newcommand{\Pattern}{\Vector{p}}

\newcommand{\ba}{\Vector{a}}
\newcommand{\bb}{\Vector{b}}
\newcommand{\bc}{\Vector{c}}
\newcommand{\ConvD}[1]{\,\Conv_{\scriptscriptstyle #1}\,}

\newtheorem{theorem}{Theorem}

\newtheorem{conjecture}[theorem]{Conjecture}

\newtheorem{lemma}[theorem]{Lemma}
\newtheorem{proposition}[theorem]{Proposition}
\newtheorem{definition}{Definition}
\newtheorem{remark}{Remark}
\newtheorem{hypothesis}{Hypothesis}
\newtheorem{example}{Example}

\newtheorem*{thm*}{Theorem}

\newtheorem*{rem*}{Remark}
\newtheorem*{clm*}{Claim}

\newtheorem{fact}{Fact}

\usepackage{authblk}

\definecolor{mydarkblue}{rgb}{0,0.08,0.45}
\hypersetup{ %
  colorlinks=true,
  linkcolor=mydarkblue,
  citecolor=mydarkblue,
  filecolor=mydarkblue,
  urlcolor=mydarkblue}
\setcitestyle{authoryear,round,citesep={;},aysep={,},yysep={;}}

\title{Inductive Bias of Multi-Channel Linear Convolutional Networks with Bounded Weight Norm}
\author[1]{Meena Jagadeesan\thanks{mjagadeesan@berkeley.edu. This work was partly done while M. Jagadeesan and I. Razenshteyn were at Microsoft Research. M. Jagadeesan was supported in part by the Paul and Daisy Soros Fellowship.}}
\author[2]{Ilya Razenshteyn\thanks{ilya.razenshteyn@gmail.com}}
\author[3]{Suriya Gunasekar\thanks{suriyag@microsoft.com}}
\affil[1]{University of California, Berkeley}
\affil[2]{CipherMode Labs}
\affil[3]{Microsoft Research}
\date{\vspace{-5ex}}

\begin{document}

\maketitle

\begin{abstract}%

We provide a function space characterization of the inductive bias resulting from minimizing the $\ell_2$ norm of the weights in multi-channel convolutional neural networks with linear activations and empirically test our resulting hypothesis on ReLU  networks trained using gradient descent. We define an \textit{induced regularizer} in the function space as the minimum $\ell_2$ norm of weights of a network required to realize a function.  For two layer linear convolutional networks with $C$ output channels and kernel size $K$, we show the following: (a) If the inputs to the network are single channeled, the induced regularizer for any $K$ is \textit{independent} of the number of output channels $C$. Furthermore, we derive the regularizer is a norm given by a semidefinite program (SDP). (b) In contrast, for multi-channel inputs, multiple output channels can be necessary to merely realize all matrix-valued linear functions and thus the inductive bias \emph{does} depend on $C$. However, for sufficiently large $C$, the induced regularizer is again given by an SDP that is independent of $C$. In particular, the induced regularizer for  $K=1$ and $K=D$ (input dimension) is given in closed form as the nuclear norm and the $\ell_{2,1}$ group-sparse norm, respectively, of the Fourier coefficients of the linear predictor.
We investigate the broader applicability of our theoretical results to implicit regularization from gradient descent on linear and ReLU networks through experiments on MNIST and CIFAR-10 datasets. 
\end{abstract}

\section{Introduction}\label{sec:introduction}

In the study of generalization and model capacity, complexity measures based on magnitude of parameters have long been argued to play an important role in learning overparametrized models \citep{B97,bartlett2002rademacher,NTS15,ZBHRV17,bartlett2017spectrally}. In particular, the $\ell_2$ norm of weights (or parameters) is a prominent complexity measure of interest in the current practice of deep learning, with connections to explicit regularization \citep{krogh1991simple,wei2019regularization} as well as implicit regularization from optimization algorithms \citep{ji2019gradient,gunasekar2018characterizing} (see additional discussion in Section~\ref{sec:relatedwork}). 
Importantly, recent results \citep{LL20,nacson2019lexicographic,ji2020directional} show that in many (not all) instances of overparametrized classification problems, gradient descent asymptotically leads to solutions that implicitly control the $\ell_2$ norm of the parameters (see Section~\ref{sec:experiments} for a formal statement). 

We study the question: \textit{What is the nature of functions learned by controlling $\ell_2$ norm of parameters?} Consider a model class of functions (or network architecture) $\Phi(\b{\theta};.)$ with parameters (or weights)  $\b{\theta}$. The function space view of controlling $\ell_2$ norm of parameters (denoted as $\|\b{\theta}\|$) can be understood in terms of its \textit{representation cost}, \ie the minimum $\ell_2$ norm of weights needed to realize a function using a given network architecture $\Phi(\b{\theta},.)$. This defines an \textit{induced complexity measure}  over  functions, which we also refer as induced regularizer, given by
\begin{equation}\label{eq:repcost-functions}
    \R_\Phi(f) := \inf_{\Vector{\theta}} \;\norm{\Vector{\theta}}^2 \;\st\; \forall \x,  f(\x) = \Phi(\b{\theta},\x).
\end{equation}
Note that a learning objective with $\ell_2$ norm regularization of the parameters $\min_{\b{\theta}}\mathcal{L}(\Phi(\b{\theta},.))+\lambda \|\Vector{\theta}\|^2$ is equivalent to the corresponding $\R$-regularization over functions $\min_{f}\mathcal{L}(f)+\lambda \R_\Phi(f)$.\footnote{ $\R_\Phi$ can be equivalently defined as any monotonic function of $\norm{\Vector{\theta}}$. We use $\norm{\Vector{\theta}}^2$ to align with the standard  regularizer (see also, \citep{GLSS18, SESS19, OWSS20, DKS21}).}

Even for neural networks that realize the same function class, minimizing or bounding the $\ell_2$ norm of weights in different architectures can lead to remarkably different effects in function space. 
For example, consider networks with fully connected and convolution layers and linear activations. These architectures are simply different parameterizations of the same model class of linear functions. \citet{GLSS18} showed that for fully connected linear networks, the induced regularizer is the $\ell_2$ norm of the linear map realized by the network, while for  linear convolutional network with \textit{full dimensional kernels}, it is the $\ell_1$ norm of Fourier coefficients of the linear map.  This function space view reveals that minimizing the $\ell_2$ norm of weights in these networks has fundamentally different implications for learned predictors depending on the  parametrization of function class.

\subsection{Our contribution} 
In this work, we investigate the induced regularizer in \eqref{eq:repcost-functions} for multi-channel linear convolutional networks. In particular, we study two layer networks that have $\COut$ output channels, $\CIn$ input channels, and kernel size $\KerSize$. We characterize the role of the number of channels $\COut$ on the induced regularizer, for networks with arbitrary kernel size $\KerSize$. Our main contribution is that for inputs with a single channel, having multiple output channels  in the network surprisingly does not reduce the $\ell_2$-norm representational cost, despite increasing the number of parameters in the network.
\begin{theorem}[Informal]
\label{thm:maininformal}
For two layer convolutional networks with any kernel size $\KerSize$, if the inputs are single-channeled, then the induced regularizer is independent of the output channel size $\COut$.
\end{theorem}

\paragraph{Proof technique.} To prove Theorem \ref{thm:maininformal}, we construct an \textit{semidefinite program (SDP) relaxation} which corresponds to the induced regularizer when $\COut = \infty$. We then prove that this SDP relaxation is in fact tight \emph{for all} $C\ge1$, which leads to our main result. The SDP further implies a convex structure of the induced regularizer for any $\COut$. In our proof of SDP tightness, we use a polynomial representation of convolutions to  implicitly argue the existence of a rank-$1$ optimal solution. A key lemma in our proof (Lemma~\ref{lemma:additiveproperty}) shows an interesting property about  convolutions in $\bR^D$ with kernel of size $\KerSize<\Dim$. To our knowledge, this property as well as the proof technique involving polynomial representations are new and are of independent interest.

\paragraph{Extension to multi-channel inputs.}  We further extend our findings to networks with multi-channel inputs. For multi-channel inputs of dimensions $\Dim\times \CIn$, even realizing all linear functions over the inputs can require multiple output channels $\COut$ (see Lemma~\ref{lemma:realizing}). Hence, the induced regularizer \textit{does} depend on $\COut$, although for large enough $\COut$, we show a restricted form of invariance. In particular, we prove the induced regularizer is invariant to the number of output channels when $\COut \ge \CIn \cdot \KerSize$, and conjecture invariance when $\COut \ge \CIn$. 
    We then characterize the induced regularizer in the special cases of $\KerSize=1$ and $\KerSize=\Dim$ as the nuclear norm  and the $\ell_{2,1}$ group sparse norm of the Fourier coefficients, respectively (see Theorems~\ref{thm:multichannelkersize1}-\ref{thm:multichannelkersizeD}). 

\paragraph{Experiments for gradient descent.} Finally, we connect our results to the implicit regularization of gradient descent. When combined with prior work (\eg \cite{LL20}), our results also extend to asymptotic predictors learned by gradient descent on networks with ReLU or linear activations.  We thus formulate and study an empirically testable hypothesis that the $\ell_2$ norm complexity of networks learned using gradient descent is invariant to $\COut$ as long as $\COut \ge \CIn$. 
We validate this hypothesis on MNIST and CIFAR-10 datasets on linear convolutional neural networks with circular and zero padding. The behavior also holds  on MNIST in networks with ReLU non-linearity. 

\subsection{Related Work}\label{sec:relatedwork}

There is a rich literature of work connecting $\ell_2$ norm minimization of weights with explicit regularization \citep{krogh1991simple, wei2019regularization} and implicit regularization from  gradient descent \citep{NTS15, ZBHRV17, bartlett2017spectrally,gunasekar2018characterizing,GLSS18,ji2018risk,ji2019gradient,nacson2019lexicographic,LL20, ji2020directional}. While  implicit regularization from gradient descent trajectory is not always connected to $\ell_2$ norm for regression (see counterexamples in \cite{DFKL20,RC20, LLL21}), the connection is  prominent in many settings of interest. Most relevant to our work is the result by \citet{LL20} (stated in Section~\ref{sec:experiments}), showing that parameters learned using gradient descent on logistic loss asymptotically converge in the direction of max-$\ell_2$-margin solution. We combine our results with this prior work to demonstrate that the conclusions from our analysis also extend to gradient descent solutions in classification problems. 

Motivated by the connections to implicit and explicit regularization and generalization, other prior work also studies induced regularizers corresponding to minimizing $\ell_2$ of weights in different architectures.  Among  recent work, \citet{SESS19, OWSS20} provided a characterization of induced regularizer for infinite width two layer ReLU neural networks on 1D and higher dimensional inputs respectively. \citet{ZBHMS20}  empirically demonstrated such differences arising from fully connected versus convolutional architectures. In a work closely related to ours, \citet{GLSS18} characterized the induced regularizer for fully connected networks and for linear convolutional network with \textit{full dimensional kernels} ($\KerSize=\Dim$) and \textit{single-channel networks} ($\CIn=\COut=1$). \citet{YKM20} extended \citet{GLSS18} and showed a general connection for linear networks between implicit $\ell_1$ norm minimization in an orthonormal basis and the existence of  data-independent diagonalizations of the linear operator in each layer.  However, even within the class of two-layer linear convolutional networks, the conclusions in prior work \citep{GLSS18,YKM20} do not generalize to nontrivial kernel and channel sizes. For example, in the other extreme kernel size of $\KerSize=1$, the induced regularizer is in fact the $\ell_2$ norm of the linear function which is fundamentally different from the $\ell_1$ norm of the Fourier coefficients for $\KerSize=\Dim$, thus emphasizing the importance of our analysis for multi-channel networks with arbitrary kernel sizes. Subsequent to our work appearing as a preprint, \citet{DKS21} also studied the induced regularizer for convolutional neural networks, though they do not investigate the role of the number of output channels for general kernel sizes.

Lastly, in a complementary approach, a line of work \citep{PE20, EP202, EP21, SEPP20} studies the induced regularizer of neural networks including convolutional networks by looking at the bi-dual convex relaxation of the $\ell_2$ regularized least squares loss. In the context of linear convolutional networks, our results are significantly stronger, as their analysis shows invariance to number of output channels only in the limit of large $\COut$, while we show independence for all $\COut \ge 1$. We elaborate on this comparison more in Section \ref{sec:comparison}.

\subsection{Notation}
We typeface vectors, matrices, and tensors using bold characters, \eg $\vv,\x,\b{\theta},\W, \UU$.  We will use zero-based indexing with notation $[\Dim]=\crl{0,1,\ldots,\Dim-1}$, and python style slicing notation to specify the sub-entries of an array variable: \eg given $\b{Z}\in\bR^{D_1\times D_2}$,  the $d_1^\text{th}$ row and $d_2^\text{th}$ column are  denoted as $\Index{\b{Z}}{d_1,:}\in\bR^{d_2}$ and  $\Index{\b{Z}}{:,d_2}\in\bR^{d_1}$, respectively. 
Complex numbers are specified in the polar form as ${z}=\abs*{z}\e^{i\phi_{{z}}}$ with $\phi_{{z}}\in[0,2\pi)$; or in Cartesian  form as $z=\Re(z)+i\Im(z)$ (ref. $i=\sqrt{-1}$ is the imaginary unit). The complex conjugate is denoted as $\conj{{z}}=\abs*{z}\e^{-i\phi_{{z}}}$.  For $\b{a},\b{b}\in\bC^{D}$, the standard inner product is  $\innerprod{\b{a},\b{b}}=\b{a}^\top\conj{\b{b}}$, and analogously extends to matrices. 

We use $\|.\|$ to denote the standard Euclidean norm, \ie $\ell_2$ norm of entries. For arrays $\b{a},\b{b}$, $\b{a}\odot\b{b}$ denotes entry-wise multiplication and $\Vector{a}\propto\b{b}$ implies proportionality up to positive scaling. Finally, we define the convolution operator $\Conv$ as it is  used in the neural networks literature.\footnote{In signal processing,  $\star$ is known as the cross-correlation operator. To simplify analysis, we use circular padding in the definition (where $\text{mod}$ refers to the modulo operator, \ie  $p\text{ mod }D=p-D\floor{\frac{p}{D}}$). For convolutions with zero-padding, there will be different edge effects, but we expect qualitatively similar behavior for small padding sizes. We also use a  scaling of ${1}/{\sqrt{\Dim}}$--this is merely to simplify notation and does not change the analysis.}
\begin{definition}[Circular convolution] \label{def:conv}
For $\Vector{u} \in \mathbb{R}^{\KerSize}$ and $\Vector{v} \in \mathbb{R}^{\Dim}$ with $\KerSize \le \Dim$, their $\Dim$ dimensional circular convolution, 
denoted by $\b{u}\Conv \b{v}$, is a vector in $\bR^{\Dim}$ given as follows: 
\begin{equation*}
    \forall_{d\in[\Dim]}, \Index{(\b{u}\Conv\b{v})}{d}= \frac{1}{\sqrt{\Dim}} \sum_{k=0}^{\KerSize - 1} \Index{\b{u}}{k} \Index{\b{v}}{(d+k) \text{ mod }D}.
\end{equation*}
\end{definition}

\subsection{Multi-channel linear convolutional network}\label{sec:background} We consider two layer linear convolutional networks with multiple channels in the convolution layer. We first focus on multi-output channel convolutions with single channel inputs  described below. We will discuss networks with multi-channel inputs (\eg RGB color channels) in Section~\ref{sec:mult-input-channel}.

The inputs to the network are vectors\footnote{For simplicity we consider $1$D vectors $\x\in\bR^D$ as inputs, but all our results can  be extended to $2$D inputs  $\x\in\bR^{W\times H}$, such as images, with the  corresponding $2$D convolutional operator.} of dimension $\Dim$ denoted as $\x\in\bR^\Dim$. The first layer is a convolutional layer with kernel size $\KerSize$   and number of output channel  $\COut$ whose weights (parameters) are denoted by $\U \in \mathbb{R}^{\KerSize\times\COut}$. The output of the convolution layer, denoted as $h(\U;\x)\in\bR^{\Dim\times \COut}$, is given by $\;\Index{h(\U;\x)}{:,\cout}=\Index{\U}{:,\cout}\Conv\x$ for all $c\in[\COut]$. The second layer is a single output linear layer with weights $\V\in \mathbb{R}^{\Dim \times \COut}$. Thus, the output of the network, denoted as $\Phi(\U,\!\V;\x)$, is given by: 
\begin{equation}\label{eq:nn}\Phi(\U,\!\V;\x)\!=\!\innerprod*{\V,h(\U;\x)}\!=\!\!  \sum_{\cout=0}^{\COut-1} \innerprod*{\Index{\V}{:,\!\cout},\!\Index{\U}{:,\!\cout}\Conv\x }.
\end{equation} 
Since, the network described above does not have any non-linearity, the output function $\Phi(\U,\!\V;.)$ is equivalent to a linear representation $\ParVecFn(\U,\V)\in\bR^{\Dim}$ such that $\forall\x$, $\Phi(\U,\!\V;\x)=\innerprod*{\ParVecFn(\U,\V),\x}$. Using standard algebraic manipulations on \eqref{eq:nn}, one can derive $\ParVecFn{(\U,\V)}$ as follows:
\begin{equation}\label{eq:w-uv}
    \ParVecFn{(\U,\!\V)}= \sum_{\cout=0}^{\COut-1} \prn*{\Index{\U}{:,\cout}\Conv\Index{\V}{:,\cout}^{\downarrow}}^\downarrow,
\end{equation}
where $\Vector{z}^{\downarrow}$ denotes the flipped vector of $\Vector{z}\in\bR^D$, given by $\Index{\Vector{z}^{\downarrow}}{d} = \Index{\Vector{z}}{\Dim - d - 1}$ for $d=0,1,\ldots \Dim - 1$. 
\begin{rem*}\label{remark:linear}
Even for the smallest network in this class with $\KerSize=\COut=1$, any linear predictor $\w\in\bR^\Dim$ can be realized as $\ParVecFn{(\U,\V)}$ in eq.~\eqref{eq:w-uv} (\eg using  $\U=1,\V=\ParVec$). In fact, every linear predictor can be represented by multiple networks with different weights $\U,\V$. 
\end{rem*}

\paragraph{Fourier representation.}
The convolution operation in Definition~\ref{def:conv} permits a simple form in the Fourier domain arising from the \textit{convolution theorem}. 
Let $\F\in \mathbb{C}^{\Dim \times \Dim}$ denote the {unitary discrete Fourier transform (DFT) matrix} for $\bR^D$, \ie $\forall_{k,l\in[D]}$, $\Index{\F}{k,l} =\frac{1}{\sqrt{\Dim}} \e^{\frac{-2\pi i k l}{\Dim}}$  and $\F\conj{\F}^\top=\conj{\F}^\top\F=I$. For any $1 \le \KerSize \le \Dim$, let $\F_{\KerSize}\in \bC^{\Dim\times\KerSize}$ denote the submatrix of $\F$ with the first $\KerSize$ columns. For a vector $\Vector{a} \in \bR^{\KerSize}$, we denote its $\Dim$ dimensional  Fourier  representation as $\hat{\Vector{a}}=\F_{\KerSize} \Vector{a} \in \bC^{\Dim}$. From the definition of Fourier transform, we have that $\F\Vector{a}^\downarrow=\conj{\F}\Vector{a}=\conj{\hat{\Vector{a}}}$ and we can derive the \textit{convolution theorem} for our operator $\Conv$ (Definition~\ref{def:conv}) as $\F(\Vector{a} \Conv \Vector{b}) = \conj{\hat{\mathbf{a}}} \Entrywise \hat{\Vector{b}}$.

Let $\hat{\ParVecFn}{({\U},{\V})} :=\F{\ParVecFn{(\U,\V)}}$ denote the Fourier transform of the linear predictor $\ParVecFn{(\U,\V)}$ realized by our network (see eq.~\ref{eq:w-uv}). We can now express $\hat{w}$ as follows: Let $\hat{\U}=\F_{\KerSize}\U \in \mathbb{C}^{\Dim\times \COut}$ and $\hat{\V}=\F\V \in \mathbb{C}^{\Dim\times \COut}$ denote  the $\Dim$ dimensional Fourier representation of   $\U,\V$,  respectively. We have  
\begin{equation}
\label{eq:fourierlinearpredictor}
        \hat{\ParVecFn}{({\U},{\V})} = \sum_{\cout=0}^{\COut-1} \hat{\U}[:,\cout] \Entrywise \hat{\V}[:,\cout]= \diag\prn{\hat{\U}\hat{\V}^\top}.
\end{equation}

\section{Induced regularizer in the function space} \label{sec:explicitbounds} 
For the network $\Phi$ described above, we now turn to the function space view of controlling the $\ell_2$ norm of the weights $(\U,\V)$. Recall that this inductive bias is captured by the \textit{induced regularizer} or the function space representation cost  \eqref{eq:repcost-functions}. For our linear convolutional network $\Phi$, the function class realized   is exactly the set of linear predictors $\w\in\bR^\Dim$ and the induced regularizer is given by: 
\begin{equation}\label{eq:inducedregularizer}
\Rk{\ParVec}{\KerSize,\COut} := \;\min\limits_{\U \in \mathbb{R}^{\KerSize\times \COut}, \V \in \mathbb{R}^{\Dim\times \COut}}\; \norm{\U}^2 + \norm{\V}^2 
\quad \st\quad\ParVecFn{(\U,\V)} = \ParVec.
\end{equation}

\begin{remark}
\label{remark:kdecrease}
It immediately follows from  \eqref{eq:inducedregularizer}  that  $\RkCout{\ParVec}{\KerSize}{\COut}$ is weakly decreasing in both $\KerSize$ and $\COut$, \ie $\forall_{\COut}$, $\RkCout{\ParVec}{1}{\COut}\ge \RkCout{\ParVec}{2}{\COut}\ge\ldots \RkCout{\ParVec}{\Dim}{\COut}$ and $\forall_{\KerSize}$, $\RkCout{\ParVec}{\KerSize}{1}\ge \RkCout{\ParVec}{\KerSize}{2}\ge\ldots$.
\end{remark}

Even within the class of two layer linear convolutional networks, the induced regularizer can  exhibit strikingly different properties for different choices of $\KerSize$ and $\COut$. For example, we recall the following result from \cite{GLSS18} that for full dimensional kernel $\KerSize=\Dim$, the induced regularizer is equal to the $\ell_1$ norm of the Fourier transform of the predictor.
\begin{lemma}[$\KerSize = D$] {\citep[Lemma~7 in][]{GLSS18}}\label{lemma:kersizeD}
For any $\ParVec \in \mathbb{R}^{\Dim}$,  $\RkCout{\ParVec}{\Dim}{1} =2\norm{\hat{\ParVec}}_1$.
\end{lemma}
\noindent On the other hand, we have the following characterization for $K=1$ (full proof is in Appendix~\ref{appendix:proofsexplicitbounds}). 
\begin{restatable}[$\KerSize = 1$]{lem}{explicitone}
\label{lemma:kersize1}
For any  $\ParVec \in \mathbb{R}^{\Dim}$, $\RkCout{\ParVec}{1}{1} =2 \sqrt{\Dim} \norm{\hat{\ParVec}}_2 = 2 \sqrt{\Dim} \norm{\ParVec}_2$. 
\end{restatable}
The induced regularizer thus behaves fundamentally differently for $\KerSize = \Dim$ and $\KerSize = 1$. In particular, the $\ell_2$ regularization of $\RkCout{\ParVec}{1}{1}$ is basis agnostic and does not induce sparse solutions, while the $\ell_1$ regularization of $\RkCout{\ParVec}{\Dim}{1}$ promotes sparsity in the Fourier basis.

Since $\KerSize = 1$ and $\KerSize = \Dim$ permit closed-form solutions  in the Fourier space, one might hope to obtain similarly clean characterizations for other kernel sizes as well. However, neither the proof technique for Lemma \ref{lemma:kersizeD} nor the proof technique for Lemma \ref{lemma:kersize1} extend to the case of general kernel sizes. The proof of Lemma \ref{lemma:kersizeD} uses the fact that for $\KerSize=\Dim$, the weights $\U,\V\in\bR^\Dim$ are unconstrained in Fourier space. For networks with smaller kernels, the argument breaks as  $\hat{\U}=\F_\KerSize\U$ is constrained to be in a $\KerSize<\Dim$ dimensional space spanned by the columns of $\F_\KerSize$. The proof of Lemma \ref{lemma:kersize1} again uses the special structure for $\KerSize = 1$ that $\hat{\U}=\F_\KerSize\U$ points in the direction of $[1, \ldots, 1]$, which does not extend to larger kernel sizes. 

In fact, we show that even for $\KerSize = 2$, the induced regularizer $\RkCout{\ParVec}{2}{1}$ takes a much more complex form. In particular,  the characterization in Fourier space involves a maximization over  a high-degree rational function, and is thus unlikely to admit clean closed-form solutions.
\begin{restatable}{lem}{explicittwo}
\label{lemma:kersize2}
For any $\ParVec \in \mathbb{R}^{\Dim}$,  it holds that:
\[\RkCout{\ParVec}{2}{1} = 2 \sqrt{\Dim}  \sqrt{\inf_{\alpha \in (-1,1)}\sum_{d=0}^{\Dim - 1} \frac{\abs*{\Index{\hat{\ParVec}}{d}}^2}{1 + \alpha \cos\left(2 \pi d / \Dim \right)}}.\]
\end{restatable}
Although Lemma~\ref{lemma:kersize2} does not yield closed form solutions  for $\RkCout{\w}{2}{1}$, we observe that it hints at some form of band-pass frequency structure: for any $\alpha$  from the inner optimization, the resulting regularizer is a weighted sum of Fourier coefficients such that the nearby frequency components of $\hat{\w}$ are weighted with nearby values. This band-pass nature was also observed in a complementary result by \citet{YKM20} in the context of implicit bias from gradient descent on a single data point: for any $\x$, it was shown that  $\min_\w \RkCout{\ParVec}{2}{1}\st \w^\top\x>1$ corresponds to a low-pass or high pass filter depending on the sign of  $\x^\top\x^{\downarrow}$.

Even though we do not obtain closed form solutions for all kernel sizes $\KerSize$, we derive important properties about the induced regularizer for general kernel sizes in the following sections that also generalize the above results to networks with multiple output channels.

\section{Main technical tool: SDP formulation of induced regularizer}\label{sec:SDP}
To investigate the induced regularizer for general kernel sizes and channel sizes, we construct a semidefinite program (SDP) relaxation, which is a key tool in our analysis. In this section, we describe and analyze this SDP formulation for multi-output channel networks on inputs with a single channel. (We discuss generalizations to the case of multi-channel inputs in Section \ref{sec:mult-input-channel}.)

We first reformulate $\RkCoutOp{\KerSize}{\COut}$ as an SDP with a rank constraint, which immediately motivates an SDP relaxation that provides a lower bound on $\RkCoutOp{\KerSize}{\COut}$. As we will show in Theorem \ref{thm:tightness}, this SDP relaxation is actually tight for all $K$ and $C$, which enables us to deduce a number of interesting properties of the induced regularizer. 

\paragraph{\texorpdfstring{$\boldsymbol{\RkCoutOp{\KerSize}{\COut}}$}{R\_{K,C}} as an SDP with a rank constraint.} Combining  the definition of $\Rk{\ParVec}{\KerSize,\COut}$ in \eqref{eq:inducedregularizer} with the Fourier representation of $\ParVecFn{(\U,\V)}$ in \eqref{eq:fourierlinearpredictor}, we have the following: 
\begin{equation}
\label{eq:mixedrepresentation}
\Rk{\ParVec}{\KerSize,\COut} = \;\min\limits_{\U \in \mathbb{R}^{\KerSize\times \COut}, \V \in \mathbb{R}^{\Dim\times \COut}} \;\norm{\U}^2 + \norm{\V}^2 \quad\st \quad\diag(\hat{\U} \hat{\V}^{\top}) = \hat{\ParVec}.
\end{equation}
The optimization in \eqref{eq:mixedrepresentation} over $\U\in\bR^{\KerSize\times \COut},\V\in\bR^{\Dim\times \COut}$, can be specified in terms of a rank $\COut$ positive semi-definite matrix  $\SDPVar \in \mathbb{R}^{(\Dim + \KerSize)\times (\Dim+\KerSize)}$ that we define below: 
\begin{equation}
    \SDPVar=\left[\!\!\begin{array}{c} \Matrix{U} \\\Matrix{V} \end{array}\!\!\right] \!\!\!\!\begin{array}{c}[\!\!\begin{array}{cc} \Matrix{U}^{\top} & \Matrix{V}^{\top} \end{array}\!\!]\\\;\end{array}= \left[\!\!\begin{array}{cc} \Matrix{U}\Matrix{U}^{\top} & \Matrix{U}\Matrix{V}^{\top} \\ \Matrix{V} \Matrix{U}^{\top} & \Matrix{V}\Matrix{V}^{\top} \end{array}\!\!\right]\succcurlyeq 0.
\end{equation} 
We refer to $\SDPVar$ as the lifted space of parameters $\Matrix{U},\Matrix{V}$. In the lifted space, the objective and constraints of \eqref{eq:mixedrepresentation} can now be expressed as linear functions of $\SDPVar$. The objective is given by $\norm{\U}^2+\norm{\V}^2=\langle \U\U^{\top}, \Identity \rangle + \langle \V\V^{\top}, \Identity \rangle=\langle \SDPVar, \Identity \rangle$. The constraints of \eqref{eq:mixedrepresentation} are give by $\forall_{d\in[\Dim]}$, $\langle \hat{\U} \hat{\V}^{\top},\Vector{e}_d \Vector{e}_d^{\top} \rangle = \Index{\hat{\ParVec}}{d}$, where $\crl*{\b{e}_d}_{d\in[\Dim]}$ denotes the standard basis. 
Alternatively, using $\hat{\U}=\F_K\U$ and $\hat{\V}=\F\V$, the constraints are given by $\forall_{d\in[\Dim]}, \langle \U \V^{\top}, \conj{\F}_K^{\top} \Vector{e}_d \Vector{e}_d^{\top} \conj{\F}\rangle = \Index{\hat{\ParVec}}{d}$, which in the lifted space  is given by $\langle \SDPVar,  \Matrix{A}_d^\text{real}\rangle = 2 \cdot \Re(\Index{\hat{\ParVec}}{d})$ and $\langle \SDPVar,  \Matrix{A}_d^\text{img}\rangle = 2\cdot \Im(\Index{\hat{\ParVec}}{d})$, where we define $\prn{\Matrix{A}_d^\text{real},\Matrix{A}_d^\text{img}}$ as follows: 
\begin{align*}
     \Matrix{A}_d^\text{real} =\left[\!\!\begin{array}{cc} 0_K& \conj{\F}_K^{\top} \Vector{e}_d \Vector{e}_d^{\top} \conj{\F} \\ \conj{\F}^{\top} \Vector{e}_d \Vector{e}_d^{\top} \conj{\F}_K & 0_D\end{array}\!\!\right]\;\;\text{and}\;\;
     \Matrix{A}_d^\text{img} =\left[\!\!\begin{array}{cc} 0_K& i \cdot \conj{\F}_K^{\top} \Vector{e}_d \Vector{e}_d^{\top} \conj{\F} \\  -i \conj{\F}^{\top} \Vector{e}_d \Vector{e}_d^{\top} \conj{\F}_K  & 0_D\end{array}\!\!\right].
\end{align*}\newline

\noindent Now, we can formulate $\RkCout{\ParVec}{\KerSize}{\COut}$ as follows:
\begin{equation}
\label{opt:SDPrankconstraint}
\begin{split}
    \RkCout{\ParVec}{\KerSize}{\COut}\;=\;\min\limits_{\SDPVar \succcurlyeq 0}\;  \langle \SDPVar, \Matrix{I} \rangle 
    \quad\st \quad & \forall_{d \in [\Dim]},  \langle \SDPVar, \Matrix{A}_d^\text{real} \rangle = 2  \Re (\hat{\ParVec}[d])\\[-6pt]
     &\forall_{d \in [\Dim]},  \langle \SDPVar, \Matrix{A}_d^\text{img} \rangle = 2  \Im(\hat{\ParVec}[d])\\[1pt]
     & \rank(\SDPVar) \le \COut.
\end{split}
\end{equation}

The formulation in eq. \eqref{opt:SDPrankconstraint} is non-convex due to the rank constraint. We obtain a natural convex relaxation by dropping the rank constraint, leading to the following SDP: 
\begin{equation}\label{opt:SDP}
\begin{split}
     \RSDPk{\ParVec}{\KerSize}\;=\;\min\limits_{\SDPVar \succcurlyeq 0} \; \langle \SDPVar, \Matrix{I} \rangle  
    \quad\st \quad &  \forall_{d \in [\Dim]},  \langle \SDPVar, \Matrix{A}_d^\text{real} \rangle = 2  \Re (\hat{\ParVec}[d]) \\[-6pt]
     &\forall_{d \in [\Dim]},  \langle \SDPVar, \Matrix{A}_d^\text{img} \rangle = 2  \Im(\hat{\ParVec}[d]).
\end{split}
\end{equation}
\begin{rem*}
\label{remark:lowerbound} By construction, the relaxation provides lower bounds on the induced regularizer: 
for any $\KerSize \le \Dim$, any $\COut$, and any $\ParVec \in \mathbb{R}^{\Dim}$, it holds that
$\RkCout{\ParVec}{\KerSize}{\COut} \ge \RSDPk{\ParVec}{\KerSize}$.
\end{rem*}
\begin{rem*}The symmetry properties of Fourier coefficients of real signals gives us that for any $\Dim$ and any $\w\in\bR^\Dim$,  $\Index{\hat{\w}}{p}=\conj{\Index{\hat{\w}}{\Dim-p}}$ for $p\in[\Dim]$. Thus, although the optimization problems in \eqref{opt:SDPrankconstraint} and \eqref{opt:SDP} are specified with $2\cdot\Dim$ constraints for simplicity, only $\Dim$ of them are unique. 
\end{rem*}

\subsection{Tightness of the SDP Relaxation}\label{subsec:tightness}
Our main technical result is that for any kernel size $\KerSize$, the SDP relaxation is tight (in the case of networks with single-channel inputs). Thus, the induced regularizer $\RkCoutOp{\KerSize}{\COut}$ is equivalent to an SDP that only depends upon on the kernel size $\KerSize$.

\begin{restatable}{theorem}{maintheorem}[SDP tightness]
\label{thm:tightness}
For any  $\KerSize$, $\COut$, and $\ParVec \in \mathbb{R}^{\Dim}$, it holds that $\RkCout{\ParVec}{\KerSize}{\COut} = \RSDPk{\ParVec}{\KerSize}$. 
\end{restatable}

\paragraph{Proof sketch.} We can show directly from the KKT conditions that \textit{any} minimizer $\SDPVar$ of the SDP  must have rank at most $\KerSize$. However, to prove Theorem \ref{thm:tightness}, we need to show that there exists a rank $1$ solution that has the same objective value as $\SDPVar$ and satisfies the SDP constraints---this does not follow directly from the KKT conditions. 

Constructing this rank-$1$ solution is the main technical hurdle in the proof of Theorem \ref{thm:tightness}. 
In particular, the following lemma is a key intermediate result about the convolutional operation and is of independent interest beyond this paper. 
\begin{restatable}{lem}{mainlemma}
\label{lemma:additiveproperty}
For any $1\le \KerSize \le D$, and for any vectors $\Vector{a}, \Vector{b} \in \mathbb{R}^{\KerSize}$, there exists a vector $\Vector{c} \in \mathbb{R}^{\KerSize}$ such that $\Vector{a} \Conv \Vector{a} + \Vector{b} \Conv \Vector{b} = \Vector{c} \Conv \Vector{c}$, where convolutions are  w.r.t. dimension $\Dim$. 
\end{restatable}
\noindent For $\KerSize=\Dim$, Lemma~\ref{lemma:additiveproperty} follows easily from the Fourier space representation (using $\Vector{z}\Conv\Vector{z}=|\hat{\Vector{z}}|^2$), since in this case $\hat{\Vector{c}}$ is unconstrained and can be explicitly constructed as the square root of $|\hat{\Vector{a}}|^2+|\hat{\Vector{b}}|^2$. 
However, this construction does not generalize to kernel sizes $\KerSize<\Dim$. In fact, $\Vector{c}$ does not appear to have an explicit closed-form for general $\KerSize<\Dim$. At the core of our proof, we first provide an argument for the existence of $\Vector{c}$ in the special case of $D=2K-1$ and then show that the case of general $1\le K\le D$ follows from this special case.  

\vspace{3pt}
\noindent \textit{Polynomial representation: }In our proof, we show existence of $\Vector{c}$ in Lemma~\ref{lemma:additiveproperty} by using the representation of convolutions as polynomial multiplication. This representation relies on the isomorphism between $\bR^D$ and polynomials of degree $\le D-1$ with real coefficients, \ie $\Vector{a}\in\bR^D \equiv p_\Vector{a}(x)=\Vector{a}[0]+\Vector{a}[1]x+\Vector{a}[2]x^2+\ldots+\Vector{a}[D-1]x^{D-1}$.  Lemma \ref{lemma:additiveproperty} can be written in terms of polynomials as follows: for any real-coefficient polynomials $p_{\Vector{a}}, p_{\Vector{b}}$ of degree at most $\KerSize - 1$, there exists a real-coefficient polynomial $p_{\Vector{c}}$ of degree at most $\KerSize - 1$ such that:

\[x^{\KerSize - 1} p_{\Vector{c}}(x) p_{\Vector{c}}(1/x) = x^{\KerSize - 1} p_{\Vector{a}}(x)  p_{\Vector{a}}(1/x) + x^{\KerSize - 1} p_{\Vector{b}}(x)  p_{\Vector{b}}(1/x).\]
The remainder of the proof involves implicitly constructing $p_{\Vector{c}}$ in terms of its roots (and leading coefficient). To do so, we show that the roots of the polynomial $x^{\KerSize - 1} p_{\Vector{a}}(x)  p_{\Vector{a}}(1/x) + x^{\KerSize - 1} p_{\Vector{b}}(x)  p_{\Vector{b}}(1/x)$ satisfy certain structural properties which allow us to establish the existence of the desired real-coefficient polynomial $p_{\Vector{c}}$. 
The full proof is in Appendix~\ref{appendix:tightness-main} with additional details on the proof technique in Appendix \ref{appendix:prooftechnique}.

\subsection{Implications of SDP Tightness in Theorem \ref{thm:tightness}}
The first implication of Theorem \ref{thm:tightness} is that although the optimization in \eqref{eq:inducedregularizer} is non-convex, the SDP formulation allows us to efficiently compute $\RkCout{\ParVec}{\KerSize}{\COut}$ exactly. In the remainder of the section, we discuss a number of other interesting properties that we can deduce from Theorem \ref{thm:tightness}. 

\subsubsection{$\RkCoutOp{\KerSize}{\COut}$ is independent of number of output channels $\COut$}
Theorem \ref{thm:tightness} directly implies that $\RkCoutOp{\KerSize}{\COut}$ is independent of $\COut$. This means that the linear predictors obtained by fitting training data and minimizing $\RkCout{\ParVec}{\KerSize}{\COut}$ will be invariant to $\COut$ (apart from $\RkCout{\ParVec}{\KerSize}{\COut}$ having multiple minimizers). Based on previous work (e.g. \cite{LL20}), this has implications for the asymptotic behavior of gradient descent. In particular, we can hypothesize that for networks with single channel input, the number of output channels does not influence the asymptotic predictor learned from gradient descent. 

We provide a detailed empirical evaluation of this hypothesis (along with a generalization of this hypothesis for networks with multi-channel inputs) in Section \ref{sec:experiments}. As a preview, in Figure \ref{fig:kersize3} we show the predictors learned by gradient descent on an MNIST task on a two-layer linear convolutional network with kernel size $K = 3$. Furthermore, our experiments also suggest that our theoretical results might also extend to some cases of networks with ReLU non linearity and bias (see Figure  \ref{tab:RELUbiasex}). In both cases, we see that the induced regularizer is invariant to the number of output channels. We defer a more extensive empirical evaluation to Section \ref{sec:experiments}.

\begin{figure}%
\centering
\begin{minipage}[c]{0.54\textwidth}%
\centering
    \includegraphics[width=\linewidth]{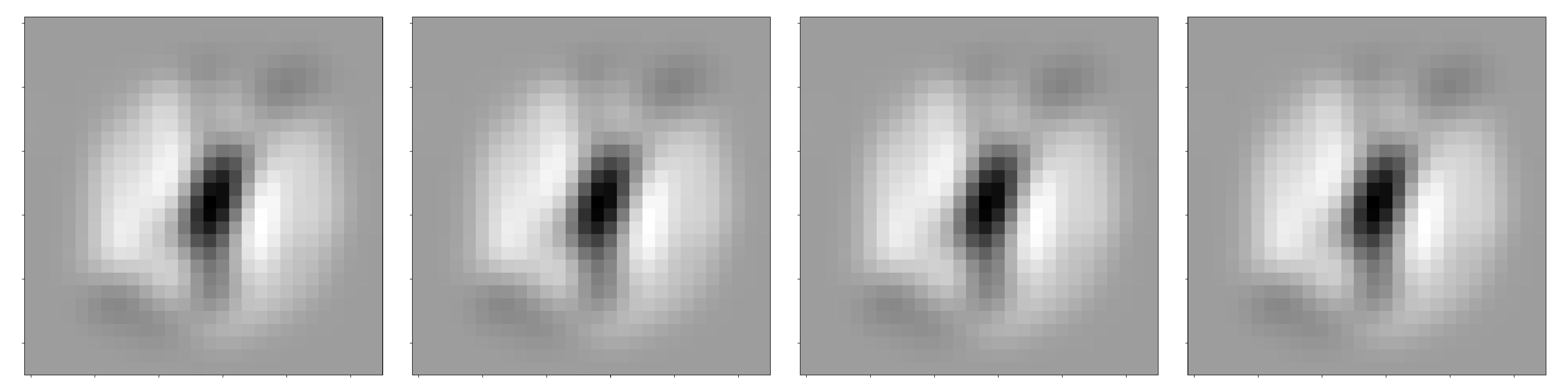}
\caption{Linear predictors learned by  two layer linear convolutional network for the task of classifying digits $0$ and $1$ in MNIST. The sub-figures depict predictors learned by using gradient descent on the  exponential loss for overparameterized networks with kernel size $\KerSize = (3, 3)$ and number of output channels $\COut \in \left\{1,2,4, 8\right\}$(left to right). }
    \label{fig:kersize3}
    \hspace{0.1cm}
\end{minipage}
~
\begin{minipage}[c]{0.42\textwidth}%
\centering\small 
\hspace{0.1cm}
    \begin{tabular}{c|c | c | c}
$\COut$ & $\KerSize: (1,1)$ &  $\KerSize: (3,3)$ & $\KerSize: (8,8)$ \\ \hline
1 & 10.581 &4.948 & 3.875  \\
2 & 10.571 & 4.945 & 3.910   \\
4 & 10.578 & 4.945 & 3.912  \\ 
8 & 10.576  &4.946  & 3.881   
\end{tabular}
    \caption{\small$\RkCoutHat{f_{\mathrm{GD}}}{\KerSize}{\COut} = \|\U\|^2+\|\V\|^2$ of the MNIST $0$-vs-$1$ classifiers learned by gradient descent on two-layer ReLU convolutional networks with bias parameters. 
    We show the median  values over 5 trials. }
    \label{tab:RELUbiasex}
\end{minipage}
\end{figure}

\subsubsection{$\boldsymbol{\RkCoutOp{\KerSize}{\COut}}$ is a norm}
Another interesting corollary of Theorem \ref{thm:tightness} is that the induced regularizer is a norm for any $\KerSize$.  

\begin{restatable}{cor}{cornorm}
\label{cor:norm}
For any $\KerSize$ and any $\COut,$
$\RkCout{\ParVec}{\KerSize}{\COut}$ is a norm. 
\end{restatable}

For the end cases of $\KerSize=1$ and $\KerSize=\Dim$, this norm can be explicitly specified: $\RkCout{\ParVec}{1}{\COut} =2\sqrt{\Dim}\norm{\hat{\ParVec}}$ (Lemma~\ref{lemma:kersize1}) and $\RkCout{\ParVec}{\Dim}{\COut}=2\norm{\hat{\ParVec}}_1$ (Lemma~\ref{lemma:kersizeD}), respectively. For intermediate kernel sizes,  $\RkCout{\ParVec}{\KerSize}{\COut}$ interpolates between the $\ell_2$ norm and the $\ell_1$ norm of the Fourier coefficients of the linear predictor.
We further use the SDP in \eqref{opt:SDP} to bound  $\RkCout{\ParVec}{\KerSize}{\COut}$ in terms of the $\ell_1$ and $\ell_2$ norms. 
\begin{restatable}{lem}{bounds}
\label{lemma:bounds}
For any $\KerSize$, $\COut$, and $\ParVec \in \mathbb{R}^{\Dim}$: 
\[{
    2\sqrt{\frac{\Dim}{\KerSize}} \norm{\hat{\ParVec}}_2\le \RkCout{\ParVec}{\KerSize}{\COut}\le 2\sqrt{D}\norm{\hat{\ParVec}}_2
     \quad\text{and}\quad
 2 \norm{\hat{\ParVec}}_1 \le 
    \RkCout{\ParVec}{\KerSize}{\COut} \le 2 \sqrt{\ceil{\frac{\Dim}{\KerSize}}} \norm{\hat{\ParVec}}_1.
}\]
\end{restatable}
\begin{rem*} For the lower bounds,  $2\norm{\hat{\ParVec}}_1$ is tight for $\ParVec = [1, 0,  \ldots, 0]$ and $2\sqrt{\frac{\Dim}{\KerSize}}\norm{\ParVec}$ is tight for $\ParVec = [1, 1, \ldots, 1]$. For the upper bounds, $2\sqrt{\ceil{\frac{\Dim}{\KerSize}}} \norm{\hat{\ParVec}}_1$ is tight when $\KerSize \mid \Dim$ for patterned vectors (see Lemma \ref{lemma:patterned}), and $2\sqrt{\Dim}\norm{\hat{\ParVec}}_2$ is tight for $[1, 0, \ldots, 0]$. 
\end{rem*}
Lemma \ref{lemma:bounds} demonstrates that when $\KerSize$ is a small constant, $\RkCout{\ParVec}{\KerSize}{\COut}$ is close to $\RkCout{\ParVec}{1}{\COut}= 2\sqrt{\Dim}\norm{\hat{\w}}_2$, whereas once $\KerSize$ is comparable to $\Dim$, $\RkCout{\ParVec}{\KerSize}{\COut}$ is close to $\RkCout{\ParVec}{\Dim}{\COut}=2 \norm{\hat{\w}}_1$.

\subsection{$\boldsymbol{\RkCoutOp{\KerSize}{\COut}}$  for patterned vectors}

Aside from general bounds on $\RkCoutOp{\KerSize}{\COut}$, the SDP formulation can also be used to analyze the behavior of the induced regularizer of special classes of vectors. One interesting  case is of \textit{patterned vectors} described as follows:
Consider vectors of the form $\ParVec(\Pattern) = [\Pattern, \Pattern, \ldots, \Pattern]\in\bR^\Dim$ consisting of repetitions of a $\SmallDim$ dimensional pattern $\Pattern\in\bR^\SmallDim$. A useful property of linear predictors of this form is that they incorporate invariance to periodic translations. 

 We show a relation between the representation cost  $\RkCout{\ParVec(\Pattern)}{\KerSize}{\COut}$ of realizing  patterned vectors in $\bR^\Dim$ and the analogous cost (denoted as $\RkCoutDim{\SmallDim}{\Pattern}{\KerSize}{\COut}$) of realizing $\Pattern$ as a linear predictor in $\bR^\SmallDim$ using a network with the same values of $\KerSize$ and $\COut$. 
\begin{restatable}{lem}{patterned}
\label{lemma:patterned} 
Consider vectors  $\ParVec(\Pattern) = [\Pattern, \Pattern, \ldots, \Pattern]\in\bR^\Dim$ specified by $\Pattern \in \mathbb{R}^{\SmallDim}$  \st $\SmallDim$ divides $\Dim$.
\begin{compactenum}[(a)]
\item For  $\KerSize \le \SmallDim$, it holds that $\forall \COut$: 
$\RkCout{\ParVec(\Pattern)}{\KerSize}{\COut} = \frac{\Dim}{\SmallDim} \cdot \RkCoutDim{\SmallDim}{\Pattern}{\KerSize}{1}.$
\item For $\KerSize\ge\SmallDim$ if $\KerSize  = \SmallDim \cdot T$ for integer $T$, then $\forall \COut$: 
$\RkCout{\ParVec(\Pattern)}{\KerSize}{\COut} = \frac{2\Dim}{\sqrt{T} \SmallDim} \norm{\hat{\Pattern}}_1 = 2 \sqrt{\frac{\Dim}{\KerSize}} \norm{\hat{\ParVec}}_1. 
$ 
\end{compactenum}
\end{restatable}
\noindent We see that the induced regularizer of repeated patterned vectors is closely related to that of the pattern itself. In particular, for $\KerSize\le\SmallDim$, we have $\RkCout{\ParVec(\Pattern)}{\KerSize}{\COut} \propto \RkCoutDim{\SmallDim}{\Pattern}{\KerSize}{1}$.

\subsubsection{Connection to previous work}\label{sec:comparison}

\paragraph{Comparison to \citet{PE20, EP202, EP21, SEPP20}.} These works study the induced regularizer of neural networks by looking at the bi-dual convex relaxation of the $\ell_2$ regularized least squares loss. In comparison to our SDP, this method is a complementary approach to derive lower bounds on the induced regularizer when minimizing convex losses over datasets. In both cases, the relaxations are trivially tight in the limit of \textit{infinitely many} output channels. However, in our work, we use the SDP formulation to show a significantly stronger result than those in prior works.

Phrased in the terminology of our work,  the results in \citet{PE20} on linear convolutional networks show that for networks with a single input channel, \textit{if the number of output channels (or width) $\COut$ is larger than a  data-dependent threshold}, then $ \RkCout{\ParVec}{\KerSize}{\COut}=\RSDPk{\ParVec}{\KerSize}$. That is, they show that the induced regularizer is independent of $\COut$ \textit{after} $\COut$ is above a certain large finite value that can be large as the dataset size. In contrast, we show the induced regularizer independent of $\COut$ for any $\COut \ge 1$. Further, our analysis holds regardless of the dataset and training loss.

\paragraph{Comparison to \citet{GLSS18}.}  \citet{GLSS18} characterized the induced regularizer for single-channel networks with full-dimensional kernels. As we discussed in Section \ref{sec:explicitbounds}, the conclusions and closed-form solution for $\KerSize = \Dim$ do not extend to the full class of two-layer linear convolutional networks. In contrast, our result (Theorem \ref{thm:tightness}) implies properties of the induced regularizer for networks with arbitrary kernel size.

Moreover, these prior works focus on networks with single-channel inputs, whereas we extend our results to networks with multi-channel inputs in the next section.

\section{Networks with multi-channel inputs}\label{sec:mult-input-channel}
While we focused on networks with single-channel inputs in the previous sections, we now expand our results to networks with multiple input channels (\eg RGB color channels). How do these conclusions change for multiple input channels? How does the induced regularizer,  now denoted as $\RkCinCoutOp{\KerSize}{\COut}{\CIn}$,  depend on the number input channels $\CIn$?

We again consider two layer convolutional networks akin to  Section~\ref{sec:background}. We first introduce additional notation: 
 The multi-channel inputs are denotes as $\X\in\bR^{\Dim\times \CIn}$, where $\CIn$ denotes the number of input channels. The convolutional first layer now has kernel size $\KerSize$, output channel size  $\COut$ and input channel size $\CIn$ with weights  denoted by 
a set of $\CIn$ matrices $\UU=\crl*{\U_\cin}_{\cin\in[\CIn]}$ with  $\U_{\cin}\in\bR^{\KerSize\times\COut}$. The output of this convolution layer $h(\UU;\X)\in\bR^{\Dim\times\COut}$ is  given as follows:
\begin{equation}
   \forall_{\cout\in\COut},\; \Index{h(\UU;\X)}{:,\cout}=\sum_{\cin=0}^{\CIn-1}\Index{\U_\cin}{:,\cout}\Conv \Index{\X}{:,\cin}.
\end{equation}
The second layer is the same as before: a single output linear layer with weights $\V\in \mathbb{R}^{\Dim \times \COut}$. We denote the equivalent linear predictor for this network by $\ParMatrixFn{(\UU,\V)}\in\bR^{\Dim\times\CIn}$.
Following similar calculations as for single input channels, $\ParMatrixFn{(\UU,\V)}$ in signal and Fourier domain (denoted as $\hat{\ParMatrixFn}{(\UU,\V)}=\F{\ParMatrixFn}{(\UU,\V)}$) are given as follows:
\begin{equation}\label{eq:w-uv-multiinput}
 \forall_{\cin\in[\CIn]},\; \ParMatrixFn{(\UU,\V)}[:,\cin]=\sum_{\cout=0}^{\COut-1} \prn*{\U_\cin[:,\cout]\Conv\V[:,\cout]^{\downarrow }}^\downarrow,\;\;\text{and}\;\;
    \hat{\ParMatrixFn}{(\UU,\V)}[:,\cin]=\diag(\hat{\U}_\cin\hat{\V}^\top).
\end{equation}

For multi-channel inputs, the set of all linear predictors is the space of matrices  $\W\in\bR^{\Dim\times\CIn}$, and we define the induced complexity measure over this matrix space as follows:
\begin{equation}\label{eq:R-multiinput}
\RkCinCout{\W}{\KerSize}{\COut}{\CIn} := \inf\limits_{\UU, \V} \sum\limits_{\cin\in[\CIn]}\norm*{\U_\cin}^2 + \norm*{\V}^2\;
\quad\st \quad\ParMatrixFn{(\UU,\V)} = \ParMatrix.
\end{equation}

\subsection{Role of output channel size \texorpdfstring{$\boldsymbol{\COut}$}{C}} 

For multi-channel inputs, we first observe that multiple output channels can be necessary to realize all linear maps. To see this, we show that the sub-network corresponding to each output channel can realize a matrix in $\bR^{\Dim\times\CIn}$ of rank at most $\KerSize$, which places an upper bound on the total rank achievable by the full network (see a proof in Appendix \ref{appendix:proofsmulti-input-realize}). This implies the following lemma: 
\begin{restatable}{lem}{realizability}
\label{lemma:realizing}
For any $\KerSize,\COut$ and $\CIn$, in order for the 
the network represented by $\ParMatrixFn{(\UU,\V)}$ in eq.~\eqref{eq:w-uv-multiinput} to realize all linear maps in $\bR^{\Dim\times\CIn}$  it is necessary that $\KerSize\cdot\COut\ge \min\crl{\CIn,\Dim}$.
\end{restatable}
In contrast to single input channels, Lemma \ref{lemma:realizing} demonstrates that, the model class realized by linear convolutional networks over  multi-channel inputs, and consequently the induced regularizer, \textit{does}  depend on number of output channels $\COut$. Nonetheless, similar to single input channel networks, we can again obtain an SDP relaxation $\RSDPkCin{\W}{\KerSize}{\CIn}$ for $\RkCinCout{\W}{\KerSize}{\COut}{\CIn}$ that is independent of $\COut$.

\subsection{SDP relaxation for multi-channel input networks}\label{appendix:proofsmulti-input-SDP}
The SDP relaxation for $\RkCinCoutOp{\KerSize}{\COut}{\CIn}$ is derived  similarly for networks with a single input channel (see additional details in Appendix \ref{appendix:proofsmulti-input}). For  $\UU=\{\U_\cin\in\bR^{\KerSize\times\COut}\}_{\cin\in[\CIn]},\V\in\bR^{\Dim\times\COut}$, we can specify the objective and constraints of \eqref{eq:R-multiinput} as linear functions of a rank $\COut$ positive semidefinite matrix $\SDPVar\in\bR^{(\Dim+\KerSize\cdot\CIn)\times(\Dim+\KerSize\cdot\CIn)}$ that represents:
 \begin{align*}
    \SDPVar=\left[\!\!\begin{array}{c} \Matrix{U}_0 \\  \ldots \\ \Matrix{U}_{\CIn} \\ \Matrix{V} \end{array}\!\!\right]\!\!\!\begin{array}{c}\begin{array}{ccccc}[\Matrix{U}_0^{\top} &  \ldots & \Matrix{U}_{\CIn}^{\top} & \Matrix{V}^{\top}\! ]\\&&&&\\&&&&\\&&&&\end{array}\end{array} \succcurlyeq 0.
\end{align*} 
We also define Hermitian matrices $\Matrix{A}_{d,\cin}^\text{real},\Matrix{A}_{d,\cin}^\text{img}\in\bR^{(\Dim+\KerSize\cdot\CIn)\times(\Dim+\KerSize\cdot\CIn)}$ for $d\in[\Dim], \cin\in[\CIn]$ as follows. Let  $\Q_d=\conj{\F}_K^\top\e_d\e_d^\top\conj{\F}$; these matrices are given by
\footnotesize{
 \begin{align*}
     \Matrix{A}_{d,0}^\text{real} &= \!\!\!\!\begin{array}{c}\left[\!\!\begin{array}{cccc;{2pt/2pt}c}  &  & & &  \Q_d \\  &  & \Matrix{0}_{(\CIn \cdot \KerSize) \times (\CIn \cdot \KerSize)}& & \Matrix{0} \\  &  &  & & \vdots \\
     &  & & & \Matrix{0} \\ 
    \hdashline[2pt/2pt] \rule{0pt}{10pt}
    \conj{\Q}_d & \Matrix{0} & \ldots & \Matrix{0} & \Matrix{0}_{\Dim \times \Dim}  \end{array}\!\!\right]\\\;\end{array}, 
    &\Matrix{A}_{d,0}^\text{img} &= \!\!\!\!\begin{array}{c}\left[\!\!\begin{array}{cccc;{2pt/2pt}c}  &  & & &  i\cdot\Q_d \\  &  & \Matrix{0}_{(\CIn \cdot \KerSize) \times (\CIn \cdot \KerSize)} & & \Matrix{0} \\  &  & & & \vdots \\ 
    &  & & & \Matrix{0} \\ 
    \hdashline[2pt/2pt] \rule{0pt}{10pt}
    -i\cdot\conj{\Q}_d & \Matrix{0} & \ldots & \Matrix{0} & \Matrix{0}_{\Dim \times \Dim}  \end{array}\!\!\right]\\\;\end{array}, \\
    &\vdots&\vdots&\\
    \Matrix{A}_{d,\CIn}^\text{real} &= \!\!\!\!\begin{array}{c}\left[\!\!\begin{array}{cccc;{2pt/2pt}c}  &  & & &  \Matrix{0} \\  &  &\Matrix{0}_{(\CIn \cdot \KerSize) \times (\CIn \cdot \KerSize)}  & & \Matrix{0} \\  &  & & & \vdots  \\
    &  & & & \Q_d  \\ 
    \hdashline[2pt/2pt] \rule{0pt}{10pt}
    \Matrix{0} & \Matrix{0} & \ldots & \conj{\Q}_d & \Matrix{0}_{\Dim \times \Dim}  \end{array}\!\!\right]\\\;\end{array},
    &\Matrix{A}_{d,\CIn}^\text{img} &= \!\!\!\!\begin{array}{c}\left[\!\!\begin{array}{cccc;{2pt/2pt}c}  &  & & &  \Matrix{0} \\  &  &\Matrix{0}_{(\CIn \cdot \KerSize) \times (\CIn \cdot \KerSize)} & & \Matrix{0} \\  &  &  & & \vdots  \\
    &  & & & i\cdot\Q_d  \\ 
    \hdashline[2pt/2pt] \rule{0pt}{10pt}
    \Matrix{0} & \Matrix{0} & \ldots & -i\cdot\conj{\Q}_d & \Matrix{0}_{\Dim \times \Dim}  \end{array}\!\!\right]\\\;\end{array}.
 \end{align*}}
 \normalsize{}

\noindent We now have the following relaxation of $\RkCinCout{\ParMatrix}{\KerSize}{\COut}{\CIn}$ in terms of Fourier coefficients  $\hat{\W}=\F\W$:
 \begin{equation}
 \label{opt:SDPmulti-channel}
 \begin{aligned}
     \RSDPkCin{\ParMatrix}{\KerSize}{\CIn}\; = \; \min_{\SDPVar \succcurlyeq 0} \quad & \langle \SDPVar, \Matrix{I} \rangle  \\
     \textrm{s.t.} \quad & \forall d \in [\Dim], \cin\in[\CIn]  \; \; \langle \SDPVar, \Matrix{A}_{d,\cin}^\text{real} \rangle = 2  \text{Re}(\hat{\ParMatrix}[d,r]) \\
     & \forall d \in [\Dim], \cin\in[\CIn] \; \; \langle \SDPVar, \Matrix{A}_{d, \cin}^\text{img} \rangle = 2  \text{Im}(\hat{\ParMatrix}[d,r]).
 \end{aligned}
 \end{equation}
 We can check that the SDP in \eqref{opt:SDPmulti-channel} with a rank constraint of $\rank{(\SDPVar)}\le\COut$ is equivalent $\RkCinCout{\ParMatrix}{\KerSize}{\COut}{\CIn}$ and the SDP thus provides a lower bound: \ie $\forall_{\W},\; \RkCinCout{\ParMatrix}{\KerSize}{\COut}{\CIn}\ge \RSDPkCin{\ParMatrix}{\KerSize}{\CIn}$.


 \subsection{Tightness of SDP Relaxation}
 
Unlike networks with single channel input, the SDP relaxation here is not always tight when $\CIn>1$, since a sufficiently large $\COut$ is required to merely realize all matrix-valued linear function over the input space. We can however show a weaker form SDP tightness from the KKT conditions when there are sufficiently many output channels: 
\begin{restatable}{lem}{multichanneltightness}
\label{lemma:multichanneltightness}
For any $\W\in\bR^{\Dim\times\CIn}$, and any $\COut\ge\CIn \KerSize$, it holds that $\RkCinCout{\ParMatrix}{\KerSize}{\COut}{\CIn} = \RSDPkCin{\W}{\KerSize}{\CIn}$.
\end{restatable}

Note that the above bound on $\COut$ for SDP tightness is not sharp, as we showed for $\CIn = 1$ in Theorem~\ref{thm:tightness}. Based on our insights from  the proof of single-channel SDP tightness in Theorem~\ref{thm:tightness} and additional empirical evidence in Appendix~\ref{appendix:experiments}, we conjecture that SDP tightness  holds when $\COut \ge \CIn$: 
\begin{conjecture}
\label{conj:multichanneltightness} 
 For any $\W\in\bR^{\Dim\times\CIn}$, and any $\COut\ge\CIn$, it holds that $\RkCinCout{\ParMatrix}{\KerSize}{\COut}{\CIn} = \RSDPkCin{\W}{\KerSize}{\CIn}$.
\end{conjecture}

In the next subsection, we prove Conjecture \ref{conj:multichanneltightness} in the special cases of $\KerSize=1$ and $\KerSize=\Dim$. As a consequence, we show that once $\COut$ is large enough to realize all linear maps, $\RkCinCout{\W}{\KerSize}{\COut}{\CIn}$ can be expressed as interesting closed form norms independent of  $\COut$ in these special cases. 

\subsection{Induced regularizer when \texorpdfstring{$\boldsymbol{\KerSize=1}$}{K=1} and \texorpdfstring{$\boldsymbol{\KerSize=D}$}{K=D}}\label{sec:multi-input-explicit}
\begin{restatable}{theorem}{multiinputone}\label{thm:multichannelkersize1}
For any $\W\in\bR^{\Dim\times\CIn}$, and any $\COut\ge\min\crl{\CIn,\Dim}$, the induced regularizer for $\KerSize=1$ is given by the scaled nuclear norm $\|.\|_*$: \[\RkCinCout{\ParMatrix}{1}{\COut}{\CIn}=2\sqrt{\Dim}\norm*{{\ParMatrix}}_*=2\sqrt{\Dim}\norm*{\hat{\ParMatrix}}_*.\] 
\end{restatable}

\begin{restatable}{theorem}{multiinputD}\label{thm:multichannelkersizeD}
For  any $\W\in\bR^{\Dim\times\CIn}$, and any $\COut\ge1$, the induced regularizer for $\KerSize=\Dim$ is given by
\[\RkCinCout{\ParMatrix}{\Dim}{\COut}{\CIn}=2\norm*{\hat{\W}}_{2,1}:=\sum_{d=0}^{\Dim - 1} \sqrt{\sum_{\cin=0}^{\CIn-1}\big|\Index{\hat{\ParMatrix}}{d,\cin}\big|^2}.\]
\end{restatable}

From Theorems~\ref{thm:multichannelkersize1}-\ref{thm:multichannelkersizeD} it is evident that the number of input channels $\CIn$ fundamentally changes the nature of induced complexity measure
in the function space and introduces additional structures along the input channels. Even in the simplest setting of scalar convolution kernels with $\KerSize=1$, the induced regularizer is no longer a Euclidean or RKHS norm, and is instead a richer nuclear norm that encourages low-rank properties. For the case of $\KerSize=\Dim$,  the induced regularizer is  group-sparse norm on the Fourier coefficients that encourages similar weighting across channels, while promoting sparsity across frequency components. In comparison to the $\ell_1$ norm of all Fourier coefficients, this group-sparse norm is a more structured inductive bias for multi-channel inputs. Additionally, like with the single input channel case, we also observe that the induced bias has a more intuitive and interesting interpretation in Fourier domain which is not directly observed in the signal domain.

\section{Experiments}\label{sec:experiments}

We now explicitly connect our findings to the implicit regularization of gradient descent. We formally state the following result paraphrased from \citet{LL20} that relate the asymptotic implicit bias of gradient descent to $\ell_2$ norm minimization of the parameters: 
\begin{thm*}{\citep[Paraphrased from ][]{LL20}} 
 Assume that $\Phi$ is  locally Lipschitz and positive homogeneous with order $L > 0$, \ie $\forall_{\bf{\theta},\alpha>0}$, $\Phi(\alpha\boldsymbol{\theta};.) = \alpha^L \Phi(\boldsymbol{\theta};.)$. Consider minimization of logistic or exponential loss over a separable binary classification dataset $\left\{(\Vector{x}_n, y_n)\right\}_{n=1}^N$ using a gradient flow trajectory denoted as $\b{\theta}(t)$. Under the assumption that the training data points are correctly classified in finite time (\ie mean logistic/exponential loss $\mathcal
 {L}(\b{\theta}(t_0))<1$ at finite $t_0$), the limit points of the direction of parameters of gradient flow correspond to a first order stationary point (KKT point) of the following  max--$\ell_2$ margin problem in parameter space:
\begin{equation}\label{eq:max-margin} 
\min\limits_{\Vector{\theta}}\norm{\Vector{\theta}}_2^2\;\st\; \forall_{n} y_n \Phi(\Vector{\theta}; \Vector{x}_n) \ge 1.
\end{equation} 
\end{thm*}
\begin{rem*}
With additional assumptions, \citet{LL20} also prove an analogous result for gradient descent. This result was further refined in \citet{ji2020directional}, which proves the directional convergence of gradient descent under additional structure on the model class. 
\end{rem*}

\noindent Based on these results, we thus expect the implicit bias from gradient descent to be related to the \textit{max}--$\R$--margin problem 
\[\min\limits_{f}\R_\Phi(f)\;\st\; \forall_{n} y_n f(\Vector{x}_n) \ge 1.\]
However, theoretically speaking, there is an important caveat: the theorem by \citet{LL20} shows convergence of the gradient flow direction to a \textit{stationary point} of the optimization problem in  eq.~\eqref{eq:max-margin}.

Nonetheless, if we overlook these caveats, our  findings would then have important implications for predictors learned from gradient descent. In particular, our result regarding the invariance of the induced regularizer with respect to the number of output channels suggests that the asymptotic behavior of gradient descent is similarly invariant to the number of output channels. We formalize this as a testable hypothesis.  
\begin{hypothesis}
\label{hyp:outputchannels}
For a separable binary classification task with $\CIn$ input channels, let $\w_{\text{GD}}$ be the predictor learned using stochastic gradient descent on a two-layer convolutional network with kernel size $\KerSize$, $\COut$ output channels, and $\CIn$ input channels (where $\w_{\text{GD}}$ is normalized to have unit margin on the training data). Then, as long as $\COut \ge \CIn$, the induced regularizer $\RkCinCout{\w_{\text{GD}}}{\KerSize}{\COut}{\CIn}$ is invariant in the number of output channels $\COut$.
\end{hypothesis}
We show experimental support for the hypothesis on small linearly separable subsets of MNIST (with $128$ images of size $28 \times 28$ balanced across 2 classes) and CIFAR-10 (with $512$ images of size $32 \times 32$ balanced across 2 classes) datasets.  Most of our experiments are for multi-channel linear convolutional networks trained using stochastic gradient descent. We also provide some experiments on ReLU networks, where we see support of our hypothesis well beyond our theoretical study. 

Throughout the experiments sections, since we cannot always compute $\RkCinCout{\w_{\text{GD}}}{\KerSize}{\COut}{\CIn}$, we approximate it using the weight norms of the trained network $\RkCinCoutHat{\w_{\text{GD}}}{\KerSize}{\COut}{\CIn}=\sum_{\UU,\V}\norm{\UU}^2+\norm{\V}^2$, where $\UU,\V$ here denote the weights of the trained network.\footnote{In theory $\hat{\R}$ only provides an upper bound on $\R$---but in case of predictors learned by SGD, upon checking instances where $\R$ has a closed form solution, we found that the approximation is quite accurate.} The experiments are deferred to Appendix~\ref{appendix:experiments} and we summarize the findings below. 
{\begin{compactenum}
\item \textit{Single input channel binary classification on MNIST.} In linear networks, we compare  the predictors learned by gradient descent for 
$\KerSize \in \left\{1, 3, 8, 16,  28\right\}$ and $\COut \in \left\{1, 2, 4, 8\right\}$ 
across $10$ runs with random initialization. We see that both the values of estimated regularizer
$\hat{\R}$ as well as the visualization of linear predictors in signal and frequency domain are nearly invariant to $\COut$ (the values overlap within one standard deviation across runs). 
\item \textit{$3$--input channel binary classification on CIFAR-10.} In a similar setup to MNIST, we compare $\KerSize \in \left\{1, 3, 8, 20\right\}$ and $\COut \in \left\{1,2, 3, 4, 8\right\}$.
As expected from our theory, we see differences in the induced regularizer $\hat{\R}$ for $\COut<3$, but observe invariance to $\COut$ once $\COut\ge 3$. 
\item \textit{ReLU networks for binary classification on MNIST.} Although our theory is only for linear networks, our hypothesis as stated above can also be tested on networks with non-linearity. We repeat our MNIST experiments on networks with ReLU non-linearity (with and without bias parameters). Interestingly, we  observe that the estimated induced regularizer $\hat{\R}$ is invariant to $\COut$ suggesting a broader scope for our hypothesis. 
\end{compactenum}}

\noindent These findings  support  Hypothesis \ref{hyp:outputchannels} beyond the scope of our theoretical results.

\section{Discussion}\label{sec:discussion}

We showed that when minimizing $\ell_2$ norm of weights, the two basic architectural components of convolutional networks---number of output channels (width) and kernel size---have interesting effects even in the simple case of two-layer linear networks. Our results also inspire a broader hypothesis about the impact of number of output channels for networks learned with gradient descent, which we tested and provided support for in our experiments. 

Interesting directions for future work include proving tightness of the SDP relaxation for multiple input channels (formalized in Conjecture \ref{conj:multichanneltightness}); formally establishing the limiting behavior of gradient descent; and exploring architectural features such as pooling or multiple layers. Furthermore, it would be interesting to conduct an in-depth empirical investigation of the impact of non-linearity.

\bibliographystyle{plainnat}
\bibliography{bibliography-final.bib}

\clearpage
\appendix

\section{Proof of Theorem~\ref{thm:tightness}: SDP tightness} \label{appendix:tightness-main}
\maintheorem*

The high-level idea of the proof of Theorem \ref{thm:tightness} is to take an optimal solution $\SDPVar$ of the $\RSDPk{\w}{\KerSize}$ problem in \eqref{opt:SDP}, and  construct a rank 1 solution that obtains the same objective and satisfies the constraints. We reiterate the SDP formulation for easy reference:
\begin{equation}
\begin{array}{r@{\ }c@{\ }l}
    \!\!\!\!\!\RSDPk{\ParVec}{\KerSize}= & \min\limits_{\SDPVar \succcurlyeq 0}  & \langle \SDPVar, \Matrix{I} \rangle  \\
    &\st  & \forall_{d \in [\Dim]},  \langle \SDPVar, \Matrix{A}_d^\text{real} \rangle = 2  \Re (\hat{\ParVec}[d]) \\
    & &\forall_{d \in [\Dim]},  \langle \SDPVar, \Matrix{A}_d^\text{img} \rangle = 2  \Im(\hat{\ParVec}[d]).
\end{array}\tag{SDP}\label{eq:sdprepeat}
\end{equation}

\noindent The outline of the proof is detailed below:
\begin{compactenum}
\item In \ref{subsec:kkt-main}, we look at the KKT conditions  of \eqref{eq:sdprepeat} and show that Theorem~\ref{thm:tightness} follows from Lemma~\ref{lemma:additiveproperty}, which states that for all $\ba,\bb\in\bR^\KerSize$, there exists $\bc\in\bR^\KerSize$ such that $\ba\Conv\ba + \bb\Conv\bb=\bc\Conv\bc$, where the convolutions are w.r.t. $\Dim$.
\item In \ref{subsec:proofadditiveproperty} we provide the proof of Lemma~\ref{lemma:additiveproperty}, which is further subdivided as follows:
\begin{compactenum}
    \item \ref{subsec:proofspecialcasereduction} reduces Lemma~\ref{lemma:additiveproperty} to the special case where $\Dim=2\KerSize-1$.
    \item \ref{subsec:polyrep} introduces a polynomial representation of convolutions which leads to a reformulation of Lemma~\ref{lemma:additiveproperty} with $\Dim=2\KerSize-1$.
    \item \ref{subsec:proofinpolyrep} contains the core argument for existence of the desired $\bc\in\bR^\KerSize$ via reasoning about the roots of polynomial reformulation.
\end{compactenum}
\end{compactenum}

\subsection{KKT conditions for  \eqref{eq:sdprepeat}}\label{subsec:kkt-main}
The $\Dim$ constraints involving $\Re(\ParVec)$ correspond to a dual vector $\Vector{\lambda^{\text{real}}} \in \mathbb{R}^{\Dim}$; the $\Dim$ constraints involving $\Im(\ParVec)$ correspond to a dual vector $\Vector{\lambda^{\text{img}}} \in \mathbb{R}^{\Dim}$. To simplify these conditions, we take $\Vector{\lambda} \in \mathbb{C}^{\Dim}$ to be $\Vector{\lambda^{\text{real}}} + i \cdot \Vector{\lambda^{\text{img}}}$ (when $\Vector{\lambda}$ is a dual-optimal solution, $\Vector{\lambda}$ is also a Fourier transform of a real vector). The dual variable for the PSD constraint corresponds to a matrix $\Matrix{\Gamma} \succcurlyeq 0$. In this notation, the KKT conditions are primal feasibility, along with the following constraints:
\begin{align*}
    \Matrix{\Gamma} &= \Identity -  \left[\begin{array}{cc} \Matrix{0}_K & \conj{\Matrix{F}}_{\KerSize}^{\top} \Matrix{\Lambda} \conj{\Matrix{F}}  \\ \Matrix{F} \conj{\Matrix{\Lambda}} \Matrix{F}_{\KerSize}  & \Matrix{0}_{\Dim} \end{array}\right]\succcurlyeq 0 \\
    \Matrix{\SDPVar} \conj{\Matrix{\Gamma}} &= 0.
\end{align*}
Now, to simplify these conditions, suppose that $\SDPVar$ is rank $L$, in which case we can express it as $\SDPVar = \left[\begin{array}{c} \Matrix{U} \\\Matrix{V} \end{array}\right] \left[\begin{array}{cc} \Matrix{U}^{\top} & \Matrix{V}^{\top} \end{array}\right],$
where $\Matrix{U} \in \mathbb{R}^{\KerSize \times L}$ and $\Matrix{V} \in \mathbb{R}^{\Dim \times L}$. Using that $\left[\begin{array}{c} \Matrix{U} \\\Matrix{V} \end{array}\right]$ is full rank along with some algebraic manipulations, we obtain the following formulation of the KKT conditions: 
\begin{align*}
  \sum_{l=0}^L \hat{\Matrix{U}}[:, l] \Entrywise \hat{\Matrix{V}}[:, l] &= \hat{\ParVec} \label{eq:kkt1} \tag{KKT $1$}\\ 
  \left[\begin{array}{cc} \Matrix{0}_K & \conj{\Matrix{F}}_{\KerSize}^{\top} \Matrix{\Lambda} \conj{\Matrix{F}}  
  \\ \Matrix{F} \conj{\Matrix{\Lambda}} \Matrix{F}_{\KerSize}  & \Matrix{0}_{\Dim} \end{array}\right] &\preccurlyeq  \Identity_{\Dim + \KerSize}   \label{eq:kkt2}\tag{KKT $2$}\\
  \conj{\hat{\Matrix{V}}} &=  \conj{\Matrix{\Lambda}} \hat{\Matrix{U}} \label{eq:kkt3} \tag{KKT $3$}\\
  \hat{\Matrix{U}} &=  \Matrix{F}_K\conj{\Matrix{F}}_K^{\top}  \Matrix{\Lambda}  \conj{\hat{\Matrix{V}}}.\tag{KKT $4$}\label{eq:kkt4}
\end{align*}

The KKT conditions give useful properties of the solution $\SDPVar$. For $0 \le l \le L - 1$, we let $\Vector{u}_l = \Matrix{U}[:, l]$ and $\Vector{v}_l = \Matrix{V}[:, l]$. From \eqref{eq:kkt3}, we see that 
\[\forall_{l\in[L]},\;\;\hat{\Vector{v}}_l = \Vector{\lambda} \Entrywise \conj{\hat{\Vector{u}}_l}.\]  
Combining the above with \eqref{eq:kkt1}, we obtain:
\[\hat{\ParVec} = \sum_{l=0}^{L-1} \Vector{\lambda} \Entrywise \conj{\hat{\Vector{u}}_l} \Entrywise \hat{\Vector{u}}_l.\]

\noindent We now make the following assertion.
\begin{clm*} We claim that in order to prove Theorem~\ref{thm:tightness}, it suffices to find a vector $\Vector{u} \in \mathbb{R}^{\KerSize}$ such that: 
\begin{equation}
\label{eq:desired}
\conj{\hat{\Vector{u}}} \Entrywise  \hat{\Vector{u}}  = \sum_{l=0}^{L-1} \conj{\hat{\Vector{u}}_l} \Entrywise \hat{\Vector{u}}_l.\tag{CONV-REDUCTION}
\end{equation}
\end{clm*}
The most technical part of the proof is to show \eqref{eq:desired} which we show follows from Lemma~\ref{lemma:additiveproperty}. But before that we will first justify our above claim. 

\begin{proof}[Proof of claim] Assume \eqref{eq:desired} holds for $\Vector{u}\in\bR^\KerSize$. Let $\Vector{v} = \F^{-1}(\Vector{\lambda} \Entrywise \conj{\hat{\Vector{u}}})$ such that we have $\hat{\Vector{v}}=\Vector{\lambda}\Entrywise\conj{\hat{\Vector{u}}}$. We can take $\SDPVar^*$ to be $\left[\!\!\begin{array}{c} \Vector{u} \\ \Vector{v}  \end{array}\!\!\right] \!\!\!\!\begin{array}{c}[\!\!\begin{array}{cc} \Vector{u}^{\top} & \Vector{v}^{\top} \end{array}\!\!]\\\;\end{array}$. 

We can see that $(\Vector{u},\Vector{v})$ satisfies the following:
\begin{equation}
\hat{\Vector{u}} \Entrywise \hat{\Vector{v}} = \Vector{\lambda} \Entrywise \conj{\hat{\Vector{u}}} \Entrywise \hat{\Vector{u}} = \sum_{l=0}^{L-1} \Vector{\lambda} \Entrywise \conj{\hat{\Vector{u}}_l} \Entrywise \hat{\Vector{u}}_l = \hat{\ParVec}.
\end{equation}

Thus, $\SDPVar^*$ satisfies the feasibility condition for $\RSDPk{\ParVec}{\KerSize}$. Moreover, we show that the solution also achieves the optimum objective value for $\RkCout{\ParVec}{\KerSize}{\COut}$ as follows:
\begin{equation}
\begin{split}
  \langle\SDPVar^*,\Identity\rangle & = \norm{{\Vector{u}}}_2^2 + \norm{{\Vector{v}}}_2^2 = \norm{\hat{\Vector{u}}}_2^2 + \norm{\hat{\Vector{v}}}_2^2 \\
  & = \sum_{d=0}^{\Dim - 1} |\hat{\Vector{u}}_d|^2 + \sum_{d=0}^{\Dim - 1} |\hat{\Vector{v}}_d|^2 \\
  &\overset{(a)}=  \sum_{d=0}^{\Dim - 1} \sum_{l=0}^{L-1} |(\hat{\Vector{u}}_l)_d|^2 + \sum_{d=0}^{\Dim - 1} \sum_{l=0}^{L-1} |\Vector{\lambda}_d|^2 |(\hat{\Vector{u}}_l)_d|^2 \\
  &\overset{(b)}= \sum_{l=0}^{L-1}  \sum_{d=0}^{\Dim - 1} |(\hat{\Vector{u}}_l)_d|^2 + \sum_{l=0}^{L-1} \sum_{d=0}^{\Dim - 1}  |(\hat{\Vector{v}}_l)_d|^2 \\
  &= \norm{\U}^2 + \norm{\V}^2 \\
  &= \langle \SDPVar, \Identity \rangle=\RSDPk{\ParVec}{K}\\
  \end{split}
\end{equation}
where $(a)$ follows from \eqref{eq:desired} and using $\hat{\Vector{v}}=\Vector{\lambda}\Entrywise\conj{\hat{\Vector{u}}}$, which follows by definition of $\Vector{v}$, and $(b)$ follows from \eqref{eq:kkt3} that $\hat{\Vector{v}}_l=\Vector{\lambda}\Entrywise\conj{\hat{\Vector{u}}}_l$.

Thus, we have show that \eqref{eq:desired} implies that there exists a rank $1$ solution $\SDPVar^*$ that achieves the same objective value as $\SDPVar$ and satisfies the constraints, and hence is minimizer of the SDP as desired. 
\end{proof} 

\paragraph{Proof of \eqref{eq:desired} from Lemma~\ref{lemma:additiveproperty}} 
It is more convenient to write  \eqref{eq:desired} in signal space.  Taking  inverse Fourier transforms of \eqref{eq:desired}, we need to show that  given $\Vector{u}_l=\F_\KerSize^{\top}\hat{\Vector{u}}_l\in\bR^\KerSize$ there exists $\Vector{u}\in\bR^\KerSize$ such that the following  holds: 
\begin{equation}
\label{eq:mainingredient}
  \Vector{u} \Conv \Vector{u} = \sum_{l=0}^{L-1} \Vector{u}_l \Conv \Vector{u}_l,
\end{equation}
where the convolutions are taken in $\Dim$ dimensional space, \ie $\{\Vector{u}_l\}_l$ are padded with $\Dim-\KerSize$ zeros so that $\Vector{u}_l \Conv \Vector{u}_l\in\bR^D$.  \newline

\noindent We can now see that \ref{eq:mainingredient} (and hence \eqref{eq:desired}) indeed holds by recursively applying Lemma~\ref{lemma:additiveproperty}, which was  stated earlier in the main text and is  reiterated below. 
 \mainlemma*
 
 
 We have thus far shown that Theorem \ref{thm:tightness} follows from Lemma \ref{lemma:additiveproperty}: we first established that it suffices to find  $\b{u} \in \mathbb{R}^{\KerSize}$ satisfying \eqref{eq:mainingredient} (or equivalently \eqref{eq:desired}), which in turn  holds from recursively applying Lemma \ref{lemma:additiveproperty}. 
  Now, it suffices to prove Lemma~\ref{lemma:additiveproperty}; we prove this lemma in the following subsections.

\subsection{Proof of Lemma~\ref{lemma:additiveproperty}: the convolutional property}\label{subsec:proofadditiveproperty}
For $\KerSize=\Dim$, Lemma~\ref{lemma:additiveproperty} follows easily from the Fourier space representation, since the Fourier space representation of the vector $\Vector{c}$ can be explicitly constructed as the square root of the Fourier representation of $\Vector{a} \Conv \Vector{a} + \Vector{b} \Conv \Vector{b}$. However, this construction does not generalize to kernel sizes $\KerSize<\Dim$, and Lemma \ref{lemma:additiveproperty} is thus non-trivial in general. In fact, the vector $\Vector{c}$ does not even appear to have a clean closed-form characterization for general $\KerSize$. To sidestep this issue, we use a proof technique that enables us to implicitly construct the vector $\Vector{c}$. We believe that this proof technique could be of independent interest. 
\subsubsection{Reducing Lemma \ref{lemma:additiveproperty} to $\Dim = 2\KerSize - 1$ Case}\label{subsec:proofspecialcasereduction}

The first step in the proof of Lemma \ref{lemma:additiveproperty} is to show that Lemma~\ref{lemma:additiveproperty} for general $\KerSize,\Dim$ follows from the special case where $\Dim = 2\KerSize - 1$.
\begin{lemma}
\label{lemma:additivepropertyspecialcase}
For any $\KerSize \ge 1$, for $\Dim = 2\KerSize - 1$, and for any vectors $\Vector{a}, \Vector{b} \in \mathbb{R}^{\KerSize}$, there exists a vector $\Vector{c} \in \mathbb{R}^{\KerSize}$ such that $\Vector{a} \Conv \Vector{a} + \Vector{b} \Conv \Vector{b} = \Vector{c} \Conv \Vector{c}$, where convolutions are taken in dimension $\Dim$. 
\end{lemma}

\paragraph{Proof of Lemma \ref{lemma:additiveproperty} from Lemma \ref{lemma:additivepropertyspecialcase}}
By Lemma \ref{lemma:additivepropertyspecialcase}, we know that Lemma \ref{lemma:additiveproperty} holds when $\Dim = 2\KerSize - 1$. We now show that this implies the statement for a general value of $\Dim$. Exclusively in this proof, let  $\ConvD{\Dim}$ denote the convolutional operator w.r.t dimension $D$. For $\ba,\bb\in\bR^{\KerSize}$, let $\bc\in\bR^\KerSize$ be the vector from Lemma~\ref{lemma:additivepropertyspecialcase} such that  $\Vector{a} \ConvD{2\KerSize-1} \Vector{a} + \Vector{b} \ConvD{2\KerSize-1} \Vector{b} = \Vector{c} \ConvD{2\KerSize-1} \Vector{c}$. We will now show that the same $\bc$ also satisfies $\Vector{a} \ConvD{\Dim} \Vector{a} + \Vector{b} \ConvD{\Dim} \Vector{b} = \Vector{c} \ConvD{\Dim} \Vector{c}$ for all $\Dim$. 

\paragraph{Case 1: $\Dim \ge 2 \KerSize - 1$.} First, notice that for $K$ dimensional vectors $\ba,\bb\in\bR^{\KerSize}$ and $\bc\in\bR^K$, their  convolutions for $\Dim\ge2\KerSize-1$ will have at most $2\KerSize - 1$ nonzero entries, located at indices $[-\KerSize + 1, \KerSize - 1] \mod  \Dim$. Thus, for $d' \not\in [-\KerSize + 1, \KerSize - 1]$ 
\begin{equation}
    (\Vector{a} \ConvD{\Dim} \Vector{a} + \Vector{b} \ConvD{\Dim} \Vector{b})_{d'\!\!\!\!\mod \Dim} = 0 =(\Vector{c} \ConvD{\Dim} \Vector{c})_{d'\!\!\!\!\mod \Dim}.
\end{equation} 
At the same time, for $d \in [-\KerSize + 1, \KerSize - 1]$, we have that:
\begin{equation}
    \begin{split}
        (\Vector{a} \ConvD{\Dim} \Vector{a} + \Vector{b} \ConvD{\Dim} \Vector{b})_{d\!\!\!\!\mod \Dim}  &=(\Vector{a} \ConvD{2\KerSize-1} \Vector{a} + \Vector{b} \ConvD{2\KerSize-1} \Vector{b})_{d\!\!\!\!\mod (2\KerSize-1)} \\
        &=(\Vector{c} \ConvD{2\KerSize-1} \Vector{c})_{d\!\!\!\!\mod{(2\KerSize-1)}}\\
        &=(\Vector{c} \ConvD{2\KerSize-1} \Vector{c})_{d\!\!\!\!\mod{(\Dim)}}.
    \end{split}
\end{equation}

\paragraph{Case 2: $\Dim < 2 \KerSize - 1$.} 
Let $\KerSize\le \Dim<2\KerSize-1$. For any vector $\b{z}\in\bR^\KerSize$, we can check the following:
\begin{equation}
\label{eq:c}
\forall_{d\in[D]},\; (\Vector{z} \ConvD{\Dim} \Vector{z})_d = \sum_{\substack{0\le d' < 2\KerSize - 1\\ \text{s.t., } d \equiv d'\!\!\!\!\mod D}} (\Vector{z} \ConvD{{2\KerSize - 1}} \Vector{z})_{d'}.
\end{equation}
Thus, for $\ba,\bb\in\bR^\KerSize$ and $\bc\in\bR^\KerSize$ satisfying  Lemma~\ref{lemma:additivepropertyspecialcase}, we get the desired result as follows 
\begin{equation}
\label{eq:ab}
\begin{split}
\forall_{d\in[D]},\; (\Vector{a} \ConvD{\Dim} \Vector{a} + \Vector{b} \ConvD{\Dim} \Vector{b})_d &= \sum_{\substack{0 \le d' < 2\KerSize - 1 \\\text{s.t., } d \equiv d' \mod D}} (\Vector{a} \ConvD{2\KerSize-1} \Vector{a} + \Vector{b} \ConvD{2\KerSize-1} \Vector{b})_{d'}\\
&=\sum_{\substack{0 \le d' < 2\KerSize - 1 \\\text{s.t., } d \equiv d' \mod D}} (\Vector{c} \ConvD{2\KerSize-1} \Vector{c})_{d'}\\
&= (\Vector{c} \Conv_D \Vector{c})_d.
\end{split}
\end{equation}
This concludes the proof of showing Lemma~\ref{lemma:additiveproperty} from Lemma~\ref{lemma:additivepropertyspecialcase}. 
\mybox

The remainder of the section of devoting to proving Lemma \ref{lemma:additivepropertyspecialcase}. This statement trivially holds with $\Vector{c}$ as the zero vector if $\Vector{a} \Conv \Vector{a} + \Vector{b} \Conv \Vector{b} = 0$, so for the remainder of the proof, we assume that $\Vector{a} \Conv \Vector{a} + \Vector{b} \Conv \Vector{b}$ is nonzero.

\subsubsection{Introducing the polynomial representation} \label{subsec:polyrep}
The key idea for the proof of Lemma \ref{lemma:additivepropertyspecialcase} is to use the polynomial formulation of convolutions.\footnote{E.g., see \url{https://en.wikipedia.org/wiki/Convolution}} We then use factorization of the polynomials to implicitly construct $\Vector{c}$. We use the following notation. Let $\mathcal{P}_k \subseteq \mathbb{R}[x]$ denote the set of degree $\le \KerSize - 1$ polynomials with \textit{real coefficients}. For a vector $\Vector{z} \in \mathbb{R}^\KerSize$, we define the \textit{polynomial representation} $p_{\Vector{z}}(x) \in \mathcal{P}_{\KerSize -1}$ to be the polynomial $\Vector{z}[0] + \Vector{z}[1] x + \ldots + \Vector{z}[{\KerSize - 1}] x^{\KerSize - 1}$. Using the polynomial representations, convolutions can be expressed as polynomial multiplication:
\begin{fact}
\label{fact:polymultiplication}
Let $\Vector{a} \in \mathbb{R}^\KerSize$. The $\Dim=2\KerSize-1$ dimensional convolution $\Vector{a} \Conv \Vector{a}$ has polynomial representation $p_{\Vector{a} \Conv \Vector{a}}(x)$ that is equivalent to the polynomial $x^{\KerSize - 1} p_{\Vector{a}}(x) \cdot p_{\Vector{a}}(1/x) \in \mathcal{P}_{2 \KerSize - 2}$ up to permuting the coefficients appropriately.
\end{fact}
The polynomial representation enables us to construct a vector $\Vector{c}$ in terms of the roots of the relevant polynomials. We now reformulate Lemma~\ref{lemma:additivepropertyspecialcase}  using the polynomial representation. Recall that we wish to show that there exists a vector $\Vector{c}\in\bR^{\KerSize}$ such that $\Vector{c} \Conv \Vector{c} = \Vector{a} \Conv \Vector{a} + \Vector{b} \Conv \Vector{b}.$

In the polynomial representation, we equivalently want $p_{\Vector{c} \Conv \Vector{c}}(x) = p_{\Vector{a} \Conv \Vector{a}}(x) + p_{\Vector{b} \Conv \Vector{b}}(x).$ Now, applying Fact \ref{fact:polymultiplication}, we see that the polynomial formulation of the Lemma~\ref{lemma:additivepropertyspecialcase} is the following:
\[\exists \bc\in\bR^\KerSize\;\st \;x^{\KerSize - 1} p_{\Vector{c}}(x) \cdot p_{\Vector{c}}(1/x) = x^{\KerSize - 1} p_{\Vector{a}}(x) p_{\Vector{a}}(1/x) + x^{\KerSize - 1} p_{\Vector{b}}(x) p_{\Vector{b}}(1/x).\]
To simplify notation, we denote the right-hand-side of the previous equation by $Q(x)$: 
\begin{equation}
Q(x) := x^{\KerSize - 1} p_{\Vector{a}}(x) p_{\Vector{a}}(1/x) + x^{\KerSize - 1}  p_{\Vector{b}}(x) p_{\Vector{b}}(1/x) \in \mathcal{P}_{2 \KerSize - 2}.
\label{eq:q}
\end{equation}
Since there is a 1-to-1 correspondence between polynomials in $\mathcal{P}_{\KerSize - 1}$ and vectors in $\mathbb{R}^{\KerSize}$. In this notation, our goal is to show that there exists a polynomial $p \in \mathcal{P}_{\KerSize - 1}$ such that:
\begin{equation}
\label{eq:main}
\text{(to show)}\quad\quad Q(x) = x^{\KerSize - 1} p(x) p(1/x).\quad\quad\quad\quad
\end{equation}
The coefficients of such a $p$ would then give us the desired $\bc\in\bR^\KerSize$ in Lemma~\ref{lemma:additivepropertyspecialcase}. 
\noindent The remainder of the proof boils down to constructing $p \in \mathcal{P}_{\KerSize - 1}$ such that \eqref{eq:main} is satisfied. 

\subsubsection{Proving the polynomial representation version of the lemma statement} \label{subsec:proofinpolyrep}

The first property of $Q(x)$ that we  leverage is that $Q(x)$ is a \textit{palindromic polynomial} of order $2\KerSize-2$ (\ie its coefficients zero padded to degree $2\KerSize-2$ forms a palindrome) with  real coefficients. To see that our $Q(x)$ is palindromic, note any polynomial $Q(x)\in\mathcal{P}_{2\KerSize-2}$ is a palindromic polynomial if and only if $Q(x)=x^{2\KerSize-2}Q(1/x)$, which is satisfied by our definition in  \eqref{eq:q}. 

At first glance, it would appear that \eqref{eq:main} follows immediately from standard properties of palindromic polynomials with real coefficients whose complex roots are known to come in reciprocal pairs as $(\alpha,1/\alpha)$.\footnote{This is a standard fact: e.g., see \url{https://en.wikipedia.org/wiki/Reciprocal_polynomial}.} However, we cannot obtain eq. \eqref{eq:main} from the palindromic property alone. To see this, consider the following example:
\begin{example}
\label{ex:pal}
Consider the palindromic polynomial with real coefficients $x^2 + 1 = i x (x -i)(1/x - i)$. This polynomial is not expressible as $x p'(x) p'(1/x)$ for any \textbf{real} polynomial $p'$.
\end{example}

The proof of \eqref{eq:main} thus must leverage further structure of $Q(x)$, which we will ultimately extract through examining the roots of $Q(x)$. Using that $\mathbb{C}$ is algebraically closed, we can factor $Q(x)$ into a polynomial $c_Q \prod_{i} (x - \alpha_i)$ with exactly $2 \KerSize - 2$ roots where the $\alpha_i$ need not be distinct. To show \eqref{eq:main}, it suffices to show that $x^{\KerSize - 1} p(x)\cdot p(1/x)$ has roots (with multiplicities) given by the multi-set $S_Q = \left\{\alpha_i\right\}$ and has the same leading coefficient $c_Q$. Drawing upon this formulation, we will construct $p$ implicitly by the multi-set $S_p$ of its roots (with multiplicities) and its leading  coefficient $c_p \neq 0$: that is, so that $p = c_p \prod_{\alpha \in S_p} (x - \alpha)$. \newline

\noindent The remainder of the proof is structured as follows:
\begin{compactenum}[\bf Step 1.]
\item We establish key properties of the multi-set $S_Q$ of $2\KerSize-1$ complex roots of $Q(x)$.
\item We then construct the roots $S_p$ of $p$ from exactly half of elements in $S_Q$.  
\item We show that $S_p$ is the ``right" set of roots in the sense that with any nonzero choice of leading coefficient $c_p$, the resulting polynomial $p$ is of degree at most $\KerSize-1$, and the multi-set of roots of $x^{K-1}p(x)p(1/x)$ match $S_Q$.
\item Finally, we choose a real leading coefficient $c_p$ of $p(x)$ to get the  desired property in \eqref{eq:main}.
\end{compactenum}

\subsection*{Step 1. Properties of the complex roots of $Q(x)$ }
Before we construct $S_p$ and $c_p$, it is helpful to establish properties of the multi-set of roots $S_Q$.  
\begin{enumerate}
    \item{(P1)} For every root $\alpha$ with multiplicity $m$, $\conj{\alpha}$ is a root and has multiplicity $m$. 
    \item{(P2)} If $\alpha \neq 0$ is a root with multiplicity $m$, then $1 / \alpha$ is a distinct root with multiplicity $m$. 
    \item{(P3)}  If $\alpha$ such that $|\alpha| = 1$ is a root, then $\alpha$ has even multiplicity. 
\end{enumerate}
The first two properties (P1) and (P2) follow from the fact that $Q(x)$ is a palindromic polynomial with real coefficients. In particular, we see that (P1) follows from the fact that $Q(x)$ has real coefficients so that the roots come in conjugate pairs, and (P2) follows from standard properties of palindromic polynomials.\footnote{This is a standard fact: e.g., see \url{https://en.wikipedia.org/wiki/Reciprocal_polynomial}.}

The last property, (P3), uses deeper aspects of the structure of $Q(x)$. In particular, it uses that $Q(x)$ is the sum of polynomials of the form $x^{\KerSize - 1} p'(x) p'(1/x)$ where $p' \in \mathcal{P}_{\KerSize - 1}$ has real coefficients, rather than just an arbitrary palindromic polynomial. To see this, let's return to Example \ref{ex:pal} and observe that $x^2 + 1$ does not satisfy (P3) since its roots are $i$ and $-i$, each with odd-multiplicity of $1$. Hence, we use further structure of $Q(x)$ and we show: 
\begin{lemma}
\label{lemma:multiplicity}
Consider vectors $\Vector{a}, \Vector{b} \in \mathbb{R}^{\KerSize}$, and let $p_{\Vector{a}}, p_{\Vector{b}} \in \mathcal{P}_{\KerSize -1}$ be their polynomial representation. If $\alpha \in \mathbb{C}$ such that $|\alpha| = 1$ is a root of $p_{\Vector{a}}(x) (x^{\KerSize - 1} p_{\Vector{a}}(1/x)) + p_{\Vector{b}}(x) (x^{\KerSize - 1} p_{\Vector{b}}(1/x))$, then $\alpha$ has even multiplicity.  
\end{lemma}
\begin{proof}
Suppose that $\alpha$ is a root of $Q(x)$ and $|\alpha| = 1$. We see that 
\begin{align*}
    0 &= Q(\alpha) \\
    &= p_{\Vector{a}}(\alpha) (\alpha^{\KerSize - 1} p_{\Vector{a}}(1/\alpha)) + p_{\Vector{b}}(\alpha) (\alpha^{\KerSize - 1} p_{\Vector{b}}(1/\alpha)) \\
    &=  p_{\Vector{a}}(\alpha) (\alpha^{\KerSize - 1} p_{\Vector{a}}(\conj{\alpha})) + p_{\Vector{b}}(\alpha) (\alpha^{\KerSize - 1} p_{\Vector{b}}(\conj{\alpha})) \\
    &= \alpha^{\KerSize - 1} \left( p_{\Vector{a}}(\alpha) \conj{p_{\Vector{a}}(\alpha))} +p_{\Vector{b}}(\alpha) \conj{p_{\Vector{b}}(\alpha))}\right) \\
    &=  \alpha^{\KerSize - 1} \left( |p_{\Vector{a}}(\alpha)|^2 + |p_{\Vector{b}}(\alpha)|^2 \right).
\end{align*}
Since $\alpha \neq 0$, this means that $|p_{\Vector{a}}(\alpha)|^2 + |p_{\Vector{b}}(\alpha)|^2 = 0$. Thus, $p_{\Vector{a}}(\alpha) = 0$ and $p_{\Vector{b}}(\alpha) = 0$. Now, it suffices to show that $\alpha$ is a root with even multiplicity in  $p_{\Vector{a}}(x) (x^{\KerSize - 1} p_{\Vector{a}}(1/x))$ and in $p_{\Vector{b}}(x) (x^{\KerSize - 1} p_{\Vector{b}}(1/x))$. 

We show that $\alpha$ has even multiplicity in $p_{\Vector{a}}(x) (x^{\KerSize - 1} p_{\Vector{a}}(1/x))$ (an analogous argument shows this for $p_{\Vector{a}}(x) (x^{\KerSize - 1} p_{\Vector{a}}(1/x))$).
Suppose that $\alpha$ has  multiplicity $m$ in $p_{\Vector{a}}(x)$. 
Since $p_{\Vector{a}}(x)$ has real coefficients, we know that $\conj{\alpha}$ is a root of $p_{\Vector{a}}(x)$ with multiplicity $m$. We also know that $1/\alpha = \conj{\alpha}$ is a root with multiplicity $m$ of $x^{k - 1} p_{\Vector{a}}(1/x))$. Since $x^{\KerSize - 1} p_{\Vector{a}}(1/x))$ has real coefficients, we know that $\conj{\conj{\alpha}} = \alpha$ is a root with multiplicity $m$ of $x^{k - 1} p_{\Vector{a}}(1/x))$. This means that $\alpha$ has multiplicity $2m$ in $p_{\Vector{a}}(x) (x^{\KerSize - 1} p_{\Vector{a}}(1/x))$ as desired. 
\end{proof}
We note that (P3) follows immediately from Lemma \ref{lemma:multiplicity}.

\subsection*{Step 2: Constructing the roots of $S_p$ }
We now construct the multi-set of roots $S_p$. To do this, we begin by constructing the \textit{nonzero} roots in $S_p$, and then we add in the zero roots with the appropriate multiplicities at the end. 

\vspace{3pt}
\noindent\textit{\underline{Constructing the nonzero roots.}}
The high-level intuition for the construction is that we ultimately need the nonzero roots of $x^{\KerSize - 1} p(x) p(1/x)$ for a degree $\KerSize-1$ polynomial $p\in\mathcal{P}_{K-1}$ to exactly match the nonzero roots in $S_Q$. Since the nonzero roots of $x^{K-1}p(1/x)$ are the reciprocals of the roots of $p(x)$ (with multiplicity preserved), and since polynomials with real coefficients have roots in conjugate pairs, we wish to divide the real nonzero roots of $Q(x)$ into disjoint pairs $(\alpha, 1/\alpha)$ (so that $p(x)$ has root $\alpha$ and $x^{K-1}p(1/x)$ has root $1/\alpha$) and the complex roots into disjoint quadruples $(\alpha, \conj{\alpha}, 1/\alpha, 1/\conj{\alpha})$ (so that $p(x)$ has roots $\alpha$ and $\conj{\alpha}$ and $x^{K-1}p(1/x)$ has roots $1/\alpha$ and $1/\conj{\alpha}$).

Let's formalize this high level argument by constructing an undirected graph $G$ where the vertices are the {nonzero} roots in $S_Q$ (a root with multiplicity $m$ corresponds to $m$ separate vertices). We will now add edges such that $G$ forms a bipartite graph with following properties : (a) each edge is of form $(\alpha,1/\alpha)$ without any self-loop; and (b) each vertex is connected to one and exactly one other vertex (\ie all edges are disjoint). 

The edges are defined as follows. 
When $\alpha \neq \pm 1$ is a nonzero root of $Q(x)$, we have ${1}/{\alpha}\neq \alpha$ and it follows from (P2) that $\alpha$ and ${1}/{\alpha}$ will have same multiplicity, say $m$. Thus, we can create $m$ disjoint edges by connecting each vertex with value $\alpha$ to exactly one vertex with value $1/\alpha$. This ensures that the subgraph of $G$ with values $\alpha$ or $1/\alpha$ and respects our desired construction. 
Using (P3), we can also handle roots $\alpha=\pm1$ in a similar construction. In this case, we will have have even number of vertices with value $\alpha$ and hence can form non-self-loop edges of the form $(\alpha, \alpha)$ such that all the edges are disjoint.

Having constructed $G$ as above, let $G^{\text{real}}$ be the subgraph of $G$ consisting of vertices corresponding to real roots, and let $G^{\text{complex}}$ be the subgraph of $G$ consisting of vertices corresponding to roots with nonzero imaginary part. It is easy to see that $G = G^{\text{real}} \cup G^{\text{complex}}$ and these graphs are disjoint. 

Let us now use these graphs to construct the set of nonzero roots in $S_p$, which will contain half of the vertices in $G$.  
For edges of the form $(\alpha,{1}/{\alpha})$ in $G^{\text{real}}$, we add one vertex from each edge in $G^{\text{real}}$ to $S_p$. For $G^{\text{complex}}$, we first pair up the edges in the subgraph $G^{\text{complex}}$ as follows.
When $|\alpha| \neq 1$, we can pair up the edges $(\alpha, 1/\alpha)$ and $(\conj{\alpha}, 1/\conj{\alpha})$ so that the pairs are disjoint (follows from (P1)) and further  no edge is paired with itself (since $1/\alpha \neq \conj{\alpha}$ when $|\alpha\neq1$).  
When $|\alpha| = 1$, we have $1/\alpha=\bar{\alpha}$. Here we use the fact that $\alpha$ and $\bar{\alpha}$ have same even multiplicity (see (P3)), say $2m$. Thus, we can again pair up $m$ edges $(\alpha,\bar{\alpha})\equiv(\alpha, 1/\alpha)$ with $m$ distinct edges $(\alpha,\bar{\alpha})\equiv(\conj{\alpha}, 1/\conj{\alpha})$ in such a way that all the pairs continue to be disjoint  and no edge is paired with itself. 
We have thus paired all edges in $G^{\text{complex}}$ as $(\alpha, 1/\alpha)$ and $(\conj{\alpha}, 1/\conj{\alpha})$ such that no edge is paired with itself and all pairings are disjoint. 
Now, from each distint pair of edges $(\alpha, 1/\alpha)$ and $(\conj{\alpha}, 1/\conj{\alpha})$, we add $\alpha$ from first edge and $\bar{\alpha}$ from second edge to $S_p$. 
\vspace{3pt}

\noindent\textit{\underline{{Adding the zero roots.}}}
To construct $S_p$, all that remains is to determine the multiplicity of the zero roots. Let $m$ be the multiplicity of $0$ in $S_Q$. We then simply add $m$ copies of $0$ to $S_p$. 

\vspace{3pt}
\noindent In summary, all the real nonzero roots of $S_Q$ were partitioned into disjoint pairs $(\alpha, 1/\alpha)$ (with $\alpha$ from each pair included in $S_p$) and the complex nonzero roots of $S_Q$ were partitioned into disjoint quadruples $(\alpha, \conj{\alpha}, 1/\alpha, 1/\conj{\alpha})$ (with $\alpha,\conj{\alpha}$ from each quadruple included in  $S_p$). We then included all zero entries of $S_Q$ in $S_p$. 

\subsection*{Step 3. Proof that $S_p$ is the correct multi-set of roots}
We will derive the leading coefficient $c_p$ of our desired polynomial shortly, but first show that $S_p$ is the right multi-set of roots for $p$.  We first prove that $S_p$ corresponds to the roots of a degree $\KerSize-1$ polynomial with real coefficients; we then show that with any nonzero leading coefficient $c_p$, the resulting polynomial $p(x)$ is such that the multiset of roots of $x^{\KerSize -1} p(x) p(1/x)$ is equal to $S_Q$. 
\newline

\noindent{\textit{\underline{Proof that $S_p$ corresponds to the roots of a polynomial in $\mathcal{P}_{\KerSize - 1}$.}}} For $S_p$ to be a valid multi-set of roots for a polynomial $p \in \mathcal{P}_{\KerSize - 1}$, we need to ensure that $S_p$ consists of {at most} $\KerSize - 1$ elements and that the complex roots come in conjugate pairs. 

The fact that the complex roots (with nonzero imaginary components) come in conjugate pairs follows from the construction of the graphs above. Recall the complex roots of $S_Q$ were partitioned into disjoint quadruples $(\alpha, \conj{\alpha}, 1/\alpha, 1/\conj{\alpha})$ so that $S_p$ has roots $\alpha$ and $\conj{\alpha}$ from each quadruple. This ensures that the roots comes up in conjugate pairs as desired.

We now show that $S_p$ consists of exactly $\KerSize - 1$ elements using the following root counting argument. Recall that we defined $Q(x)=x^{\KerSize-1}p_{\ba}(x)p_{\ba}(1/x)+x^{\KerSize-1}p_{\bb}(x)p_{\bb}(1/x)$ to be a palindromic polynomial of order $2\KerSize-2$, \ie for $d=0,1,\ldots 2\KerSize-2$, coefficient of $x^{d}$ is the same as coefficient of $x^{2\KerSize-2-d}$. Thus, if $Q(x)$ has a zero root with multiplicity $m$, \ie the coefficients of $x^0,x^1,x^2,\ldots x^{m-1}$ are zero, then the palindromic property would ensure that the coefficients of $x^{2\KerSize-2},x^{2\KerSize-3},\ldots, x^{2\KerSize-1-m}$ are also zero. This in turn implies that the max degree of $Q(x)$ is of degree $2 \KerSize - 2 - m$. 
Thus, we see that $Q(x)$ has $2 \KerSize - 2 -m$ roots  including multiplicity, of which $m$ are zero roots and $2 \KerSize - 2 -2m$ nonzero roots. By the construction of $S_p$, we included exactly half of nonzero roots of $S_Q$ and exactly $m$ zero roots in $S_p$. Thus $S_p$ has $\KerSize-1-m$ nonzero entries and exactly $m$ zero entries leading to a total of $\KerSize-1$ roots.
 \newline 

\noindent{\textit{\underline{Proof that the multi-set of roots of $x^{\KerSize - 1} p(x) p(1/x)$ equals the multi-set $S_Q$.}}} We denote the monic polynomial given by the roots of $S_p$ as  $p_{\text{monic}}(x)=\prod_{\alpha\in S_p}(x-\alpha)$ and consider 
\[Q_1(x) := p_{\text{monic}}(x) x^{\KerSize - 1} p_{\text{monic}}(1/x).\] It suffices to show that the multi-set of roots of $Q_1(x)$ with multiplicities is equal to the multi-set $S_Q$. 

For the nonzero roots, we use the construction of $S_p$. Recall that the real roots of $S_Q$ were partitioned into disjoint pairs $(\alpha, 1/\alpha)$ (so that $p_{\text{monic}}(x)$ has root $\alpha$) and the complex roots of $S_Q$ were partitioned into disjoint quadruples $(\alpha, \conj{\alpha}, 1/\alpha, 1/\conj{\alpha})$ (so that $p_{\text{monic}}(x)$ has roots $\alpha$ and $\conj{\alpha}$). This, coupled with the fact that the nonzero roots of $x^{\KerSize-1}p_{\text{monic}}(1/x)$ are the inverses of the nonzero roots of $p_{\text{monic}}(x)$, means that the multi-set of nonzero roots of $Q_1(x)$ which is the union of the multi-set of nonzero roots in $p_{\text{monic}}(x)$ with the multi-set of nonzero roots of $x^{\KerSize-1}p_{\text{monic}}(1/x)$ is exactly equal to the multi-set of nonzero roots of $Q(x)$. 

For the zero roots, we simply need to show that $x^{\KerSize - 1} p_{\text{monic}}(1/x)$ has no roots that are $0$. Using the fact that the constant term of $x^{\KerSize - 1} p_{\text{monic}}(1/x)$ is equal to the coefficient of $x^{\KerSize - 1}$ in $p_{\text{monic}}(x)$, it suffices to show that the coefficient of $x^{\KerSize - 1}$ in $p_{\text{monic}}(x)$ is nonzero. The latter condition follows from the argument in the previous paragraph where we showed $S_p$ to have exactly $K-1$ roots (with multiplicity) and hence the degree of $p_{\text{monic}}(x)$ must be $\KerSize - 1$ with a nonzero coefficient of $x^{\KerSize - 1}$. We are now done as the zero roots of $Q(x)$ matches those of $p_{\text{monic}}(x)$ with $x^{\KerSize - 1} p_{\text{monic}}(1/x)$ adding no additional zero roots. 
 
This concludes the argument that the roots of $Q_1(x)$ with multiplicities matches those  of $Q(x)$.

\subsection*{Step 4. Constructing the leading coefficient of $p$}
Now, we need to just construct the leading coefficient of $p$. As above, let $p_{\text{monic}}(x) = \prod_{\alpha \in S_p} (x - \alpha)$ be the monic polynomial given by the roots of $S_p$, and consider $Q_1(x) = p_{\text{monic}}(x) x^{\KerSize - 1} p_{\text{monic}}(1/x)$. Since $Q_1(x)$ and $Q(x)$ have the same set of roots with multiplicities, we know that $Q(x)= \gamma \cdot Q_1(x)$ for some $\gamma \neq 0$. Let's take $c_p = \sqrt{\gamma}$. With this choice of $c_p$, we define $p(x) = c_p \prod_{\alpha \in S_p} (x - \alpha)$ and have established by construction that $Q(x)=x^{\KerSize-1}p(x)p(1/x)$. The final component is to show that $c_p$ is real so that $p(x) = c_p \prod_{\alpha \in S_p} (x - \alpha)$ is a polynomial with real coefficients. 

To show that $c_p$ is real, it suffices to show that $\gamma$ is positive. We show this as follows. 
Let  $\Vector{c}_{\text{monic}} \in \mathbb{R}^{\KerSize}$ denote the vector corresponding to the polynomial representation $p_{\text{monic}}(x)$ such that $Q_1(x) = p_{\text{monic}}(x) x^{\KerSize - 1} p_{\text{monic}}(1/x)$ is equivalent to the polynomial representation of $\Vector{c}_{\text{monic}} \star \Vector{c}_{\text{monic}}$ (using Fact~\ref{fact:polymultiplication}). Recall that $Q(x)$ is analogously equivalent to the polynomial representation of $\ba\Conv\ba +\bb\Conv\bb$. Further, note that the component at index $0$ of self-convolution operations satisfy, $(\ba\Conv\ba +\bb\Conv\bb)_0=\|\ba\|^2+\|\bb\|^2$ and $(\Vector{c}_{\text{monic}} \star \Vector{c}_{\text{monic}})_0=\|\Vector{c}_{\text{monic}}\|^2$. One can check that, these $0$-index components in turn appear as the coefficients of $x^{\KerSize-1}$ in $Q(x)$ and $Q_1(x)$, respectively. This along with $Q(x)=\gamma Q_1(x)$ implies that $\|\ba\|^2+\|\bb\|^2=\gamma \|\bc_{\text{monic}}\|^2$. Since the LHS is strictly positive (as $Q(x)\neq 0$ without loss of generality), we must have $\gamma>0$ as desired. 

\paragraph{Concluding Lemma \ref{lemma:additivepropertyspecialcase}.}
We have shown that $p(x) = c_p \prod_{\alpha \in S_p} (x - \alpha)$ is a polynomial of degree $\KerSize - 1$ with real coefficients. We have shown that $Q(x)$ and $x^{\KerSize - 1} p(x)\cdot p(1/x)$ have the same multi-set of roots and the same leading coefficients, thus proving the polynomial formulation of Lemma \ref{lemma:additivepropertyspecialcase} in  \eqref{eq:main}. This concludes the proof of Lemma \ref{lemma:additivepropertyspecialcase}. 

\subsection{Discussion of the proof technique}\label{appendix:prooftechnique}

We conclude with a discussion of the analysis and highlight the main parts of the proof. At the beginning of the section, we used the KKT conditions to show that it suffices to prove \ref{eq:desired}, an additive property about convolutions of for kernel size. We then showed that it suffices to prove a version of this statement for the sum of two such convolutions, i.e. Lemma \ref{lemma:additiveproperty}. We believe that this property could be of independent interest. 

The bulk of the proof boils down to proving Lemma \ref{lemma:additiveproperty} in the special case of $\Dim = 2 \KerSize - 1$, \ie Lemma \ref{lemma:additivepropertyspecialcase}. Proving Lemma \ref{lemma:additivepropertyspecialcase} was the core technical contribution in this section. Since $\Vector{c}$ does not necessarily always a clean closed-form solution as a function of $\Vector{a}$ and $\Vector{b}$, we needed to construct $\Vector{c}$ \textit{implicitly}. The polynomial representation of convolutions enabled us to implicitly construct $\Vector{c}$ via its roots. To construct $\Vector{c}$ and ensure that the corresponding polynomial representation $p$ had real coefficients, we needed to leverage the structure of the polynomial representation $Q(x)$ of $\Vector{a} \Conv \Vector{a} + \Vector{b} \Conv \Vector{b}$ beyond its palindromic structure. (This additional property was proven in Lemma \ref{lemma:multiplicity}) With this structure, we can factor $Q(x)$ and partition its roots in order to construct the roots of $p$.

\section{Proofs for Section~\ref{sec:explicitbounds}: Induced regularizer in special cases}\label{appendix:proofsexplicitbounds}
For a single channel convolutional network (\ie $\COut=1$), we denote the weights in the first and second layer as $\U\in\bR^\KerSize$ and $\V\in\bR^\Dim$, respectively. Recall that the $\Dim$ dimensional discrete Fourier transform of the weights $\U,\V$ and the linear predictor $\w\in\bR^D$ are denoted as $\hat{\U}=\F_\KerSize\U, \hat{\V}=\F\V, \hat{\w}=\F\w$, respectively. Moreover, the Fourier transform is normalized to be unitary such that for any $\b{z}\in\bR^{\KerSize}$, $\|\b{z}\|=\|\hat{\b{z}}\|$ and $\F\conj{\F}^\top=\conj{\F}{\F^\top}=\Identity$. 
Thus, for all the quantities of interest, we use the $\ell_2$ norm in signal domain interchangeably with $\ell_2$ norm in the Fourier domain, \eg $\|\U\|=\|\hat{\U}\|$, $\|\V\|=\|\hat{\V}\|$, and $\|\w\|=\|\hat{\w}\|$. 

In the  following proofs, we use the formulation of $\RkCout{\ParVec}{\KerSize}{\COut}$ in eq. \eqref{eq:mixedrepresentation} for the case of $\COut=1$ as:
\begin{equation*}
\RkCout{\ParVec}{\KerSize}{1} = \min\limits_{\U \in \mathbb{R}^{\KerSize}, \V \in \mathbb{R}^{\Dim}} \|\hat{\U}\|^2 + \|\hat{\V}\|^2 \quad\st\quad  \hat{\U}\odot \hat{\V} = \hat{\ParVec}.
\end{equation*}

\subsection{Proof of Lemma \ref{lemma:kersize1}}
\explicitone*
\begin{proof}
This statement is trivially true for $\ParVec = 0$, so it suffices to show this for $\ParVec \neq 0$.  When the $\KerSize=1$, the first layer weight $\U\in \mathbb{R}^{1 \times 1}$ is a scalar. Let this scalar be $\U=u \neq 0$. We then have $\hat{\U} = \frac{1}{\sqrt{\Dim}} [u, u, \ldots, u]$. Since $\hat{\U} \Entrywise \hat{\V} = \hat{\ParVec}$, we  have $\hat{\V} = \frac{\sqrt{\Dim}}{u} \hat{\ParVec}$. This means that $\RkCout{\w}{1}{1}=\min_{u} \|\hat{\U}\|^2 + \|{\hat{\V}}\|^2 = \min_u u^2 + \frac{\Dim}{u^2} \norm{\hat{\ParVec}}_2^2$. By using the  AM-GM inequality ($a^2+b^2\ge2ab$),   this is at most $2 \sqrt{\Dim}\norm{\hat{\ParVec}}_2$. Moreover, we can pick $u^2=\sqrt{\Dim}\norm{\hat{\w}}_2$ to achieve equality. 
\end{proof}

\subsection{Induced Regularizer for $\KerSize = 2$}\label{appendix:kersize2}

\explicittwo*
\begin{proof}[Proof of Lemma \ref{lemma:kersize2}]
We first note that for any $\U\in\bR^\KerSize,\V\in\bR^\Dim$ we can re-scale the norms so that $\|\U\|=\|\V\|$ while satisfying the constraints of $\hat{\U}\odot\hat{\V}=\hat{\w}$  in the definition of $\RkCout{\ParVec}{\KerSize}{1}$. Further, such a scaling would be optimal for minimizing the $\ell_2$ norm of weights based on AM-GM inequality that $\|\hat{\U}\|^2 + \|{\hat{\V}}\|^2\ge 2\|\hat{\U}\|\cdot\|\hat{\V}\|$. Thus, in the rest of the proof, we consider the following equivalent formulation of $\RkCout{\ParVec}{\KerSize}{1}$ as: \begin{equation}\label{eq:prodR}\RkCout{\ParVec}{\KerSize}{1} = \min\limits_{\U \in \mathbb{R}^{\KerSize}, \V \in \mathbb{R}^{\Dim}} 2\|{\U}\|\cdot\|\hat{\V}\| \quad \st \hat{\U}\odot \hat{\V} = \hat{\ParVec}\end{equation}

We see that for any $\U,\V$ satisfying the constraint in the above equation, we have: $\forall_{d \in \text{supp}(\hat{\ParVec})}$, it holds that $\Index{\hat{\V}}{d} = \frac{\Index{\hat{\ParVec}}{d}}{\Index{\hat{\U}}{d}}$ (where $\text{supp}(\hat{\ParVec})=\crl*{d\in[\Dim]:|\hat{\w}|\neq0}$). Moreover, at an optimal solution, it is easy to see that $\Index{\hat{\V}}{d} = 0\iff\Index{\hat{\w}}{d} = 0$. 

Let $\U = [c_0, c_1]$. This means that the objective can be written as
\begin{equation}\label{eq:lemmaK2-obj}2 \|\U\| \|{\hat{\V}}\| = 2 \norm{\U} \sqrt{\sum_{d=0}^{\Dim -1} |\Index{\hat{\V}}{d}|^2} = 2 \norm{\U} \sqrt{\sum_{d \in \text{supp}(\hat{\ParVec})} |\hat{\V}[d] |^2} =  2 \sqrt{c_0^2 + c_1^2}\sqrt{\sum_{d \in \text{supp}(\hat{\ParVec})} \frac{|\hat{\ParVec}[d]|^2}{|\hat{\U}[d]|^2}}.\end{equation}
We write the second term in terms of the signal domain representation of $c_0$ and $c_1$. We see that the Fourier transform of $\U$ is given by $\hat{\U}[d]=\frac{1}{\sqrt{\Dim}} \left(c_0 + c_1 \e^{-2 \pi i d/ \Dim}\right) = \left(c_0 + c_1 \cos(-2 \pi i d/ \Dim )\right) + i \left(c_1 \sin(-2 \pi i d/ \Dim )\right)$. We thus have that:
\begin{equation}\label{eq:lemmak2-fu}
    \begin{split}
  |\hat{\U}[d]|^2 &= \frac{1}{\Dim} \left[\left(c_0 + c_1 \cos(-2 \pi i d/ \Dim )\right)^2 + \left(c_1 \sin(-2 \pi i d/ \Dim )\right)^2\right]\\
  &= \frac{1}{\Dim} \left(c_0^2 + c_1^2 + 2 c_0 c_1 \cos(-2 \pi i k/ \Dim )\right). \\
  &= \frac{1}{\Dim} (c_0^2 + c_1^2) \left(1 + \frac{2 c_0 c_1}{c_0^2 + c_1^2} \cos(-2 \pi i d/ \Dim )\right). 
\end{split}
\end{equation}
Let $\alpha = \frac{2 c_0 c_1}{c_0^2 + c_1^2}$. Plugging eq. \eqref{eq:lemmak2-fu} back into the objective in eq. \eqref{eq:lemmaK2-obj} and using $\cos(z)=\cos(-z)$, we get that for any $\U,\V$ satisfying the constraints in the computation of $\RkCout{\w}{2}{1}$, the objective is in the desired formulation:
\begin{equation}\label{eq:lemmaK2-obj2}2 \norm{\U} \norm*{\hat{\V}} = 2 \sqrt{D} \sqrt{\sum_{d \in \text{supp}(\hat{\ParVec})} \frac{\left|\hat{\ParVec}[d]\right|^2}{1 + \alpha \cos(2 \pi i d/ \Dim )}}. \end{equation}

Let us now consider the domain of $\alpha$, which is the only unknown in the above equation. Observe that for any $c_0,c_1$, $\frac{2c_0 c_1}{c_0^2 + c_1^2} \in [-1,1]$. Moreover, any $\alpha \in [-1,1]$ be realized by some values of $c_0$ and $c_1$. Thus all $\alpha \in [-1,1]$ are valid. Here we further remark that the denominator in eq.~\eqref{eq:lemmaK2-obj2} is zero if and only if $\hat{\U}[d]=0$ for any $d\in\Dim$. However, this can only happen if $\hat{\w}[d]=0$ as otherwise the constraints $\hat{\U}\odot\hat{\V}=\w$ is not satisfied for any $\V$.  We can thus, minimize the RHS of eq.~\eqref{eq:lemmaK2-obj2} over $\alpha \in [-1,1]$ to obtain $\RkCout{\w}{2}{1}$. 

If we include the terms corresponding to $d\notin\text{supp}(\hat{\w})$ in the summation in eq.~\eqref{eq:lemmaK2-obj2}, there is a technical condition than can lead to $0/0$ terms in end cases of $\alpha=1$ (when $\hat{\w}[0]=0$) and $\alpha=-1$ (when $\hat{\w}[D/2]=0$). To avoid this technicality, we consider the infimum over $\alpha\in(-1,1)$ rather than minimum over $\alpha\in[-1,1]$. This is equivalent because the expression is continuous on the set of $\alpha$ on which it is well-defined. 
This completes the proof.
\end{proof}

\section{Remaining proofs of results in Section~\ref{sec:SDP}}\label{appendix:sdp-other-proofs}

\subsection{Proof of Corollary \ref{cor:norm}}
\cornorm*
\noindent Corollary \ref{cor:norm} follows from Theorem \ref{thm:tightness}. 
\begin{proof}[Proof of Corollary \ref{cor:norm}]
It suffices to establish the scalar multiplication property, the triangle inequality, and point separation. 

\paragraph{Scalar multiplication.} Let $\gamma \in \mathbb{R}$. By definition, we see that  
\[\RkCout{\gamma \ParVec}{\KerSize}{\COut}=\min_{\U,\V}\|\U\|^2+\|\V\|^2\quad\st\quad\diag(\hat{\U}\hat{\V}^\top)=\gamma\hat{\w}.\]
Let's do a change of variables ${\U}\gets \frac{1}{\sqrt{|\gamma|}}{\U}, {\V}\gets \frac{1}{\text{sign}(\gamma)\sqrt{|\gamma|}}{\V}$ to see that 
\[\RkCout{\gamma \ParVec}{\KerSize}{\COut}= |\gamma| \left[\min_{\U,\V}\|\U\|^2+\|\V\|^2\quad\st\quad\diag(\hat{\U}\hat{\V}^\top)=\hat{\w}\right] = |\gamma| \RkCout{\ParVec}{\KerSize}{\COut}\]
as desired.

\paragraph{Triangle inequality.} It suffices to show that if $\ParVec = \ParVec_1 + \ParVec_2$, then $\RkCout{\ParVec}{\KerSize}{\COut} \le \RkCout{\ParVec_1}{\KerSize}{\COut} + \RkCout{\ParVec_2}{\KerSize}{\COut}$. By Theorem \ref{thm:tightness}, it suffices to show that  $\RSDPk{\ParVec}{\KerSize} \le \RkCout{\ParVec_1}{\KerSize}{\COut} + \RkCout{\ParVec_2}{\KerSize}{\COut}$. Suppose that $\U_1$ and $\V_1$ are optimal solutions to $\RkCout{\ParVec_1}{\KerSize}{\COut}$, and suppose that $\U_2$ and $\V_2$ are optimal solutions to $\RkCout{\ParVec_2}{\KerSize}{\COut}$. If we define:
\begin{align*}
  \SDPVar_1 &= \left[\begin{array}{cc} \Matrix{U}_1\Matrix{U}_1^{\top} & \Matrix{U}_1\Matrix{V}_1^{\top} \\ \Matrix{V}_1 \Matrix{U}_1^{\top} & \Matrix{V}_1\Matrix{V}_1^{\top} \end{array}\right]  \\
    \SDPVar_2 &= \left[\begin{array}{cc} \Matrix{U}_2\Matrix{U}_2^{\top} & \Matrix{U}_2\Matrix{V}_2^{\top} \\ \Matrix{V}_2 \Matrix{U}_2^{\top} & \Matrix{V}_2\Matrix{V}_2^{\top} \end{array}\right]  \\
    \SDPVar &= \SDPVar_1 + \SDPVar_2,
\end{align*}
then we see that the SDP objective $\langle \SDPVar, \Identity \rangle = \langle  \SDPVar_1, \Identity\rangle + \langle  \SDPVar_2, \Identity\rangle = \RkCout{\ParVec_1}{\KerSize}{\COut} + \RkCout{\ParVec_2}{\KerSize}{\COut}$. Moreover, $\SDPVar$ is a feasible solution to \eqref{opt:SDP} for $\ParVec$. This shows $\RSDPk{\ParVec}{k} \le \RkCout{\ParVec_1}{\KerSize}{\COut} + \RkCout{\ParVec_2}{\KerSize}{\COut}$ as desired. 

\paragraph{Point separation.} Notice that $\RkCout{\ParVec}{\KerSize}{\COut} = 0$, then there exist $\U$ and $\V$ such that $\norm{\U}^2 + \norm{\V}^2 = 0$. This means that $\U = 0$ and $\V = 0$, which means that $\ParVec = 0$ as desired. Moreover, if $\ParVec = 0$, then it's clear that $\RkCout{\ParVec}{\KerSize}{\COut} = 0$.  
\end{proof}

\subsection{The dual formulation of SDP}

In order to analyze the SDP formulation, we consider the dual. We use the formulation of the dual variable in \ref{appendix:tightness-main} as $\Vector{\lambda} \in \mathbb{C}^{\Dim}$. In this form, the dual can be expressed as:
\[
\begin{aligned}
    \max_{\Vector{\lambda} \in \mathbb{C}^\Dim} \quad & \Re\left(\langle \Vector{\lambda}, \hat{\ParVec} \rangle\right)   \\
    \textrm{s.t.} \quad & \left[\begin{array}{cc} \Matrix{0}_K & \Matrix{F}_{\KerSize}^{\top} \conj{\Matrix{\Lambda}} \Matrix{F}  \\ \conj{\Matrix{F}} \Matrix{\Lambda} \conj{\Matrix{F}}_{\KerSize}   & \Matrix{0}_{\Dim} \end{array}\right] \preccurlyeq \Identity.
\end{aligned}
\]
To simplify the objective $\Re\left(\langle \Vector{\lambda}, \hat{\ParVec}\right)$, notice that 
the phases of $\Vector{\lambda}$ can be set to align with $\hat{\ParVec}$ without affecting the constraint. Thus, we can set the objective to be $\left|\langle \Vector{\lambda}, \hat{\ParVec} \rangle\right| $. For convenience, we also expand out the conic constraint in vector form, and this reformulation incurs a factor of $2$ on the objective. We thus obtain the following equivalent formulation of the dual:
\begin{equation}
\label{eqn:dualfreqspace}
\begin{aligned}
    \max_{\Vector{\lambda} \in \mathbb{C}^\Dim} \quad & 2 \sum_{d=0}^{\Dim - 1} |\Vector{\lambda}[d]| \cdot |\ParVec[d]|    \\
    \textrm{s.t.} \quad& \forall \Vector{x} \in \mathbb{C}^{\KerSize}, \; \; \sum_{d=0}^{\Dim- 1} |\hat{\Vector{x}}[d]|^2 \cdot |\Vector{\lambda}[d]|^2 \le 1.
\end{aligned}
\end{equation}

We now show that strong duality holds for this SDP.

\begin{proposition}
\label{prop:strongduality}
The SDP in \eqref{opt:SDP} satisfies strong duality. 
\end{proposition}
\begin{proof}
 To show strong duality, it suffices to show Slater's condition. We just need to find a solution $\Vector{\lambda} \in \mathbb{C}^d$ where the inequality constraint is not tight. That is, we need to find $\Vector{\lambda}$ such that $\forall_{\Vector{x} \in \mathbb{C}^{\KerSize}}, \; \; \sum_{d=0}^{\Dim- 1} |\hat{\Vector{x}}[d]|^2 \cdot |\Vector{\lambda}[d]|^2 < 1$. Let's take $\Vector{\lambda} = [1/2, 0, 0, \ldots, 0]$. Notice that $\sum_{d=0}^{\Dim- 1} |\hat{\Vector{x}}[d]|^2 \cdot |\Vector{\lambda}[d]|^2 = 0.5 |\hat{\Vector{x}}[0]|^2 \le 0.5 < 1$, as desired.
\end{proof}

With the dual, along the fact that $\RSDPDualk{\ParVec}{\KerSize}$ is a norm, we are equipped to prove general upper and lower bounds on the induced regularizer as well as sharper bounds for patterned vectors. 

\subsection{Proof of Lemma \ref{lemma:bounds}}
\bounds*


\begin{proof}
The bounds of $2 \norm{\hat{\ParVec}}_1$ and $2 \sqrt{\Dim} \norm{\hat{\ParVec}}_2$ follow in a straightforward way from Lemma \ref{lemma:kersizeD} and Lemma~\ref{lemma:kersize1} coupled with Theorem \ref{thm:tightness}. 

The lower bound of $2 \norm{\hat{\ParVec}}_1$ follows as:
$2 \norm{\hat{\ParVec}}_1 \overset{(a)}= \RkCout{\ParVec}{\Dim}{1} \overset{(b)}= \RkCout{\ParVec}{\Dim}{\COut} \overset{(c)}\le \RkCout{\ParVec}{\KerSize}{\COut},$ 
where (a) follows from Lemma \ref{lemma:kersize1}, (b) follows from Theorem \ref{thm:tightness}, and (c) follows from Remark \ref{remark:kdecrease}.
Similarly, the upper bound of $2 \sqrt{\Dim} \norm{\hat{\ParVec}}_2$ follows as: $2 \sqrt{\Dim} \norm{\hat{\ParVec}}_2 \overset{(a)}= \RkCout{\ParVec}{1}{1} \overset{(b)}=  \RkCout{\ParVec}{1}{\COut} \ge \RkCout{\ParVec}{\KerSize}{\COut},$ 
where (a) follows from Lemma \ref{lemma:kersize1} and (b) follows from Theorem \ref{thm:tightness}, and (c) follows from Remark \ref{remark:kdecrease}

The bulk of the proof lies in showing the lower bound of $2\sqrt{\frac{\Dim}{\KerSize}} \norm{\hat{\ParVec}}_2$, and an upper bound of $2 \sqrt{\ceil{\frac{\Dim}{\KerSize}}} \norm{\hat{\ParVec}}_1$. We first prove the lower bound, and then we prove the upper bound. 

\paragraph{\bf Proof of the lower bound $2\sqrt{\frac{\Dim}{\KerSize}} \norm{\hat{\ParVec}}_2$.} We prove that $\RkCout{\ParVec}{\KerSize}{\COut} \ge 2\sqrt{\frac{\Dim}{\KerSize}} \norm{\hat{\ParVec}}_2$. It suffices to consider a dual feasible vector to eq. \eqref{eqn:dualfreqspace} that achieves an objective $2\sqrt{ \frac{\Dim}{\KerSize}} \norm{\ParVec}_2$. We consider 
\[\Vector{\lambda} = \frac{\ParVec}{\norm{\ParVec}}\sqrt{ \frac{\Dim}{\KerSize}}.\] 
We see that the objective is equal to 
\[2 \sum_{d=0}^{\Dim - 1} |\Vector{\lambda}[d]| |\ParVec[d]| = 2\sqrt{ \frac{\Dim}{\KerSize}} \norm{\ParVec}_2, \] 
as desired. It thus suffices to show that $\Vector{\lambda}$ satisfies $\sum_{d=0}^{\Dim- 1} |\hat{\Vector{x}}[d]|^2 \cdot |\Vector{\lambda}[d]|^2 \le 1$ for all $\Vector{x} \in \mathbb{C}^d$ such that  $\norm{\Vector{x}}_2 \le 1$. Using Holder's inequality, we can bound: \[\sum_{d=0}^{\Dim- 1} |\hat{\Vector{x}}[d]|^2 \cdot |\Vector{\lambda}[d]|^2 \le \left(\max_{0 \le d \le \Dim - 1} |\hat{\Vector{x}}[d]|^2\right) \left(\sum_{d=0}^{\Dim- 1} |\Vector{\lambda}[d]|^2 \right) = \left(\max_{0 \le d \le \Dim - 1} |\hat{\Vector{x}}[d]|^2\right) \norm{\Vector{\lambda}}_2^2.\] We can bound the first term by: 
\[|\hat{\Vector{x}}[d]| = \frac{1}{\sqrt{\Dim}} \left|\sum_{k=0}^{\KerSize -1} \Vector{x}[k] \e^{-2\pi i k d/\Dim}\right| \le \frac{1}{\sqrt{\Dim}} \sum_{k=0}^{\KerSize -1} |\Vector{x}[k] \e^{-2\pi i k d/\Dim}| = \frac{\norm{x}_1}{\sqrt{\Dim}} \le \frac{\sqrt{\KerSize}}{\sqrt{\Dim}}.\] Moreover, we see that $\norm{\Vector{\lambda}} = \sqrt{ \frac{\Dim}{\KerSize}}$. This means that 
$\left(\max_{0 \le d \le \Dim - 1} |\hat{\Vector{x}}[d]|^2\right) \norm{\Vector{\lambda}}_2^2 \le 1$, as desired. 

\paragraph{\bf Proof of the upper bound $2 \sqrt{\ceil{\frac{\Dim}{\KerSize}}} \norm{\hat{\ParVec}}_1$.} We prove that $\RkCout{\ParVec}{\KerSize}{\COut} \le 2 \sqrt{\ceil{\frac{\Dim}{\KerSize}}} \norm{\hat{\ParVec}}_1$. Our main ingredient is Corollary \ref{cor:norm} which tells us that $\RkCout{\ParVec}{\KerSize}{\COut}$ is a norm. We define $T = \ceil{\Dim / \KerSize}$ vectors $\ParVec_0, \ldots, \ParVec_{T-1} \in \mathbb{R}^{\Dim}$ where $\ParVec = \sum_{t=0}^{T-1} \ParVec_t$, and apply Corollary \ref{cor:norm} to obtain that:
\[\RkCout{\ParVec}{\KerSize}{\COut} \le \sum_{t=0}^{T-1} \RkCout{\ParVec_t}{\KerSize}{\COut}.\] 
These vectors are chosen that each $\RkCout{\ParVec_t}{\KerSize}{\COut}$ takes on a simple closed-form solution. 

In order to construct the vectors $\ParVec_t$, we consider $\hat{\Vector{q}} = \sqrt{\hat{\ParVec}}$, defined so that $\ParVec = \Vector{q}^{\downarrow} \Conv \Vector{q}$ and $\norm{\hat{\Vector{q}}}^2=\norm{{\Vector{q}}}^2=\norm{\hat{\w}}_1$. We define vectors $\Vector{r}_0, \ldots, \Vector{r}_{T-1} \in \mathbb{R}^{\Dim}$ such that $\sum_{t=0}^{T-1} \Vector{r}_t = \Vector{q}$ as follows. Roughly speaking these vectors consist of the disjoint subsets of the coordinates of $\Vector{q}$ corresponding to the $t^{\text{th}}$ block of size $\KerSize$. More formally, for $0 \le t \le T-1$, let $\Vector{r}_t$ be defined so that $\Vector{r}_t[l] = \Vector{q}[l]$ for $l \in [t \cdot \KerSize, \min((t+1) \cdot \KerSize - 1, \Dim - 1)]$, and $\Vector{r}_t[l] = 0$ otherwise.  Let $\ParVec_t =  \Vector{r}_t^{\downarrow} \Conv \Vector{q}$ for $0 \le t \le T - 1$. It is evident that $\sum_{t}\Vector{r}_t=\Vector{q}_t$ and hence $\sum_t\w_t=\w$. 

We now show that $\RkCout{\ParVec_t}{\KerSize}{\COut} \le 2 \norm{\Vector{r}_t} \norm{\Vector{q}}$. We show this by explicitly constructing solutions to \eqref{eq:inducedregularizer}, taking advantage of the fact that $\Vector{r}_t$ is effectively a vector in $\mathbb{R}^{\KerSize}$ that is zero-padded appropriately. This $\KerSize$ dimensional vector $\Vector{r}'_t \in \mathbb{R}^{\KerSize}$ is given by 
$\Vector{r}'_t[k] = \Vector{r}_t[(t \cdot \KerSize + k) \mod \Dim]$ for $0 \le k \le \KerSize-1$. Now, we wish to write $\ParVec_t$ as a convolution   $(\Vector{r}'_t)^{\downarrow} \Conv  \Vector{q}_t$, for some suitably chosen vector $\Vector{q}_t$. Since $\ParVec_t = \Vector{r}_t^{\downarrow} \Conv \Vector{q}$ and $\Vector{r}_t'$ is merely a circular shifted version of $\Vector{r}_t$, we can take $\Vector{q}_t \in \mathbb{R}^{\Dim}$ to be $\Vector{q}$ with the coordinates shifted appropriately. Now, we can rescale $\Vector{r}'_t$ and $\Vector{q}_t$ so that they have equal $\ell_2$ norms, and obtain the following vectors: $\left( \sqrt{\frac{\norm{\Vector{q}}_t}{\norm{\Vector{r}'_t}}} \right)\Vector{r}'_t$ and $\left(\sqrt{\frac{\norm{\Vector{r}'_t}}{\norm{\Vector{q}}_t}} \right)\Vector{{q}}_t$. These vectors are a feasible solution to eq. \eqref{eq:inducedregularizer} for $\RkCout{\ParVec_t}{\KerSize}{\COut}$ and achieve an objective of $ 2 \norm{\Vector{r}'_t} \norm{\Vector{q}_t} = 2 \norm{\Vector{r}_t} \norm{\Vector{q}}$, as desired. 

Using that $\RkCout{\ParVec_t}{\KerSize}{\COut} \le 2 \norm{\Vector{r}_t} \norm{\Vector{q}}$ for $0 \le t \le T-1$, we obtain the following bound on $\RkCout{\ParVec}{\KerSize}{\COut}$:
\[\RkCout{\ParVec}{\KerSize}{\COut}  \le 2 \norm{\Vector{q}} \sum_{t=0}^{T-1} \norm{\Vector{r}_t}.\]Now, notice that since the supports of $\Vector{r}_t$ for $0 \le t \le T-1$ are disjoint, $\sum_{t=0}^{T-1} \norm{\Vector{r}_t}^2 = \norm{\Vector{q}}^2 $. Applying AM-GM, this means that 
\[(\sum_{t=0}^{T-1} \norm{\Vector{r}_t})^2 \le T (\sum_{t=0}^{T-1} \norm{\Vector{r}_t}^2) = T \norm{\Vector{q}}^2.\] 
Thus, we have that 
\[\RkCout{\ParVec}{\KerSize}{\COut}  \le 2 \sqrt{T}  \norm{\Vector{q}}^2 = 2 \sqrt{T} \norm{\hat{\ParVec}}_1 = 2 \sqrt{\ceil{\frac{\Dim}{\KerSize}}} \norm{\hat{\ParVec}}_1 .\] 

\end{proof}

\subsection{Proof of Lemma \ref{lemma:patterned}}

\patterned*
\begin{proof}[Proof of Lemma \ref{lemma:patterned}]

We first prove an upper bound on $\RkCout{\ParVec}{\KerSize}{\COut}$, and then we prove a matching lower bound. In these proofs, let $\ParVec = \ParVec(\Pattern)$. Both of these proofs use the standard fact that $\hat{\ParVec}[(\Dim / P) \cdot p] = \sqrt{\frac{\Dim}{P}} \hat{\Pattern}[p]$ for $0 \le p \le P - 1$, and $\hat{\ParVec}[d] = 0$ if $(\Dim / P) \nmid d$. 

\paragraph{\textbf{Upper bound.}} For the upper bound, it suffices to consider the case of a single output channel. We explicitly construct a pair $(\Vector{u}, \Vector{v})$ where $\norm{\Vector{u}}^2 + \norm{\Vector{v}}^2$ achieves the desired objective. 

First, we consider the case where $\KerSize \le \SmallDim$. We construct $(\Vector{u}, \Vector{v})$  using an optimal solution $\Vector{u}_{\Pattern}$ and $\Vector{v}_{\Pattern}$ to eq. \eqref{eq:inducedregularizer} for $\RkCoutDim{\SmallDim}{\Pattern}{\KerSize}{1}$. We let $\Vector{u}$ be defined so that $\Vector{u}[k] = \sqrt{\frac{\Dim}{P}} \Vector{u}_{\Pattern}[k]$ for $ 0 \le k \le \KerSize -1$. We let $\Vector{v}$ be defined to be $\Vector{v} = [\Vector{v}_{\Pattern}, \ldots, \Vector{v}_{\Pattern}]$. Notice that: \[\norm{\Vector{u}}^2 + \norm{\Vector{v}}^2 = \frac{\Dim}{P} \norm{\Vector{u}_{\Pattern}}^2 +  \frac{\Dim}{P} \norm{\Vector{v}_{\Pattern}}^2 = \frac{\Dim}{P} \RkCoutDim{\SmallDim}{\Pattern}{\KerSize}{1},\] as desired. To see that $\Vector{u}^{\downarrow} \Conv \Vector{v} = \ParVec$, it suffices to show $\hat{\Vector{u}} \Entrywise \hat{\Vector{v}} = \ParVec$. Notice that $\hat{\Vector{v}}[(\Dim / P) \cdot p] = \frac{\sqrt{\Dim}}{\sqrt{P}} \hat{\Vector{v}_{\Pattern}}[p]$ for $0 \le p \le P - 1$, and $\hat{\Vector{v}}[d] = 0$ if $(\Dim / P) \nmid d$. This means that $(\hat{\Vector{u}} \Entrywise \hat{\Vector{v}})[d] = 0 = \hat{\ParVec}[d] $ if $(\Dim / P) \nmid d$ as desired. Thus it suffices to handle $(\hat{\Vector{u}} \Entrywise \hat{\Vector{v}})[(\Dim / P) \cdot p]$. Notice that $\hat{\Vector{v}}[(\Dim / P) \cdot p] = \frac{\sqrt{\Dim}}{\sqrt{P}} \hat{\Vector{v}}_{\Pattern}[p]$, and $\hat{\Vector{u}}[(\Dim / P) \cdot p] = \frac{\sqrt{P}}{\sqrt{\Dim}} \frac{\sqrt{\Dim}}{\sqrt{P}} \hat{\Vector{u}}_{\Pattern}[p] = \hat{\Vector{u}}_{\Pattern}[p]$. This means that:
\[(\hat{\Vector{u}} \Entrywise \hat{\Vector{v}})[(\Dim / P) \cdot p] = \frac{\sqrt{\Dim}}{\sqrt{P}}  \hat{\Vector{u}}_{\Pattern}[p] \hat{\Vector{v}}_{\Pattern}[p] = \frac{\sqrt{\Dim}}{\sqrt{P}} \hat{\Pattern}[p] = \hat{\ParVec}[(\Dim / P) \cdot p],\]
as desired. 

Next, we consider the case where $\KerSize  = T \cdot P$.  We construct $(\Vector{u}, \Vector{v})$  using an optimal solution $\Vector{u}_{\Pattern}$ and $\Vector{v}_{\Pattern}$ to eq. \eqref{eq:inducedregularizer} for $\RkCoutDim{\SmallDim}{\Pattern}{\SmallDim}{1}$. We let $\Vector{u} = \frac{\sqrt{\Dim}}{T^{3/4} \sqrt{P}}  [\Vector{u}_{\Pattern}, \ldots, \Vector{u}_{\Pattern}]$ be a scaled version of $T$ repeated copies of $\Vector{u}_{\Pattern}$. We let $\Vector{v} = \frac{1}{T^{1/4}} [\Vector{v}_{\Pattern}, \ldots, \Vector{v}_{\Pattern}]$ be a scaled version of $\frac{\Dim}{P}$ repeated copies of $\Vector{v}_{\Pattern}$. Then,
\[\norm{\Vector{u}}^2 + \norm{\Vector{v}}^2 = \frac{\Dim}{\sqrt{T} P} \norm{\Vector{u}_{\Pattern}}^2 +  \frac{\Dim}{\sqrt{T} P} \norm{\Vector{v}_{\Pattern}}^2 =  \frac{\Dim}{\sqrt{T} P} \RkCout{\Pattern}{\SmallDim}{1} = 2 \frac{\Dim}{\sqrt{T} P} \norm{\hat{\Pattern}}_1 ,\] as desired. To see that $\Vector{u}^{\downarrow} \Conv \Vector{v} = \ParVec$, it suffices to show $\hat{\Vector{u}} \Entrywise \hat{\Vector{v}} = \hat{\ParVec}$. Notice that $\hat{\Vector{v}}[(\Dim / P) \cdot p] = \frac{\sqrt{\Dim}}{T^{1/4} \sqrt{P}} \hat{\Vector{v}}_{\Pattern}[p]$ for $0 \le p \le P - 1$, and $\hat{\Vector{v}}[d] = 0$ if $(\Dim / P) \nmid d$. This means that $(\hat{\Vector{u}} \Entrywise \hat{\Vector{v}})[d] = 0 = \hat{\ParVec}[d] $ if $(\Dim / P) \nmid d$ as desired. Thus it suffices to handle $(\hat{\Vector{u}} \Entrywise \hat{\Vector{v}})[(\Dim / P) \cdot p]$. Notice that $\hat{\Vector{v}}[(\Dim / P) \cdot p] = \frac{\sqrt{\Dim}}{T^{1/4} \sqrt{P}} \hat{\Vector{v}}_{\Pattern}[p]$, and $\hat{\Vector{u}}[(\Dim / P) \cdot p] = \frac{T \sqrt{P}}{\sqrt{\Dim}} \frac{\sqrt{\Dim}}{\sqrt{P} T^{3/4}} \hat{\Vector{u}}_{\Pattern}[p] = T^{1/4} \hat{\Vector{u}}_{\Pattern}[p]$. This gives 
\[(\hat{\Vector{u}} \Entrywise \hat{\Vector{v}})[(\Dim / P) \cdot p] = \frac{\sqrt{\Dim}}{\sqrt{P}}  \hat{\Vector{u}}_{\Pattern}[p] \hat{\Vector{v}}_{\Pattern}[p] = \frac{\sqrt{\Dim}}{\sqrt{P}} \hat{\Pattern}[p] = \hat{\ParVec}[(\Dim / P) \cdot p],\]
as desired.

\paragraph{\textbf{Lower bound.} } For the lower bound, we use the dual formulation in eq. \eqref{eqn:dualfreqspace}. Our approach is to construct a dual vector $\Vector{\lambda} \in \mathbb{C}^{\Dim}$ so that eq. \eqref{eqn:dualfreqspace} achieves the desired objective.

Again, first we consider the case where $\KerSize \le \SmallDim$. Let $\Vector{\lambda}_{\Pattern} \in \mathbb{C}^{P}$ be the dual optimal solution for $\RSDPk{\Pattern}{\KerSize}$. Now, let $\Vector{\lambda}[(\Dim / P) \cdot p] = \frac{\sqrt{\Dim}}{\sqrt{P}} \Vector{\lambda}_{\Pattern}[p]$ for $0 \le p \le P - 1$, and $\Vector{\lambda}[d] = 0$ if $(\Dim / P) \nmid d$. Notice that the objective becomes:
\begin{align*}
  2 \sum_{d=0}^{\Dim} |\Vector{\lambda}[d]| \cdot |\hat{\ParVec}[d]| &= 2 \sum_{p=0}^{P-1} |\Vector{\lambda}[(\Dim / P) \cdot p]| \cdot |\hat{\ParVec}[(\Dim / P) \cdot p]| \\
  &=   2 \frac{\Dim}{P} \sum_{p=0}^{P-1} |\Vector{\lambda}_{\Pattern} [p]| \cdot |\hat{\Pattern}[p]|  \\
  &=  2 \frac{\Dim}{P} \sum_{p=0}^{P-1} |\Vector{\lambda}_{\Pattern} [p]| \cdot |\hat{\Pattern}[p]|  \\
  &= \frac{\Dim}{P} \RSDPk{\Pattern}{\KerSize} \\
  &= \frac{\Dim}{P} \RkCout{\Pattern}{\KerSize}{\COut},
\end{align*}
where the last equality follows from tightness of the SDP (Theorem \ref{thm:tightness}). It thus suffices to show that $\Vector{\lambda}$ is dual feasible. For $\Vector{x} \in \mathbb{C}^{\KerSize}$ such that $\norm{\Vector{x}} \le 1$, consider:
\begin{align*}
    \sum_{d=0}^{\Dim- 1} |\hat{\Vector{x}}[d]|^2 \cdot |\Vector{\lambda}[d]|^2 &= \sum_{p=0}^{P- 1} |\hat{\Vector{x}}[(\Dim / P) \cdot p]|^2 \cdot |\Vector{\lambda}[(\Dim / P) \cdot p]|^2 \\
    &= \frac{\Dim}{P} \sum_{p=0}^{P- 1} |\hat{\Vector{x}}[(\Dim / P) \cdot p]|^2 \cdot |\Vector{\lambda}_{\Pattern}[p]|^2.
\end{align*}
Now,  let $\hat{x}^{(P)}\in \mathbb{C}^{P}$ be the Fourier representation of $x$ when the base dimension is $P$. Observe that $\hat{\Vector{x}}[(\Dim / P) \cdot p]$ is equal to $\sqrt{\frac{P}{\Dim}} \hat{\Vector{x}}^{(P)}[p]$. Thus the above expression is equal to: 
\[\frac{\Dim}{P} \frac{P}{\Dim}  \sum_{p=0}^{P- 1} | \hat{\Vector{x}}^{(P)}[p]|^2 \cdot |\Vector{\lambda}_{\Pattern}[p]|^2 = \sum_{p=0}^{P- 1} | \hat{\Vector{x}}^{(P)}[p]|^2 \cdot |\Vector{\lambda}_{\Pattern}[p]|^2 .\]
Since $\Vector{\lambda}_{\Pattern}[p]$ is dual feasible for the $P$ dimensional problem, we see that this is at most $1$, as desired. 

Next, we consider the case where $\KerSize  = T \cdot P$. We consider $\Vector{\lambda}[(\Dim / P) \cdot p] = \frac{\sqrt{\Dim}}{\sqrt{K}}$ for $0 \le p \le P - 1$, and $\Vector{\lambda}[d] = 0$ if $(\Dim / P) \nmid d$. Notice that the objective in eq. \eqref{eqn:dualfreqspace} is equal to:
\[2 \sum_{d=0}^{\Dim} |\Vector{\lambda}[d]| \cdot |\hat{\ParVec}[d]| = 2 \frac{\sqrt{\Dim}}{\sqrt{\KerSize}} \sum_{p=0}^{P-1} |\hat{\ParVec}[(\Dim / P) \cdot p]| =  2 \frac{\sqrt{\Dim}}{\sqrt{\KerSize}} \norm{\hat{\ParVec}}_1 ,\] as desired. It thus suffices to show that $\Vector{\lambda}$ is dual feasible. For $\Vector{x} \in \mathbb{C}^{\KerSize}$ such that $\norm{\Vector{x}} \le 1$, we consider
\[\sum_{d=0}^{\Dim- 1} |\hat{\Vector{x}}[d]|^2 \cdot |\Vector{\lambda}[d]|^2 = \frac{\Dim}{\KerSize} \sum_{p=0}^{P- 1} |\hat{\Vector{x}}[(\Dim / P) \cdot p]|^2. \] Let $\Vector{x}_0, \ldots, \Vector{x}_{T-1} \in \mathbb{C}^{P}$ be defined so that $\Vector{x}_t[p] = \Vector{x}[t \cdot \KerSize + p]$ for $0 \le p \le P - 1$ and $0 \le t \le T - 1$. Now, let $\hat{\Vector{x}}_t^{(P)}$ denote the Fourier representation of $\Vector{x}$ when the base dimension is $P$, and observe that $\hat{\Vector{x}}[(\Dim / P) \cdot p] =  \sqrt{\frac{P}{\Dim}} \sum_{t=0}^{T-1}  \hat{\Vector{x}}^{(P)}_t[p]$. 
Thus we can rewrite the above expression as:
\begin{align*}
    \frac{\Dim}{\KerSize} \frac{P}{\Dim} \sum_{p=0}^{P- 1}  \left|\sum_{t=0}^{T-1}  \hat{\Vector{x}}^{(P)}_t[p]\right|^2 &= \frac{1}{T} \sum_{p=0}^{P- 1} \sum_{t=0}^{T-1}  T \left|\hat{\Vector{x}}^{(P)}_t[p]\right|^2 \\
    &=  \sum_{t=0}^{T-1}  \norm{\hat{\Vector{x}}^{(P)}_t}^2 \\
    &= \sum_{t=0}^{T-1}  \norm{\Vector{x}_t}^2 \\
    &= \norm{\Vector{x}}^2 \le 1,
\end{align*}
as desired. This completes the proof.
\end{proof}


\section{Appendix for Section~\ref{sec:mult-input-channel}: Networks with multi-channel inputs}\label{appendix:proofsmulti-input}
For all the results in this appendix, we recall that the weights of the first and second layer are denoted as $\UU=\{\U_\cin\}_{\cin\in[\CIn]}$ with $\U_\cin\in\bR^{\KerSize\times\COut} \forall_{\cin\in[\CIn]}$ and $\V\in\bR^{\Dim\times\COut}$, respectively. 

First, we provide a more formal descriptions of the matrices in the SDP.  We provide a block-wise description of only the upper diagonal blocks, with lower diagonal blocks filled to satisfy the Hermitian matrix property. Additionally, for matrices $\crl{\Matrix{A}_{d,\cin}^\text{real},\Matrix{A}_{d,\cin}^\text{img}}_{d\in[D],\cin\in[\CIn]}$ any unspecified block is by default treated as zero matrix $\Matrix{0}$ of appropriate dimension: 
 \begin{itemize}
 \item For $\cin_1,\cin_2\in[\CIn]$ with $\cin_2\ge\cin_1$, the $\KerSize\times\KerSize$ block with indices $(\cin_1:(\cin_1+1)\KerSize)$ along rows and $(\cin_2:(\cin_2+1)\KerSize)$ along columns is given as \[\Index{\SDPVar}{\cin_1:(\cin_1+1)\KerSize, \cin_2:(\cin_2+1)\KerSize}=\U_{\cin_1}\U_{\cin_2}^\top;\] 
 \item For $\cin\in[\CIn]$, the $\KerSize\times\Dim$ blocks with indices $(\cin:(\cin+1)\KerSize)$ along rows and $(\CIn\KerSize:(\Dim+\CIn\KerSize))$ along column, are given as follows: 
 \[\Index{\SDPVar}{\cin:(\cin+1)\KerSize, \CIn\KerSize:(\Dim+\CIn\KerSize)}=\U_{\cin}\V^\top.\]
 
 Further,  
\begin{equation*}
\begin{split}
    \forall_{d\in[\Dim]},\; &\Index{\Matrix{A}_{d,\cin}^\text{real}}{\cin:(\cin+1)\KerSize, \CIn\KerSize:(\Dim+\CIn\KerSize)}=\Q_{d},\text{ and  }\\&\Index{\Matrix{A}_{d,\cin}^\text{img}}{\cin:(\cin+1)\KerSize, \CIn\KerSize:(\Dim+\CIn\KerSize)}=i\cdot \Q_d.
\end{split}
\end{equation*}

Note that for $\cin^\prime\neq \cin$, the corresponding blocks in  ${\Matrix{A}_{d,\cin^\prime}^\text{real}},{\Matrix{A}_{d,\cin^\prime}^\text{img}}$ remain the default zero.
 \item Finally, the lower-right $\Dim\times \Dim$ block is given as  \[\Index{\SDPVar}{\CIn\KerSize:(\Dim+\CIn\KerSize),\CIn\KerSize:(\Dim+\CIn\KerSize)}=\V\V^\top.\] 
 \end{itemize} 

\subsection{Realizability of linear functions}\label{appendix:proofsmulti-input-realize}
We show that multiple output channels can be needed to merely realize all linear maps for multi-input-channel networks. 
\realizability*

We prove this lemma by showing that the sub-network corresponding to each output channel can realize a matrix in $\bR^{\Dim\times\CIn}$ of rank at most $\KerSize$. \newline 

\begin{proof}
We first reiterate the expressions for the linear predictor $\ParMatrixFn(\UU,\V)$ in terms of $\UU,\V$:
\begin{equation}\label{eq:conv-multi-input-proof}
 \forall_{\cin\in[\CIn]},\; \ParMatrixFn{(\UU,\V)}[:,\cin]= \sum_{\cout=0}^{\COut-1} \prn*{\U_\cin[:,\cout]\Conv\V[:,\cout]^{\downarrow }}^\downarrow.
\end{equation}
We will now express the above formulation as matrix multiplication using the following new notation: $\forall_{\cout\in[\COut]}$ let $\underline{\U}_\cout\in\bR^{\KerSize\times\CIn}$ denote the representation of first layer weights corresponding to each output channel such that \[\forall_{\cin\in[\CIn]}\forall_{\cout\in[\COut]},\;\underline{\U}_\cout[:,\cin]=\U_\cin[:,\cout].\]

For $\cout\in[\COut]$, consider the following matrix which consists of first $\KerSize$ columns of the circulant matrix formed by $\V[:,\cout]$:
\[ \forall_{\cout\in[\COut]},\;\widetilde{\V}_\cout=\frac{1}{\sqrt{\Dim}}\brk*{\begin{array}{cccc}\V[0,\cout]&\V[\Dim-1,\cout]&\cdots&\V[\Dim-\KerSize+1,\cout]\\\V[1,\cout]&\V[0,\cout]&\cdots&\V[\Dim-\KerSize+2,\cout]\\\vdots&\vdots&&\vdots\\\V[\Dim-1,\cout]&\V[\Dim-2,\cout]&\cdots&\V[\Dim-\KerSize,\cout]\end{array}}\in\bR^{\Dim\times\KerSize}\]
Based on this notation, we can check by following the definitions that for all $\cout$,
\[\prn*{\U_\cin[:,\cout]\Conv\V[:,\cout]^{\downarrow }}^\downarrow=\widetilde{\V}_\cout\U_\cin[:,\cout]=\widetilde{\V}_\cout\,{\underline{\U}}_\cout[:,\cin].\] We can thus write  $\ParMatrixFn(\UU,\V)$ as follows:
\begin{equation}\label{eq:conv-matmul}
\ParMatrixFn{(\UU,\V)}= \sum_{\cout=0}^{\COut-1} \widetilde{V}_\cout\,\underline{\U}_\cout.
\end{equation}

We now observe that each term in the summation $\widetilde{V}_\cout\,\underline{\U}_\cout$ is of rank utmost $\KerSize$ as $\widetilde{V}_\cout\in\bR^{\Dim\times\KerSize}$ and $\underline{\U}_\cout\in\bR^{\Dim\times\KerSize}$. Thus, for any $\UU,\V$, $\rank(\ParMatrixFn{(\UU,\V)})\le\KerSize\cdot\COut$. From this we conclude that in order to realize all linear maps in the multi-channel input space of  $\bR^{\Dim\times\CIn}$, we necessarily need $\KerSize\cdot\COut\ge\min\crl{\CIn,\Dim}$. \newline

Additionally, in eq. \eqref{eq:conv-matmul} we see that since $\underline{\U}_\cout$ are unconstrained, each term in the sum can realize any rank $1$ matrix. This implies that $\COut\ge\min\crl{\CIn,\Dim}$ is a sufficient condition for $\ParMatrixFn{(\UU,\V)}$ to realize any $\W\in\bR^{\Dim\times\CIn}$. However, from Theorem~\ref{thm:multichannelkersizeD}  we know that this condition is not necessary. It is an open question to derive the tightest necessary and sufficient conditions.
\end{proof}
 
 We note that a similar proof as above can be shown using the Fourier representation in eq. \eqref{eq:w-uv-multiinput}. 
 In the special cases of $\KerSize=1$ and $\KerSize=\Dim$, in Theorems~\ref{thm:multichannelkersize1}-\ref{thm:multichannelkersizeD} we show that the SDP is tight once $\COut$ is large enough to realize all linear functions, which in these cases is $\COut\ge\CIn/\KerSize$. In these end cases, we further derive interesting closed form expressions of $\RkCinCout{\ParMatrix}{\KerSize}{\COut}{\CIn}$. 
 
 \subsection{Proof of Lemma \ref{lemma:multichanneltightness}}
 \multichanneltightness*
 \begin{proof}
 We show a stronger statement: \textit{any} optimal solution to the SDP has rank at most $\CIn \KerSize$. This implies the desired result, because $\RkCinCout{\ParMatrix}{\KerSize}{\COut}{\CIn}$ is equivalent to the SDP with a rank constraint of $\COut$.
 
 For the remainder of the proof, we let $\SDPVar$ be an optimal solution to the SDP and we prove that $\text{rank}(\SDPVar) \le \CIn \KerSize$. Let $\text{rank}(\SDPVar) = L$. We can write:
 \[ \SDPVar=\left[\!\!\begin{array}{c} \Matrix{U}_0 \\ \Matrix{U}_1 \\ \ldots \\ \Matrix{U}_{\CIn-1} \\ \Matrix{V} \end{array}\!\!\right] \!\!\!\!\begin{array}{c}\begin{array}{ccccc}[\Matrix{U}_0^{\top} & \Matrix{U}_1^{\top} & \ldots & \Matrix{U}_{\CIn-1}^{\top} & \Matrix{V}^{\top}\! ]\end{array}\end{array}\]
 where the matrices $\U_r \in \mathbb{R}^{\KerSize \times L}$  for $0 \le \cin \le \CIn - 1$ correspond to the weights in the convolution layer, the matrix $\V \in \mathbb{R}^{\Dim \times L}$ corresponds to the weights in the linear layer. 
 
 It suffices to show that there exists a spanning set of the column space of $\left[\Matrix{U}_0^\top ,\Matrix{U}_1^\top, \ldots,\Matrix{U}_{\CIn}^\top,\Matrix{V}^\top\right]^\top$
with  at most $\CIn \KerSize$ elements.
 The key ingredient of the proof is the KKT conditions. Using similar logic to Appendix \ref{appendix:tightness-main}, we can write the KKT conditions in the following form: 
 \begin{align*}
    \text{ for all } 0 \le \cin \le \CIn - 1: \hat{\ParMatrix}[:, r] &= \hat{\U}_\cin \odot \hat{\V}  \label{eq:kkt1multi} \tag{KKT $1$} \\
    \!\!\!\!\begin{array}{c}\left[\!\!\begin{array}{cccc;{2pt/2pt}c}  &  & & &   \conj{\F}_K^T \Matrix{\Lambda}_0 \conj{\F} \\  &  & \Matrix{0}_{(\CIn \cdot \KerSize) \times (\CIn \cdot \KerSize)}& &  \conj{\F}_K^T \Matrix{\Lambda}_1 \conj{\F}  \\  &  &  & & \vdots \\
     &  & & &  \conj{\F}_K^T \Matrix{\Lambda}_{\CIn - 1} \conj{\F}  \\ 
    \hdashline[2pt/2pt] \rule{0pt}{10pt}
   \F \conj{\Matrix{\Lambda}}_0 \F_K &    \F \conj{\Matrix{\Lambda}}_1 \F_K & \ldots &  \F \conj{\Matrix{\Lambda}}_{\CIn - 1} \F_K & \Matrix{0}_{\Dim \times \Dim}  \end{array}\!\!\right]\\\;\end{array} &\preceq \Matrix{I}_{\Dim + \KerSize \CIn}  \tag{KKT $2$} \\
    \conj{\hat{\V}} &= \sum_{\cin = 0}^{\CIn - 1} \conj{\Matrix{\Lambda}}_\cin \hat{\U}_\cin  \label{eq:kkt3multi} \tag{KKT $3$} \\
   \text{ for all } 0 \le \cin \le \CIn - 1:  \hat{\U}_\cin &= \F_K \conj{\F}_K^T \Matrix{\Lambda}_\cin \conj{\hat{\V}} \label{eq:kkt4multi} \tag{KKT $4$}.
 \end{align*}
 where the matrices $\Matrix{\Lambda}_r \in \mathbb{C}^{\Dim \times \Dim}$  for $0 \le \cin \le \CIn - 1$ are diagonal and correspond to the relevant dual variables.  Let's use \eqref{eq:kkt4multi} to construct a spanning set of the column space of this matrix. Let $S_{\cin} =\left\{\Vector{u}_{\cin} \in \mathbb{R}^{\KerSize}\right\}$ be a basis for the column space of $\U_{\cin}$ for each $0 \le \cin \le \CIn$; since $\U_{\cin}$ is $\KerSize \times L$ dimensional, we see that $S_{\cin}$ has at most $\KerSize$ elements. We see by \eqref{eq:kkt4multi} that the set of vectors given by the concatenation of $[\Vector{u}_0, \Vector{u}_1, \ldots, \Vector{u}_{\CIn - 1}, \conj{F} \left(\sum_{\cin = 0}^{\CIn - 1} \conj{\Matrix{\Lambda}}_\cin \hat{\Vector{u}}_\cin]\right)]$ for $\Vector{u}_{\cin} \in S_{\cin}$ for each $0 \le \cin \le \CIn$ spans the column space of the desired matrix. By construction, this spanning set has at most $\CIn \KerSize$ elements, as desired. 
 
 \end{proof}

 \subsection{Proofs of Theorems~\ref{thm:multichannelkersize1}-\ref{thm:multichannelkersizeD}: induced regularizer for \texorpdfstring{$\KerSize=1$}{K=1} and \texorpdfstring{$\KerSize=\Dim$}{K=D}}\label{appendix:proofsmulti-input-specialcases}
 \multiinputone*
 
 \begin{proof}
 For $\KerSize=1$, we have that $\forall_{\cin\in[\CIn]}$, $\U_r\in\bR^{1\times\COut}$. For this proof, we stack the vectors $\U_r$ to obtain $\widetilde{\U}\in\bR^{\CIn\times\COut}$ such that $\forall_{\cin\in[\CIn]},\, \widetilde{\U}[\cin,:]=\U_\cin$.\newline

\noindent We work with the definition of the linear predictor realized by network as: $\forall_{\cin\in[\CIn]},\; \ParMatrixFn{(\UU,\V)}[:,\cin]= \sum_{\cout=0}^{\COut-1} \prn*{\U_\cin[:,\cout]\Conv\V[:,\cout]^{\downarrow }}^\downarrow$  (from eq.~\eqref{eq:w-uv-multiinput}). 
 For kernel size of $1$, we notice that from the definition of convolution in Definition~\ref{def:conv}, we have the following:
 \begin{equation}
     \prn*{\U_\cin[:,\cout]\Conv\V[:,\cout]^{\downarrow}}^{\downarrow}=\frac{\tilde{\U}[\cin,\cout]}{\sqrt{\Dim}}\V[:,\cout]\propto \V[:,\cout].
 \end{equation}
 Plugging this back into the expression of $\ParMatrixFn{(\UU,\V)}$, we have the following:
 \begin{equation}
     \ParMatrixFn{(\UU,\V)}[:,\cin]=\frac{1}{\sqrt{\Dim}}\V\widetilde{\U}[\cin,:]\implies \ParMatrixFn{(\UU,\V)}=\frac{1}{\sqrt{\Dim}}\V\widetilde{\U}^\top.
 \end{equation}
In the above formulation $\V\in\bR^{\Dim\times \COut},\widetilde{\U}\in\bR^{\CIn\times\COut}$ are of rank $\COut$, but otherwise completely unconstrained. Thus, they can realize any rank $\KerSize$ matrix. So as long as $\COut\ge \min\crl{\CIn,\Dim}$, the network can realize any linear predictor $\W\in\bR^{\Dim\times\CIn}$.

The rest of the proof follows from connecting the above expression into the variational characterization of the nuclear norm. 
The induced regularizer $\RkCinCout{\ParMatrix}{1}{\COut}{\CIn}$ from eq. \eqref{eq:R-multiinput} for  $\COut\ge \min\crl{\CIn,\Dim}$ can now be expressed as follows:
\begin{equation}\label{eq:R-1CR-varchar}
\begin{array}{r@{\ }c@{\ }l}
\RkCinCout{\W}{1}{\COut}{\CIn} = &\min\limits_{\widetilde{\U}\in\bR^{\CIn\times\COut}, \V\in\bR^{\Dim\times\COut}} &\|\widetilde{\U}\|^2 + \|\V\|^2\\
 &\st & \sqrt{\Dim}\ParMatrix=\V\widetilde{\U}^\top.
\end{array}
\end{equation}
For $\COut\ge \min\crl{\CIn,\Dim}$ eq.~\eqref{eq:R-1CR-varchar} is exactly the variational definition of nuclear norm (see \citet[Lemma~1][]{rennie2005fast}) and thus $\RkCinCout{\W}{1}{\COut}{\CIn}=2\sqrt{\Dim}\|\W\|_*$ for even unbounded $\COut$. The fact that $\COut=\min\crl{\CIn,\Dim}$ is sufficient can be seen by obtaining the optimum nuclear norm as upper bound from using ${\V}=\Matrix{L}\sqrt{\Matrix{\Sigma}}$ and $\widetilde{\U}=\Matrix{R}\sqrt{\Matrix{\Sigma}}$, where $\sqrt{\Dim}\W=\Matrix{L}\Matrix{\Sigma}\Matrix{R}^\top$ is the singular value decomposition of $\sqrt{\Dim}\W$. Finally, we note that based on our normalization of Fourier transform, we have $\|\W\|_*=\|\hat{\W}\|_*$. This completes our proof. \newline

\noindent \textit{Note: } For $\KerSize=1$ we provided the proof of $\RkCinCout{\W}{1}{\COut}{\CIn}$ in the signal space of $\W$, but the Theorem can also be proved in the Fourier domain (similar to the proof of Theorem~\ref{thm:multichannelkersizeD} given below) by first showing that the SDP relaxation evaluates to the nuclear norm and combining this with the matching upper bound for $\RkCinCout{\W}{1}{\COut}{\CIn}$ shown above. 
  \end{proof}

\multiinputD*

\begin{proof} We begin by expressing induced regularizer for full dimensional kernels $\RkCinCout{\ParMatrix}{\Dim}{\COut}{\CIn}$ from eq. \eqref{eq:R-multiinput} in terms of the Fourier representation of the linear predictor realized by the network $\ParMatrixFn(\UU,\V)$:
\begin{equation}\label{eq:R-DCR-fourier}
\begin{array}{r@{\ }c@{\ }l}
\RkCinCout{\W}{\Dim}{\COut}{\CIn} = &\inf\limits_{\UU,\V}&\|\hat{\U}_\cin\|^2 + \|\hat{\V}\|^2\\
 &\st & \forall_{\cin\in[\CIn]}\ParMatrix[:,\cin]=\diag(\hat{\U}_\cin\hat{\V}^\top),
\end{array}
\end{equation} 
where $\UU,\V$ are of dimensions $\UU=\{\U_\cin\in\bR^{\Dim\times\COut}\}_{\cin\in[\CIn]}, \V\in\bR^{\Dim\times\COut}$.\newline 

\noindent Our proof for networks with multi-channel inputs with $\KerSize=\Dim$  follows the following structure:
\begin{compactenum}[\bf Step 1.]
\item We first show an upper bound on the induced regularizer for single output channel $\COut=1$ with full dimensional kernel  $\KerSize=\Dim$ as $\RkCinCout{\W}{\Dim}{1}{\CIn}\le 2\|\hat{\W}\|_{2,1}$ by providing a construction of $\UU,\V$. It immediately follows from the monotonicity of $\RkCinCoutOp{\Dim}{\COut}{\CIn}$ that for all $\COut$,  $\RkCinCout{\W}{\Dim}{\COut}{\CIn}\le \RkCinCout{\W}{\Dim}{1}{\CIn}\le 2\|\hat{\W}\|_{2,1}$. This step also consequently shows that when $\KerSize=\Dim$, every linear predictor over the multi-channel input space of $\bR^{\Dim\times\CIn}$ is realizable by a network with even a single output channel. 
\item The bulk of our proof lies in matching the upper bound with a lower bound on the dual problem of the SDP in eq. \eqref{opt:SDPmulti-channel} as $\RSDPkCin{\W}{\Dim}{\CIn}\ge2\|\hat{\W}\|_{2,1} $. This gives us that for all $\COut\ge 1$, $\RkCinCout{\W}{\Dim}{\COut}{\CIn}\ge\RSDPkCin{\W}{\Dim}{\CIn}\ge2\|\hat{\W}\|_{2,1}$. 
\end{compactenum}

\paragraph{\textbf{Step }1. Upper bound on the induced regularizer $\boldsymbol{\RkCinCoutOp{\Dim}{\COut}{\CIn}}$:} We first show that $\RkCinCout{\W}{\Dim}{1}{\CIn}\le 2\|\hat{\W}\|_{2,1}=2\sum_{d\in[\Dim]}\|\hat{\W}[d,:]\|$. 
Since $\RkCinCout{\W}{\Dim}{\CIn}{\COut}$ is decreasing in $\COut$, it suffices to show this for $\COut = 1$. For $\COut=1$ and $\KerSize=\Dim$, we have $\V\in\bR^\Dim$ and $\forall_{\cin\in[\CIn]}$ $\U_\cin\in\bR^\Dim$. For full dimensional kernels, the Fourier domain representations $\hat{\U}_0,\hat{\U}_1,\ldots,\hat{\U}_\CIn,\hat{\V}\in\bC^\Dim$ are all unconstrained beyond the symmetry properties of Fourier transform of real  matrices. Thus consider the following $\UU,\V$:
\begin{equation}
        \forall_{d\in[\Dim]}\forall_{\cin\in[\CIn]},\hat{\U}_\cin[d]=\frac{\hat{\W}[d,\cin]}{\sqrt{\|\hat{\W}[d,:]\|}},\quad\text{and}\quad\hat{\V}[d]={\sqrt{\|\hat{\W}[d,:]\|}}. 
\end{equation}
It is easy to see that the above $\UU,\V$ satisfy the constraints of $\RkCinCout{\W}{\Dim}{1}{\CIn}$ in eq.~\eqref{eq:R-DCR-fourier} that is  $\forall_{\cin\in[\CIn]}\ParMatrix[:,\cin]=\diag(\hat{\U}_\cin\hat{\V}^\top)$. Further $\hat{\U}_0,\hat{\U}_1,\ldots,\hat{\U}_\CIn,\hat{\V}\in\bC^\Dim$ satisfy the required symmetry properties since $\W$ is real.\newline 

\noindent Now computing the objective, we immediately have that $\|\hat{\V}\|^2=\sum_{d\in[\Dim]}\|\hat{\W}[d,:]\|$, and further, 
\[\sum_{\cin\in[\CIn]}\|\hat{\U}_\cin\|^2=\sum_{d\in[\Dim]}\brk*{\sum_{\cin\in[\CIn]}\frac{|\hat{\W}[d,\cin]|^2}{{\|\hat{\W}[d,:]\|}}}=\sum_{d\in[\Dim]}\|\hat{\W}[d,:]\|.
\]
This construction thus gives us the desired upper bound 
\begin{equation}\label{eq:multi-input-kd-ub}
\RkCinCout{\W}{\Dim}{\COut}{\CIn}\le \RkCinCout{\W}{\Dim}{1}{\CIn}\le 2\|\hat{\W}\|_{2,1}.
\end{equation}
\vspace{5pt}
\paragraph{Step 2: Lower bound on the induced regularizer $\boldsymbol{\RkCinCoutOp{\Dim}{\COut}{\CIn}}$}
We show the lower bound by lower bounding the dual problem of the SDP in eq. \eqref{opt:SDPmulti-channel}. 
For $\cin \in[\CIn]$, let $\DualRealCin{\cin}\in\bR^\Dim$ and $\DualImgCin{\cin}\in\bR^\Dim$ denote the dual variables corresponding to the constraints in the SDP in eq. \eqref{opt:SDPmulti-channel} for the real and imaginary parts, respectively, of $\hat{\W}[:,\cin]$. Similar to single input channel proofs, we define $\DualVec_\cin=\DualRealCin{\cin}+i\cdot\DualImgCin{\cin}$ and $\Matrix{\Lambda}_\cin=\diag(\DualVec_\cin)$. Additionally, we introduce the notation for the matrix obtained by taking $\{\DualVec_\cin\}_\cin$ as columns: $\Xi=[\DualVec_0,\DualVec_1,\ldots,\DualVec_\CIn]\in\bC^{\Dim\times\CIn}$ such that  $\forall_\cin,\,\Xi[:,\cin]=\DualVec_\cin$.\newline

\noindent Based on weak duality for the SDP in eq. \eqref{opt:SDPmulti-channel}, we have the following:
\begin{equation}\label{eq:weakduality}
\begin{array}{r@{\ }c@{\ }l}
\RSDPkCin{\W}{\Dim}{\CIn}\ge &\max\limits_{\Xi=[\DualVec_0,\DualVec_1,\ldots,\DualVec_\CIn]} &2\cdot\Re(\innerprod{\hat{\W},\Xi})\\
 &\st & \sum\limits_{d,\cin}(\DualRealCin{\cin}[d] \cdot \Matrix{A}_{d,\cin}^\text{real}+\DualImgCin{\cin}[d] \cdot \Matrix{A}_{d,\cin}^\text{img})\preccurlyeq \Matrix{I}.
\end{array}
\end{equation}
Our rest of the proof obtains the lower bound by constructing an appropriate $\Xi$ satisfying the constraints: 
From the definitions of $\crl{\Matrix{A}_{d,\cin}^\text{real},\Matrix{A}_{d,\cin}^\text{img}}_{d\in[D],\cin\in\CIn}$ in Appendix~\ref{appendix:proofsmulti-input-SDP} and using $\Q_d=\conj{\F}\e_d\e_d^\top \conj{\F}$ (note that for $\KerSize=\Dim$ and $\F_\KerSize=\F=\F^\top$), we have the following:

\begin{equation}
    \sum_{d,\cin}(\DualRealCin{\cin}[d] \cdot \Matrix{A}_{d,\cin}^\text{real}+\DualImgCin{\cin}[d] \cdot \Matrix{A}_{d,\cin}^\text{img})=
    \begin{array}{c}\left[\!\!\begin{array}{cccc;{2pt/2pt}c}  &  & & &  \conj{\F}\Matrix{\Lambda}_0\conj{\F} \\  &  & \Matrix{0}_{(\CIn \cdot \Dim) \times (\CIn \cdot \Dim)}& & \conj{\F}\Matrix{\Lambda}_1\conj{\F} \\  &  &  & & \vdots  \\  &  & & & \conj{\F}\Matrix{\Lambda}_\CIn\conj{\F}  \\
    \hdashline[2pt/2pt]\rule{0pt}{10pt}
    \F\conj{\Matrix{\Lambda}}_0\F_K & \F\conj{\Matrix{\Lambda}}_1\F_K & \ldots & \F\conj{\Matrix{\Lambda}}_\CIn\F_K & \Matrix{0}_{\Dim \times \Dim}  \end{array}\!\!\right]\end{array}
\end{equation}

\noindent We state and prove the following claim:
\begin{clm*}$\Xi$ satisfies the constraints of the dual problem in the RHS of eq. \eqref{eq:weakduality} if $\max\limits_{d\in[\Dim]}\|\Xi[d,:]\|\le1$
\end{clm*}
\begin{proof}[Proof of claim.]
The relevant constraint in eq. \eqref{eq:weakduality} is $\sum\limits_{d,\cin}(\DualRealCin{\cin}[d] \cdot \Matrix{A}_{d,\cin}^\text{real}+\DualImgCin{\cin}[d] \cdot \Matrix{A}_{d,\cin}^\text{img})\preccurlyeq \Matrix{I}$.
It thus suffices to show that for all $\Vector{y} \in \mathbb{R}^\Dim$ such that $\|\b{y}\|=1$, it holds that $\sum_{\cin=0}^{\CIn}\|\conj{\F}\Matrix{\Lambda}_\cin\conj{\F}\b{y}\|^2 \le 1$. We have the following set of inequalities that prove the claim $\forall_{\b{y}:\|\b{y}\|=1}$:
\begin{align*}
    \sum_{\cin=0}^{\CIn}\|\conj{\F}\Matrix{\Lambda}_\cin\conj{\F}\b{y}\|^2 &\overset{(a)}=\sum_{\cin=0}^{\CIn}\|\Matrix{\Lambda}_\cin(\conj{\F}\b{y})\|^2\\
    &=\sum_{d=0}^\Dim\sum_{\cin=0}^{\CIn}|\DualVec_\cin[d]|^2\cdot|(\conj{\F}\b{y})[d]|^2\overset{(b)}=\sum_{d=0}^\Dim \|\Xi[d,:]\|^2\cdot|(\conj{\F}\b{y})[d]|^2\\
    &\le\max_{d\in[\Dim]} \|\Xi[d,:]\|^2\|\conj{\F}\b{y}\|^2\\
    &\overset{(c)}=\max_{d\in[\Dim]} \|\Xi[d,:]\|^2,
\end{align*}
where $(a)$ follows from $\conj{\F}$ being unitary, $(b)$ from definition of $\Xi$, and $(c)$ from $\|\F\b{y}\|=\|\b{y}\|=1$. 
\end{proof}
Consider $\Xi$ defined as follows:
\[\forall_{d\in[\Dim]},\quad\Xi[d,:]=\frac{\hat{\W}[d,:]}{\|\hat{\W}[d,:]\|}\]
It is easy to check that $\max\limits_{d\in[\Dim]}\|\Xi[d,:]\|=1$ and thus based on the claim we proved above $\Xi$ satisfies the constraints of the dual optimization problem in the RHS of eq. \eqref{eq:weakduality}. Additionally, the objective evaluates to the desired bound of  $\Re(\innerprod{\hat{\W},\Xi})=\sum_d\|\hat{\W}[d,:]\|=\|\hat{\W}\|_{2,1}$. We thus have the following lower bound:
\begin{equation}\label{eq:multi-input-kd-lb}
\RkCinCout{\W}{\Dim}{\COut}{\CIn}\ge  \RSDPkCin{\W}{\Dim}{\CIn}\ge 2\Re(\innerprod{\hat{\W},\Xi})=2\|\hat{\W}\|_{2,1}.
\end{equation}
\paragraph{Conclusion of the proof. } The proof of the Theorem follows from combining the matching upper and lower bounds in eq.~\eqref{eq:multi-input-kd-ub} and eq.~\eqref{eq:multi-input-kd-lb}, respectively, to obtain $\RkCinCout{\W}{\Dim}{\COut}{\CIn}=2\|\hat{\W}\|_{2,1}$.
\end{proof}

\section{Experiments for gradient descent}\label{appendix:experiments}

We run our experiments on  two layer linear  convolutional networks on a subset of MNIST dataset \citep{MNIST} as well a subset of the CIFAR-10 dataset \citep{CIFAR10}. The input images in MNIST are of size $28\times 28$ and have a single input channel.  The input images in MNIST are of size $32 \times 32$ and have a $3$ input channels. We apply $2$D convolutions with kernel sizes $\KerSize=(\KerSize_1,\KerSize_2)$ and circular padding for image inputs. We consider binary classification task for both datasets. For MNIST, we predict digits $0$ and $1$ in  MNIST using a balanced sub-sampling of $128$ samples as training data, which ensures linear separability. For CIFAR-10, we predict classes ``automobile'' and ``dog'' using a balanced sub-sampling of $256$ samples as training data, which ensures linear separability. The initialization scale was taken to be $0.001$. 

We train our network using gradient descent on exponential loss and run gradient  descent until the training loss is $10^{-6}$. The initialization scale is taken to be $0.001$ in order to reduce the variance arising from the randomness of initialization. In order to compare the predictors across different architectures, we normalize   the weights learned by gradient descent $\U,\V$ such that the  linear predictor $\ParVecFn(\Matrix{U}, \Matrix{V})$ realized by the trained networks has unit margin on the training dataset (\ie $y\innerprod{\ParVecFn(\U,\V),\x}\ge1$ for all training samples $(\x,y)$). Note that for homogeneous models, such positive scaling of weights \textit{does not} change the classification boundary of the learned model.\footnote{The code is available at \url{https://github.com/mjagadeesan/inductive-bias-multi-channel-CNN}.}

\subsection{Impact of the number of channels on MNIST}\label{appendix:numoutputMNIST}

For networks with a single-input channel, Hypothesis \ref{hyp:outputchannels} would imply that $\RkCinCout{\ParVecFn(\Matrix{U}, \Matrix{V})}{\KerSize}{\COut}{\CIn}$ is invariant in the number of output channels $\COut$ regardless of $\COut$. To demonstrate  this, we repeat the experimental setup on $28 \times 28$ MNIST images on networks with multiple output channels $\COut \in \left\{1, 2, 4, 8\right\}$ and across different kernel sizes. As described earlier, we scale the weights learned by gradient descent $\U,\V$ such that the linear predictors $\ParVecFn(\Matrix{U}, \Matrix{V}) = \ParVec_{\mathrm{GD}}$ have unit margin on training data. Since it is difficult to directly compute $\RkCinCout{\ParVec_{\mathrm{GD}}}{\KerSize}{\COut}{\CIn}$, we turn to an approximation. In particular, we compute $\RkCoutHat{\ParVecFn(\U,\V)}{\KerSize}{\COut} := \|\U\|^2+\|\V\|^2$. Strictly speaking, this is only an \textit{upper bound} on the induced regularizer $\RkCinCout{\ParVecFn(\U,\V)}{\KerSize}{\COut}{\CIn}$.\footnote{For $\KerSize = 1$ and $\KerSize = \Dim$, we verified that the estimate is close to tight by computing the $\ell_2$ and $\ell_1$ norms of Fourier transform of the predictor, respectively.} 

Table \ref{tab:outputchannelsMNIST} shows $\RkCoutHat{\ParVecFn(\U,\V)}{\KerSize}{\COut}$ across different values of $\KerSize$ and $\COut$. We see that for each kernel size, 
the differences in $\RkCoutHat{\ParVecFn(\U,\V)}{\KerSize}{\COut}$  across different settings of $\COut$ are minimal and are usually smaller than the standard deviation for fixed settings of $\COut$. This suggests that the induced regularizer is indeed invariant to the number of output channels, thus providing evidence for Hypothesis \ref{hyp:outputchannels} in the case of a single input channel.

\begin{figure}
\centering
\begin{subfigure}[b]{0.49\textwidth}
\centering
\includegraphics[width=\textwidth]{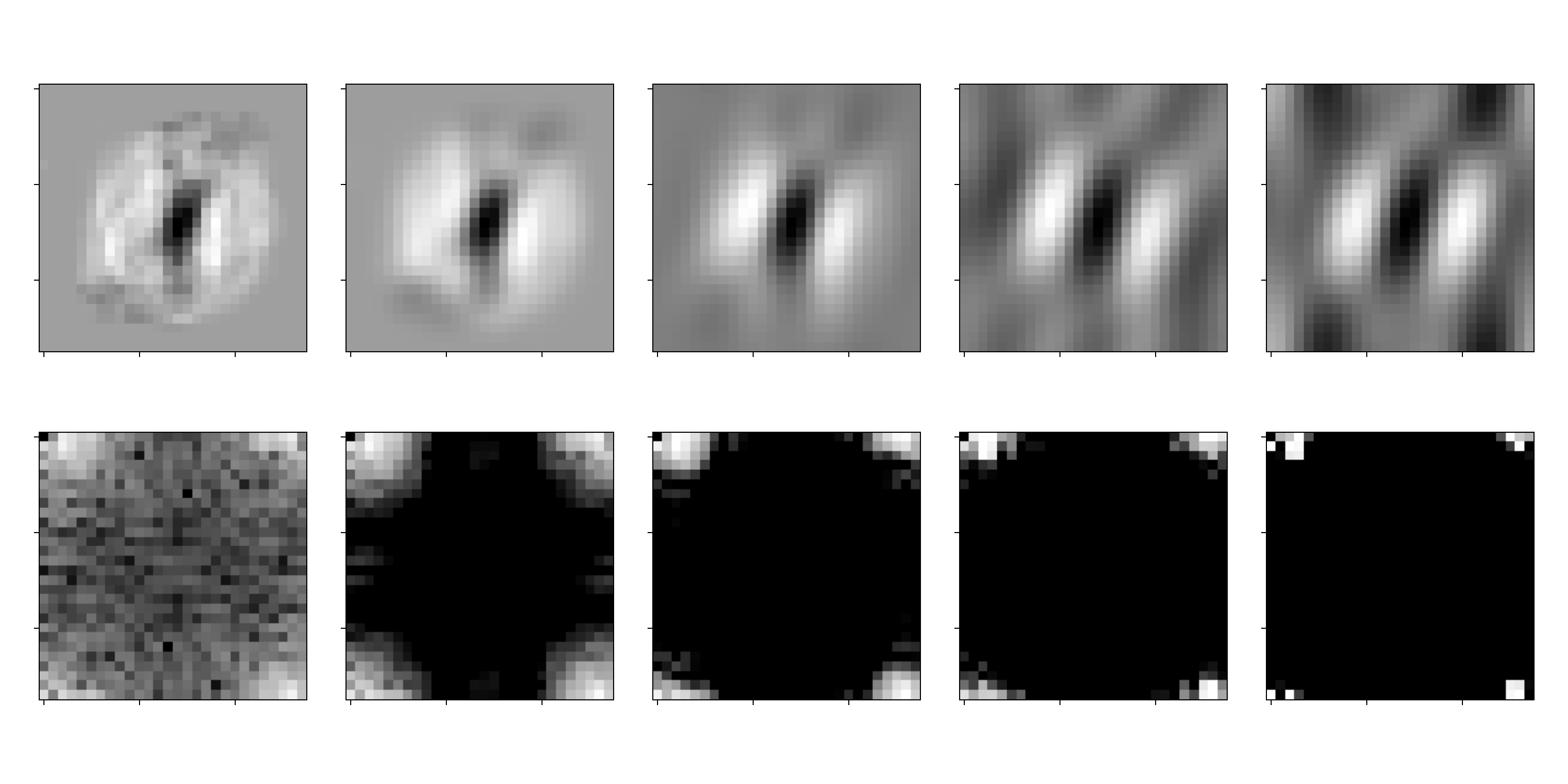}
\caption{\small $\COut = 1$}
\end{subfigure}
\begin{subfigure}[b]{0.49\textwidth}
\centering
\includegraphics[width=\textwidth]{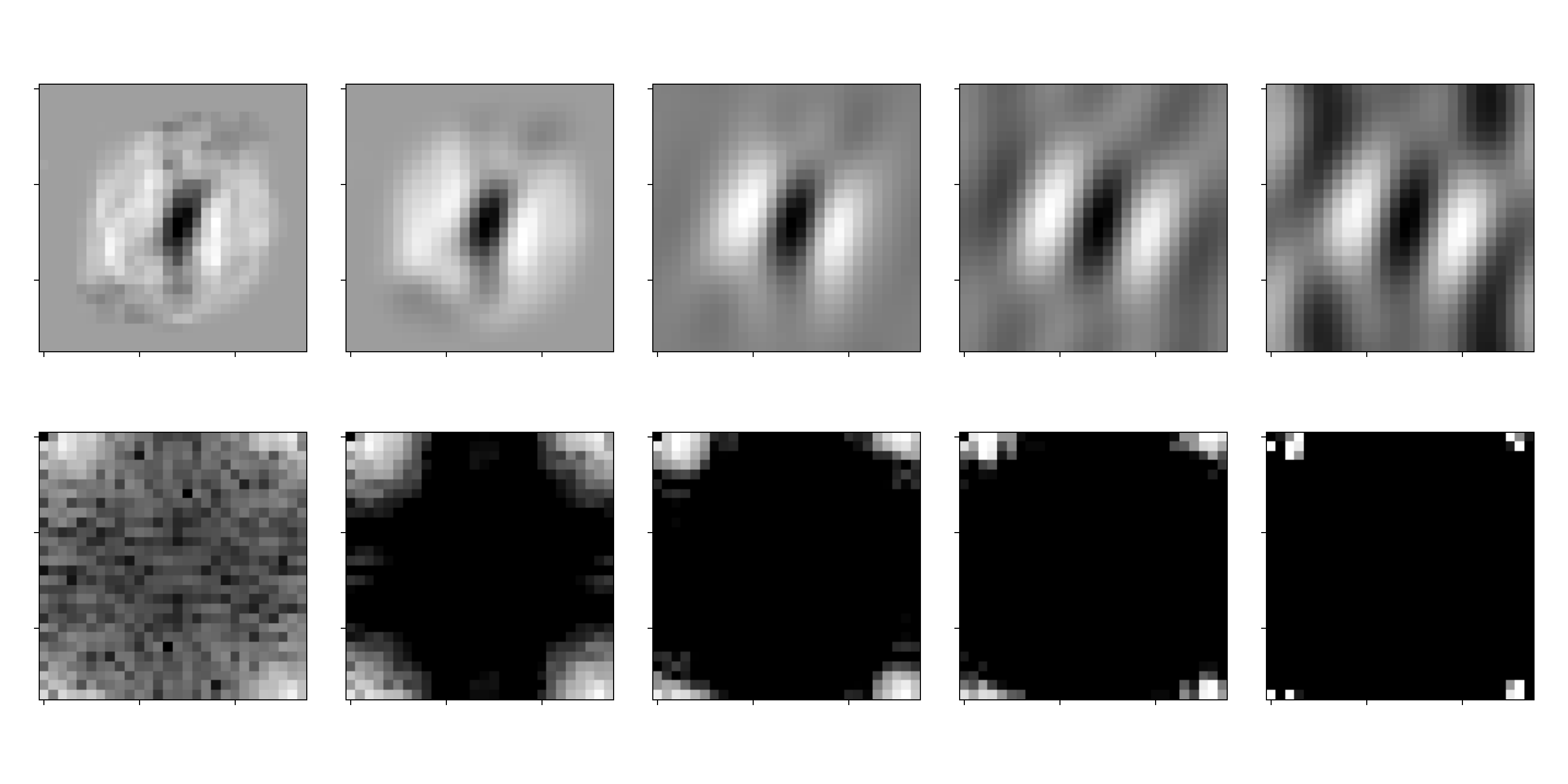}
\caption{\small $\COut = 2$}
\end{subfigure}
\begin{subfigure}[b]{0.49\textwidth}
\centering
\includegraphics[width=\textwidth]{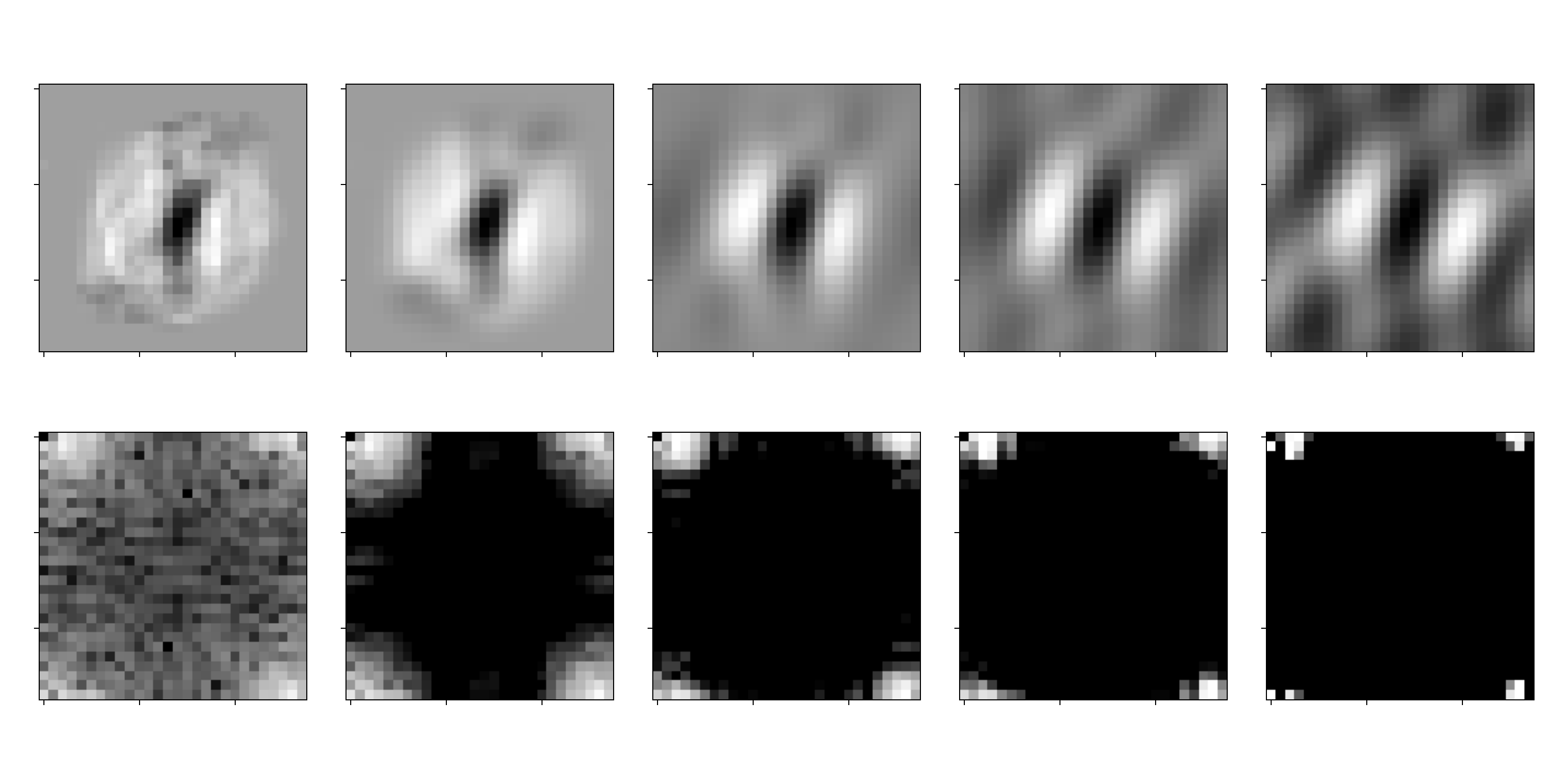}
\caption{\small $\COut = 4$}
\end{subfigure}
\begin{subfigure}[b]{0.49\textwidth}
\centering
\includegraphics[width=\textwidth]{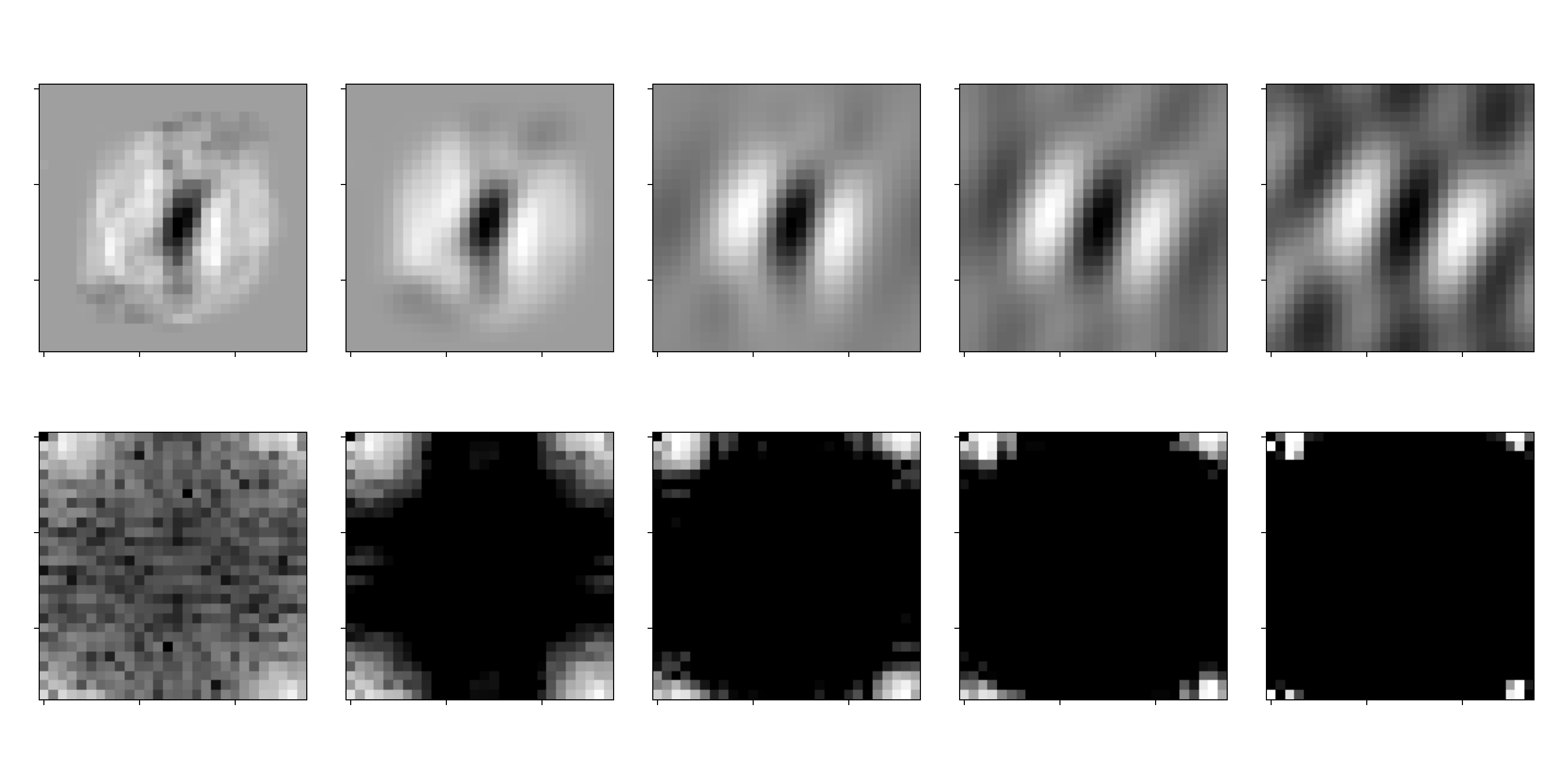}
\caption{\small $\COut = 8$}
\end{subfigure}
\caption{Linear predictors learned by  two layer linear convolutional network for the task of classifying digits $0$ and $1$ in MNIST. The sub-figures depict predictors learned by using gradient descent on the  exponential loss for overparameterized networks with $\COut=1,2,4$ and  kernel sizes $\KerSize\in \left\{(1, 1), (3, 3), (8, 8), (16, 16), (28, 28)\right\}$ (left to right). The top row in each sub--figure  is the signal domain representation $\ParVecFn{(\U,\V)}$, and the bottom row is the Fourier domain representation $\hat{\ParVecFn}{(\U,\V)}$.}
\label{fig:channels}
\end{figure}

\begin{figure}
    \centering
    \includegraphics[scale=0.2]{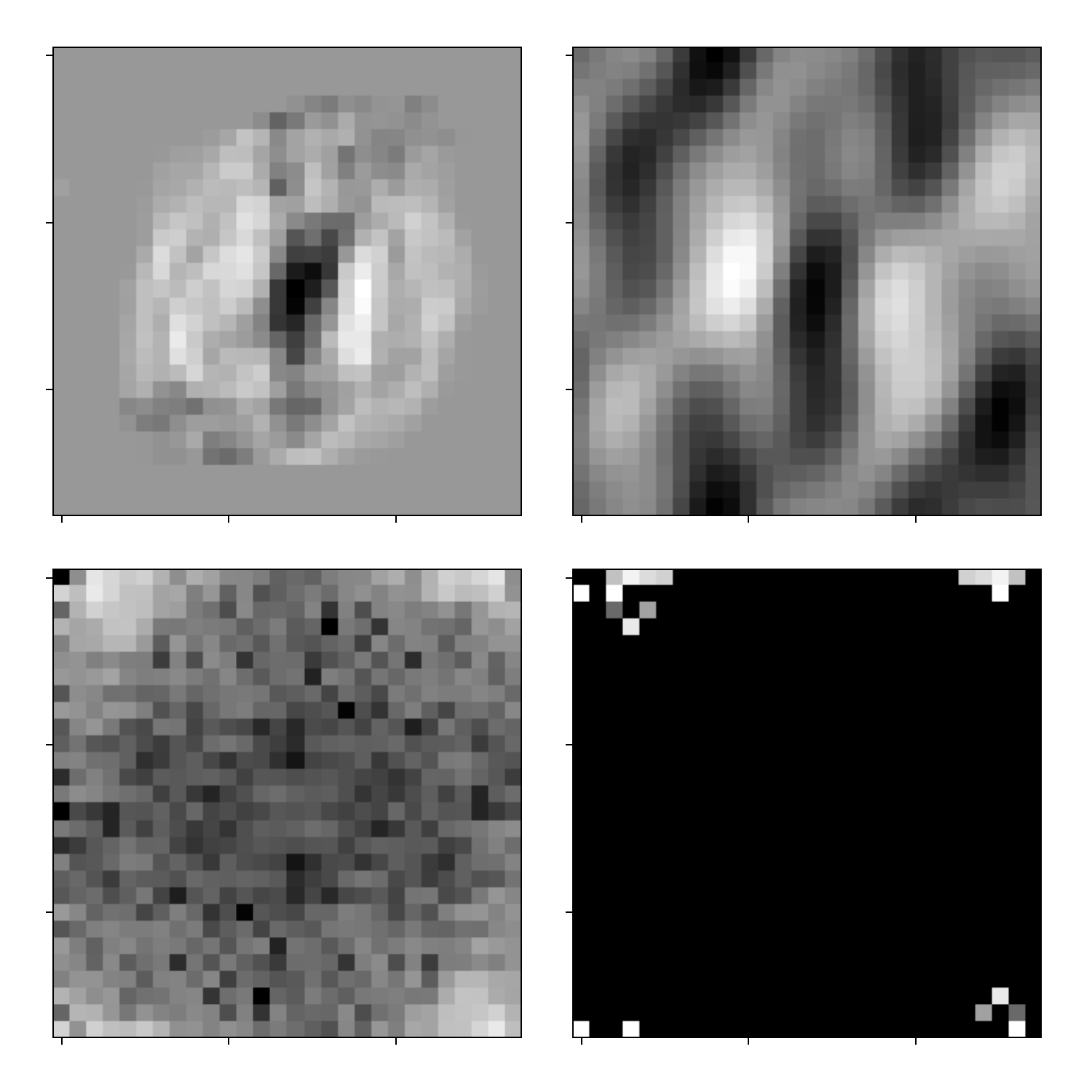}
    \caption{Explicit $\RkCoutOp{\KerSize}{\COut}$ margin predictor on sampled MNIST dataset for kernel sizes $\KerSize\in \left\{(1, 1), (28, 28)\right\}$ (left to right). The top row in each sub--figure  is the signal domain representation $\ParVecFn{(\U,\V)}$, and the bottom row is the Fourier domain representation $\hat{\ParVecFn}{(\U,\V)}$. }
    \label{fig:mnistcvx}
\end{figure}

\begin{table}[ht]
    \centering\footnotesize
 \begin{tabular}{c|c | c | c | c | c}
$\COut$ & $\KerSize = (1,1)$ &   $\KerSize = (3,3)$ &   $\KerSize = (8,8)$  &  $\KerSize = (16,16)$ &  $\KerSize = (28,28)$ \\ \hline
1 & $10.28\pm 2.34 \times 10^{-5}$ & $4.50\pm 1.51 \times 10^{-3}$ & $3.32\pm 5.35 \times 10^{-2}$ &  $3.15\pm 5.81 \times 10^{-2}$  &$2.84 \pm 1.16\times 10^{-1}$\\
2 &$10.28\pm 2.00  \times 10^{-5}$ & $4.50\pm 1.06\times 10^{-3}$  & $3.30\pm 3.13 \times 10^{-2}$ & $3.10\pm 2.63\times 10^{-2}$ & $2.79 \pm 1.27\times 10^{-1}$  \\
4 &  $10.28\pm 1.00 \times 10^{-5}$& $4.50\pm 7.48\times 10^{-4}$ & $3.30\pm 2.25 \times 10^{-2}$ &$3.10\pm 3.33 \times 10^{-2}$  & $2.77 \pm 8.20 \times 10^{-2}$ \\
8 & $10.28\pm 7.83 \times 10^{-5}$ &$4.50\pm 6.25\times 10^{-4}$ & $3.29\pm 1.68 \times 10^{-2}$ &$3.11\pm 3.27 \times 10^{-2}$ & $2.72 \pm 7.31 \times 10^{-2}$    \\
    \end{tabular}
    \caption{$\RkCoutHat{\ParVecFn(\U,\V)}{\KerSize}{\COut} = \|\U\|^2+\|\V\|^2$ of the predictor learned by gradient descent  on linear convolutional networks with different number of output channels $\COut$ and kernel sizes $\KerSize$ on the MNIST task. We show the mean over 10 trials as well as the standard deviations are also shown.}
    \label{tab:outputchannelsMNIST}
\end{table}

\paragraph{Invariance of learned predictors to $\COut$.}
While Hypothesis \ref{hyp:outputchannels} primarily pertains to the behavior of the induced regularizer, it also suggests that the predictor $\ParVec_{\mathrm{GD}}$ will also be independent of the number of channels so long as $\KerSize$ is strictly less than $\Dim$. If we overlook the caveats, we expect  gradient descent to implicitly learn a max--$\RkCoutOp{\KerSize}{\COut}$ margin predictor: $\min_{\w}\RkCout{\w}{\KerSize}{\COut}\st \forall_n y_n\innerprod{\w,\x_n}\ge1$. For $\KerSize < \Dim$, our theoretical findings suggest that the induced regularizer is a norm interpolating between the $\ell_2$ and $\ell_1$ norms. This would mean that there is a \textit{unique} global minimizer, and thus we would expect that $\ParVec_{\mathrm{GD}}$ to be invariant to $\COut$. 

To empirically validate this, we show the learned linear predictors for $\COut = \crl{1, 2, 4}$ in Figure \ref{fig:channels}. We observe that the linear predictors indeed visually appears to be invariant across different settings of $\COut$ for all for kernel sizes $\KerSize < \Dim$. For $\KerSize = (28, 28)$, there appear to be differences in the predictors---this likely arises from the fact that there are multiple linear predictors that minimize the $\ell_1$ norm on the dataset. We nonetheless emphasize that the \textit{induced regularizer} still appears to be invariant in this case, although the predictors are not. 

\subsubsection{Non-linear networks with ReLU activation}

 Although our theoretical results are restricted to networks with linear activations, it is nevertheless interesting to evaluate if our conclusions lead to useful heuristics for networks with non-linearity. As a simple demonstration, we repeat our experiment on MNIST on two-layer convolutional networks with ReLU non-linearity with and without bias parameters (\ie networks with a convolution layer, followed by ReLU layer, followed by linear layer).\footnote{The initialization scale was taken to be $0.005$ for networks without bias parameters and $0.01$ for networks with bias parameters.} As before, we first scale the weights learned by gradient descent such that the resulting predictor $f_{\mathrm{GD}}$ has unit margin on training data. We then consider the representation cost $\mathcal{R}_{\Phi_{\KerSize, \COut}}(f_{\mathrm{GD}})$, as per equation \eqref{eq:repcost-functions}, given by the minimum $\ell_2$ norm of the weights needed to realize $f$. We consider the approximation of $\mathcal{R}_{\Phi_{\KerSize, \COut}}(f_{\mathrm{GD}})$ given by  $\hat{\mathcal{R}}_{\Phi_{\KerSize, \COut}}(f_{\mathrm{GD}}) := \|\U\|_2^2+\|\V\|^2_2$ where $\U$ and $\V$ are the weights learned by gradient descent. (As before, strictly speaking, this is only an \textit{upper bound} on the representation cost $\mathcal{R}_{\Phi_{\KerSize, \COut}}(f_{\mathrm{GD}})$.)

 Table \ref{tab:RELUnobias} and  Table \ref{tab:RELUbias} show $\hat{\mathcal{R}}_{\Phi_{\KerSize, \COut}}(f_{\mathrm{GD}}) := \|\U\|_2^2+\|\V\|^2_2$ across different settings of $\COut$ and $\KerSize$, for networks with no bias as well as networks with bias parameters on both the convolution layer and the fully connected layer.\footnote{We note that the representation cost includes the magnitude of the \textit{weights} but not the magnitude of the \textit{biases}.} Like in the case of linear convolutional neural networks, $\hat{\mathcal{R}}_{\Phi_{\KerSize, \COut}}(f_{\mathrm{GD}})$ is consistent across different settings of $\COut$. This suggests that the implicit bias from gradient might result in predictors that are independent of the number of output channels, even when there is a ReLU layer, and Hypothesis \ref{hyp:outputchannels} might hold in much more generality than the scope of theoretical findings.

 \begin{table}[ht]
    \centering\small
    \begin{tabular}{c|c | c | c | c | c}
$\COut$ & $\KerSize: (1,1)$ &  $\KerSize: (3,3)$ & $\KerSize: (8,8)$ &  $\KerSize:  (16, 16)$  &  $\KerSize: (28, 28)$ \\ \hline
1 & 11.412 &5.160  & 3.998  & 3.785 & 3.520 \\
2 & 11.413 & 5.155  &  3.964 & 3.721  & 3.539 \\
4 & 11.414 & 5.153  & 3.966 & 3.719  &3.448  \\ 
8 & 11.415 & 5.156  & 3.971  & 3.738  & 3.498                 
    \end{tabular}
    \caption{\small$\RkCoutHat{f_{\mathrm{GD}}}{\KerSize}{\COut} = \|\U\|^2+\|\V\|^2$ of the predictor learned by gradient descent  on ReLU convolutional networks without bias parameters, with different number of output channels $\COut$ and kernel sizes $\KerSize$, on the MNIST task. The values shown are the medians taken over 5 trials. }
    \label{tab:RELUnobias}
\end{table}

\begin{table}[ht]
    \centering\small
    \begin{tabular}{c|c | c | c | c | c}
$\COut$ & $\KerSize: (1,1)$ &  $\KerSize: (3,3)$ & $\KerSize: (8,8)$ &  $\KerSize:  (16, 16)$  &  $\KerSize: (28, 28)$ \\ \hline
1 & 10.581 &4.948 & 3.875 & 3.714 & 3.519 \\
2 & 10.571 & 4.945 & 3.910  & 3.698 &3.413  \\
4 & 10.578 & 4.945 & 3.912 &  3.712 & 3.399 \\ 
8 & 10.576  &4.946  & 3.881 &  3.697 &   3.437                          
    \end{tabular}
    \caption{\small$\RkCoutHat{f_{\mathrm{GD}}}{\KerSize}{\COut} = \|\U\|^2+\|\V\|^2$ of the predictor learned by gradient descent  on ReLU convolutional networks with bias on both layers, with different number of output channels $\COut$ and kernel sizes $\KerSize$, on the MNIST task. The values shown are the medians taken over 5 trials. }
    \label{tab:RELUbias}
\end{table}

 We note that we observed that gradient descent sometimes leads to outliers where $\mathcal{R}_{\Phi_{\KerSize, \COut}}(f_{\mathrm{GD}})$ is very large. For example, when $\KerSize = 1$, the values of $\mathcal{R}_{\Phi_{\KerSize, \COut}}(f)$ are $[11.412, 11.412, 109.471, 11.412, 11.412]$, where $109.471$ appears to be an outlier. We anticipate that this outlier arises because gradient descent converges to a stationary point, rather than a local minima, of the max-$\ell_2$ margin problem in parameter space (see the discussion in Section \ref{sec:relatedwork}). Since our goal is to investigate the behavior of gradient descent when it does lead to global minima of the $max-\mathcal{R}_{\Phi}$ margin problem, we compute the median so that these data points do not affect our estimate.

\subsection{Impact of the number of channels on CIFAR-10}
We carry out a similar investigation of Hypothesis \ref{hyp:outputchannels} on the CIFAR-10 dataset for networks with $3$-channel inputs. As discussed in Section \ref{sec:mult-input-channel}, we expect that the induced regularizer is \textit{not} independent of the number of output channels for $\COut < \CIn$, but begins to exhibit invariance once $\COut \ge \CIn$. On the CIFAR-10 dataset, where there are 3 input channels, we would expect to see invariance once $\COut \ge 3$.

To demonstrate  this, we repeat the experimental setup on $32 \times 32$ CIFAR-10 images on networks with multiple output channels $\COut \in \left\{1, 2, 4, 8\right\}$ and across different kernel sizes. As described earlier, we scale the weights learned by gradient descent $\U,\V$ such that the linear predictor $\ParVecFn(\Matrix{U}, \Matrix{V}) = \ParVec_{\mathrm{GD}}$ has unit margin on training data. Since it is difficult to directly compute $\RkCinCout{\ParVec_{\mathrm{GD}}}{\KerSize}{\COut}{\CIn}$, we turn to an approximation, as we did in the case of single-input channels. We compute  $\RkCinCoutHat{\ParVecFn(\U,\V)}{\KerSize}{\COut}{\CIn} := \|\U\|^2+\|\V\|^2$ which strictly speaking, this is only an \textit{upper bound} on the induced regularizer $\RkCinCout{\ParVecFn(\U,\V)}{\KerSize}{\COut}{\CIn}$. 

Table \ref{tab:outputchannelsCIFAR} shows $\RkCinCoutHat{\ParVecFn(\U,\V)}{\KerSize}{\COut}{3}$ across different values of $\KerSize$ and $\COut$. We see that for each kernel size, 
the differences in $\RkCinCoutHat{\ParVecFn(\U,\V)}{\KerSize}{\COut}{3}$  across different settings of $\COut$ are minimal, as long as $\COut \ge 3$ (and often, even when $\COut \ge 2$). This suggests that the induced regularizer is indeed invariant to the number of output channels when $\COut \ge 3$, thus providing evidence for Hypothesis \ref{hyp:outputchannels} in the case of multiple input channels. Moreover, there are non-trivial differences in $\RkCinCoutHat{\ParVecFn(\U,\V)}{\KerSize}{\COut}{3}$ for $\COut = 1$ and larger $\COut$---this aligns with our theoretical findings in Section \ref{sec:mult-input-channel} that the induced regularizer does depend on $\COut$ when it is below $\CIn$. 

While Hypothesis \ref{hyp:outputchannels} primarily pertains to the behavior of the induced regularizer, we would also expected that there is a unique global minimizer in most cases, for reasons similar to for the single input channel case. Thus, we would expect $\ParVec_{\mathrm{GD}}$ to be invariant to $\COut$ as long as $\COut \ge \CIn = 3$. To empirically validate this, we show the learned linear predictors for $\COut = \crl{1, 2, 3, 4, 8}$ in Figures \ref{fig:outputchannelCIFAR1}-\ref{fig:outputchannelCIFAR8}. We observe that the linear predictors indeed visually appears to be invariant across different settings of $\COut$. 

\begin{table}[ht]
    \centering\small
 \begin{tabular}{c|c | c | c | c}
$\COut$ & $\KerSize = (1,1)$ &   $\KerSize = (3,3)$ &   $\KerSize = (8,8)$  &  $\KerSize = (20,20)$  \\ \hline
1 & 246.04 & 215.21 & 202.27 & 131.16  \\
2 & 246.26 & 182.77 & 168.40 & 124.66   \\
3 & 245.98 & 182.80 & 165.32 & 123.56  \\
4 & 246.29 & 182.83 & 164.50 & 123.37    \\
8 &245.58 & 182.82 & 164.86 &  123.59  \\
    \end{tabular}
    \caption{$\RkCinCoutHat{\ParVecFn(\U,\V)}{\KerSize}{\COut}{3} = \|\U\|^2+\|\V\|^2$ of the predictor learned by gradient descent on linear convolutional networks with different number of output channels $\COut$ and kernel sizes $\KerSize$ on the CIFAR-10 task. } 
    \label{tab:outputchannelsCIFAR}
\end{table}

\begin{figure}
\centering
\begin{subfigure}[b]{0.49 \textwidth}
\centering
\includegraphics[width=\textwidth]{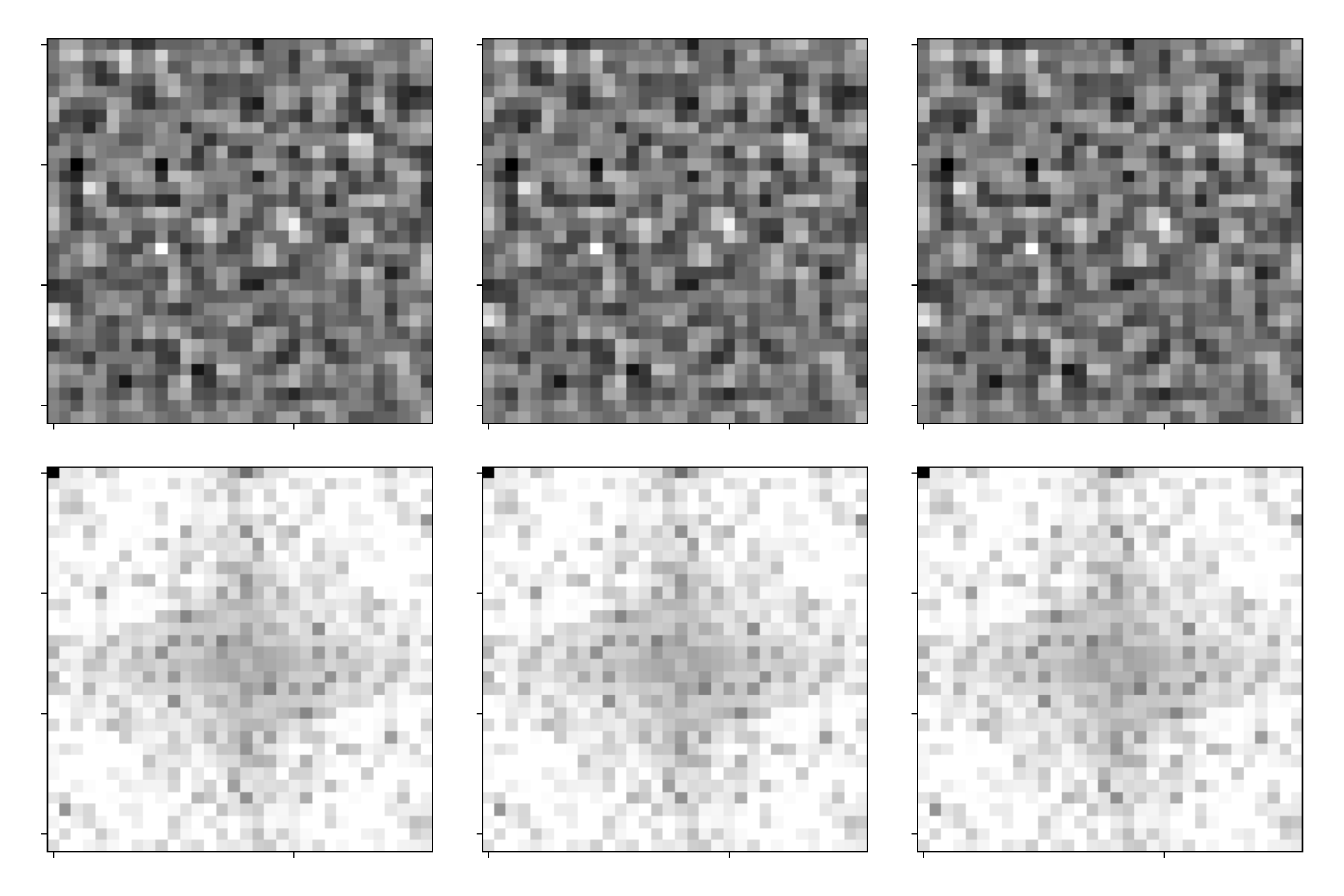}
\caption{\small $\COut = 1$}
\end{subfigure}
\begin{subfigure}[b]{0.49  \textwidth}
\centering
\includegraphics[width=\textwidth]{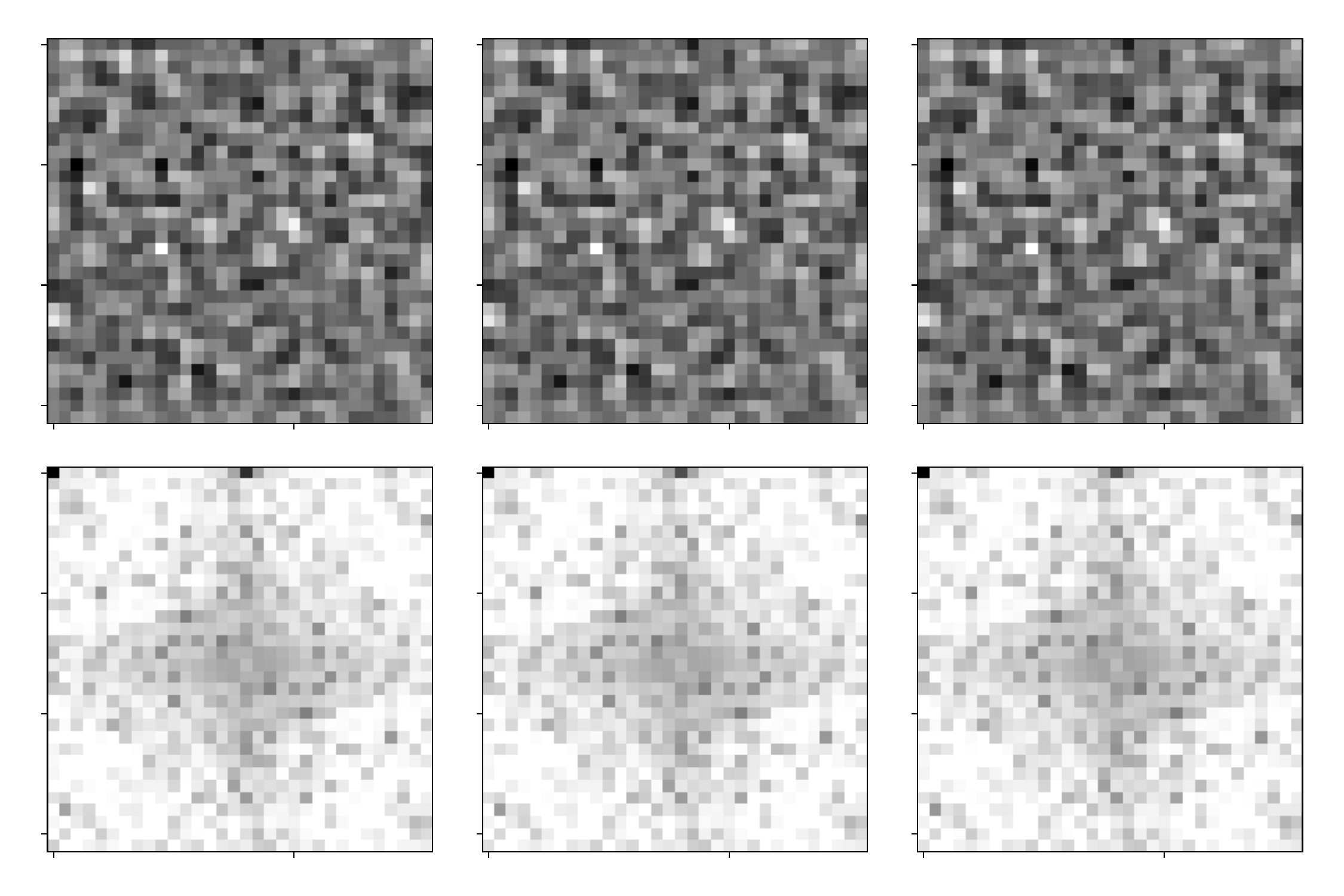}
\caption{\small $\COut = 2$}
\end{subfigure}
\begin{subfigure}[b]{0.49  \textwidth}
\centering
\includegraphics[width=\textwidth]{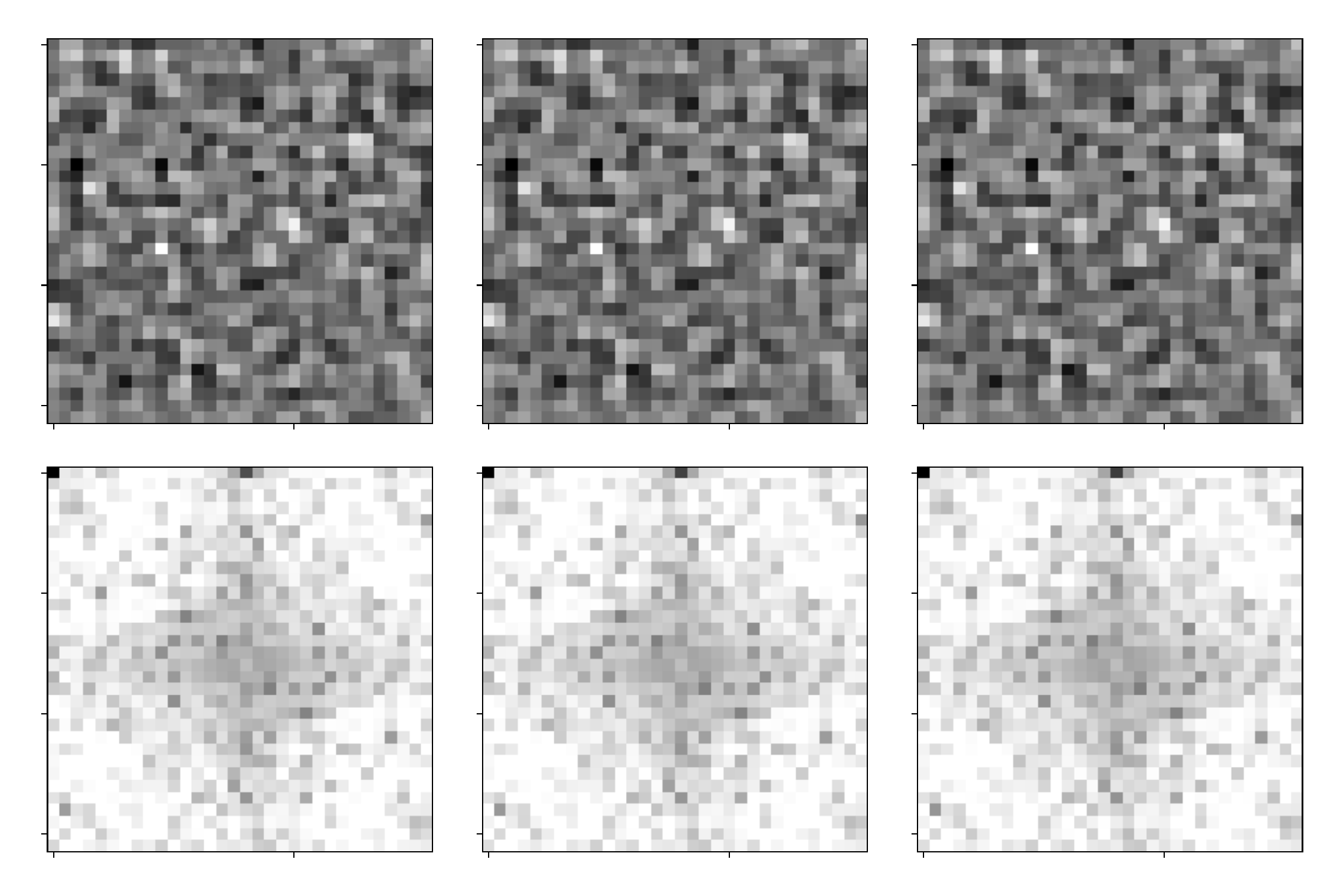}
\caption{\small $\COut = 3$}
\end{subfigure}
\begin{subfigure}[b]{0.49  \textwidth}
\centering
\includegraphics[width=\textwidth]{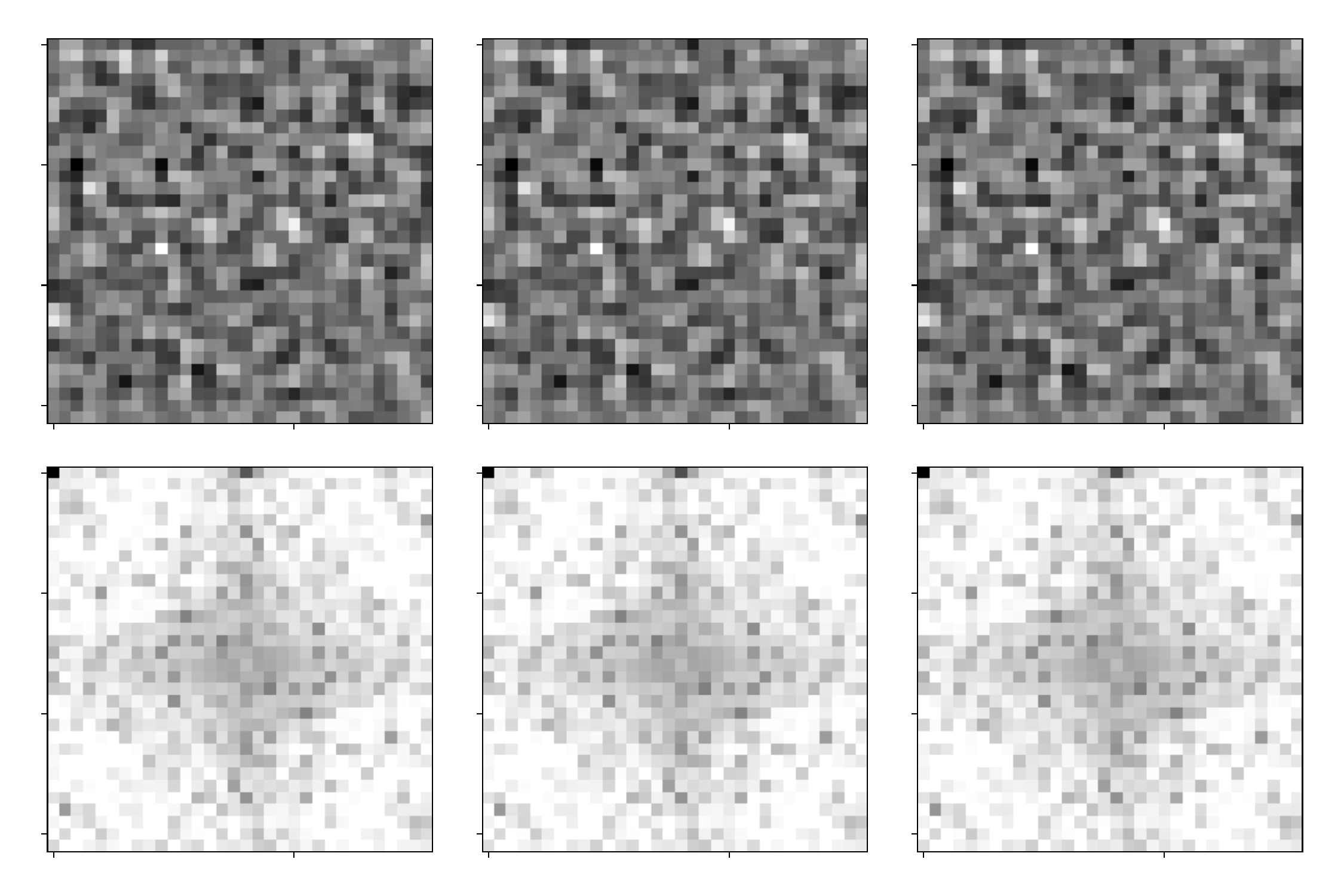}
\caption{\small $\COut = 4$}
\end{subfigure}
\begin{subfigure}[b]{ \textwidth}
\centering
\includegraphics[width=0.49 \textwidth]{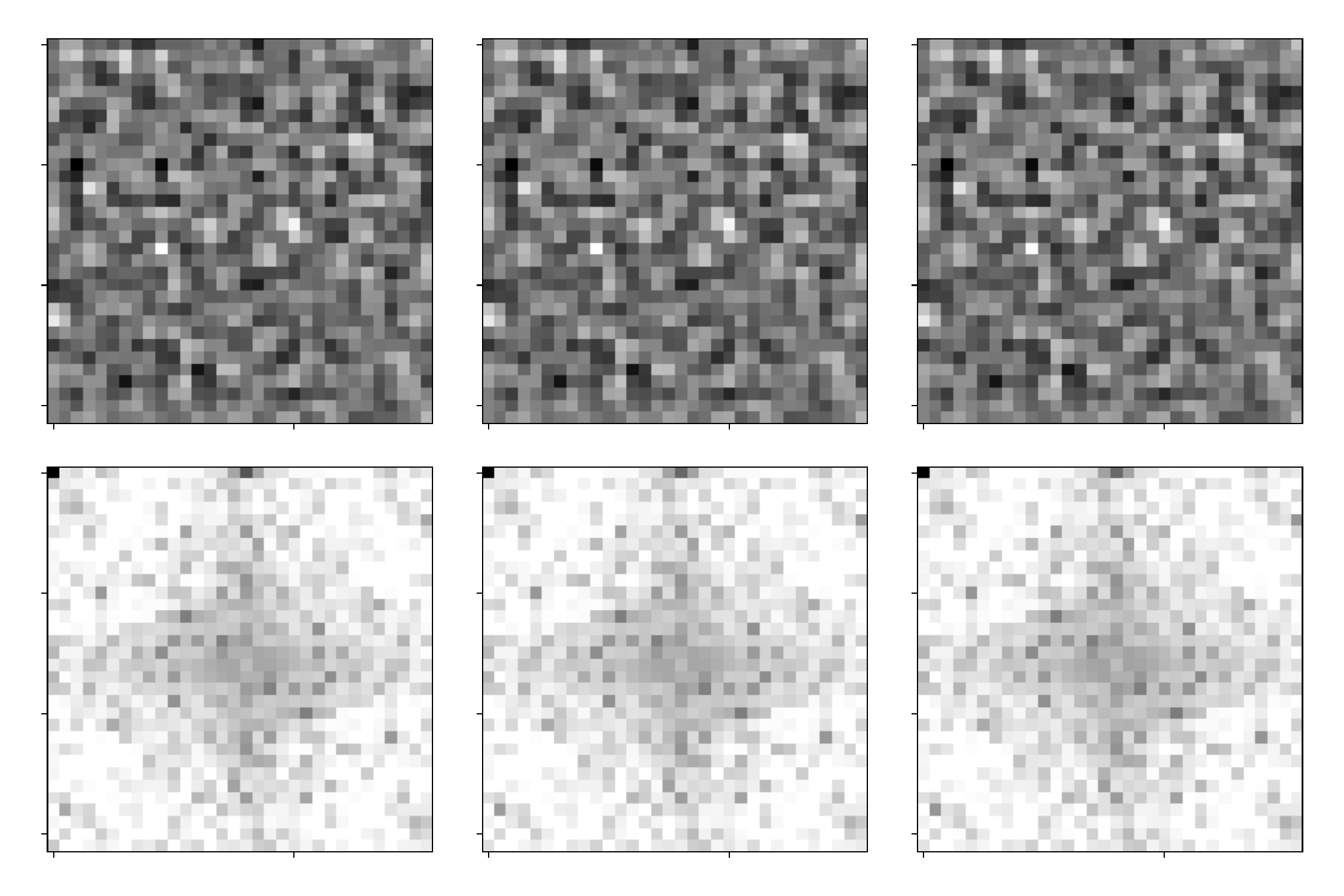}
\caption{\small $\COut = 8$}
\end{subfigure}
\caption{Linear predictors learned by  two layer linear convolutional network on CIFAR-10 task. The sub-figures depict predictors learned by using gradient descent on the  exponential loss for overparameterized networks with $\COut \in \left\{1, 2, 3, 4, 8\right\}$ and kernel size $\KerSize = (1, 1)$. The top row in each sub--figure  is the signal domain representation $\ParVecFn{(\U,\V)}$, and the bottom row is the Fourier domain representation $\hat{\ParVecFn}{(\U,\V)}$. }
\label{fig:outputchannelCIFAR1}
\end{figure}

\begin{figure}
\centering
\begin{subfigure}[b]{0.49 \textwidth}
\centering
\includegraphics[width=\textwidth]{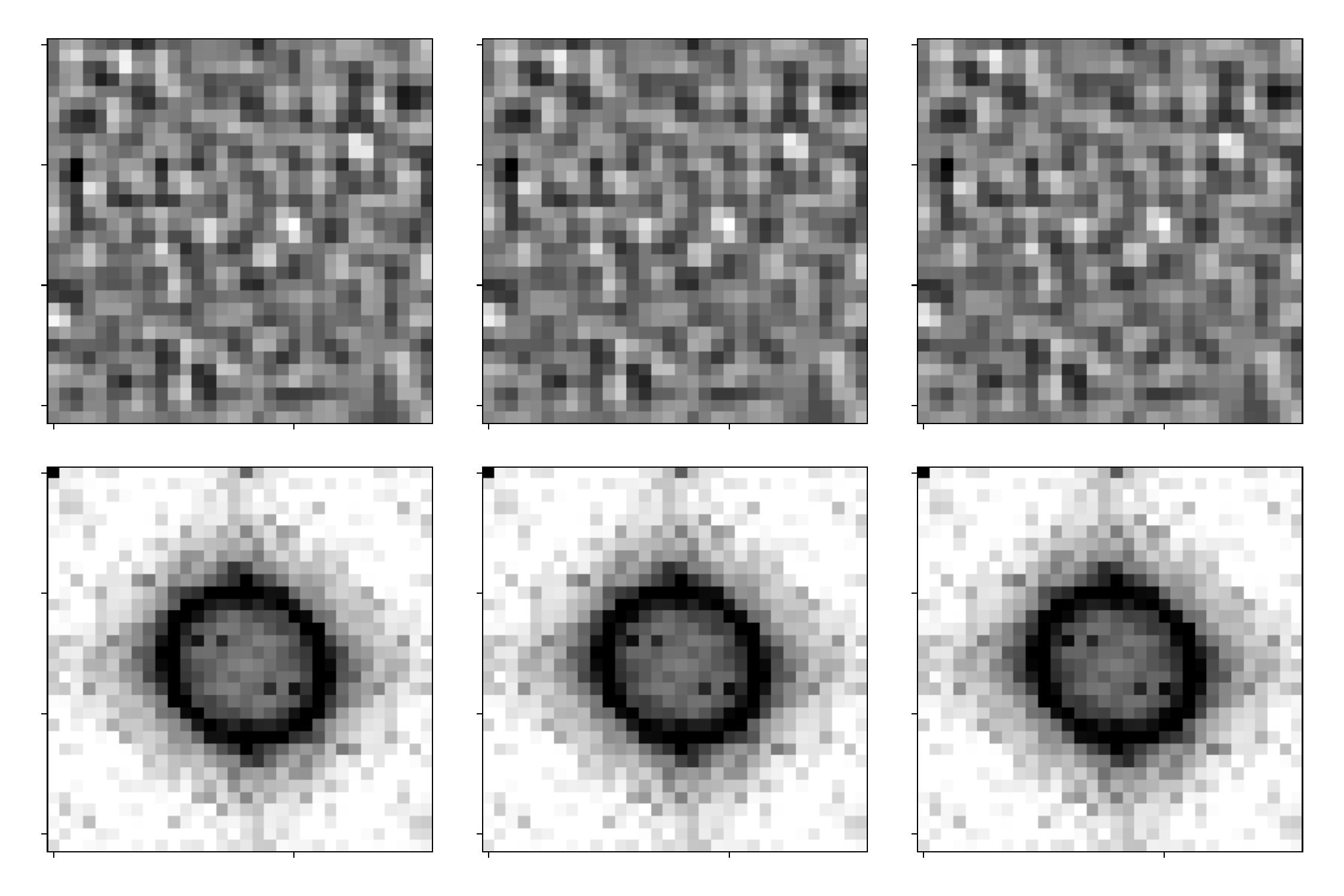}
\caption{\small $\COut = 1$}
\end{subfigure}
\begin{subfigure}[b]{0.49  \textwidth}
\centering
\includegraphics[width=\textwidth]{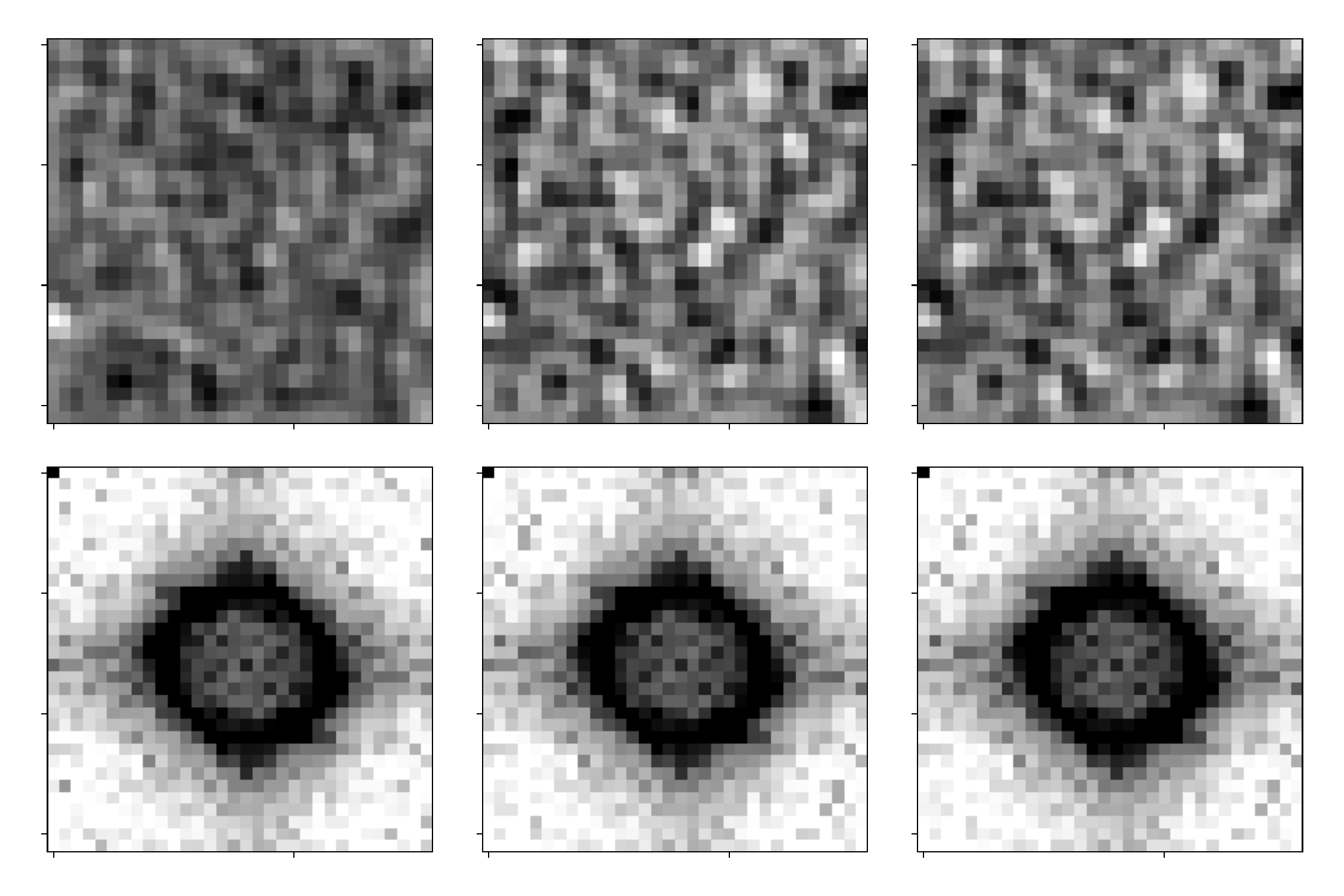}
\caption{\small $\COut = 2$}
\end{subfigure}
\begin{subfigure}[b]{0.49  \textwidth}
\centering
\includegraphics[width=\textwidth]{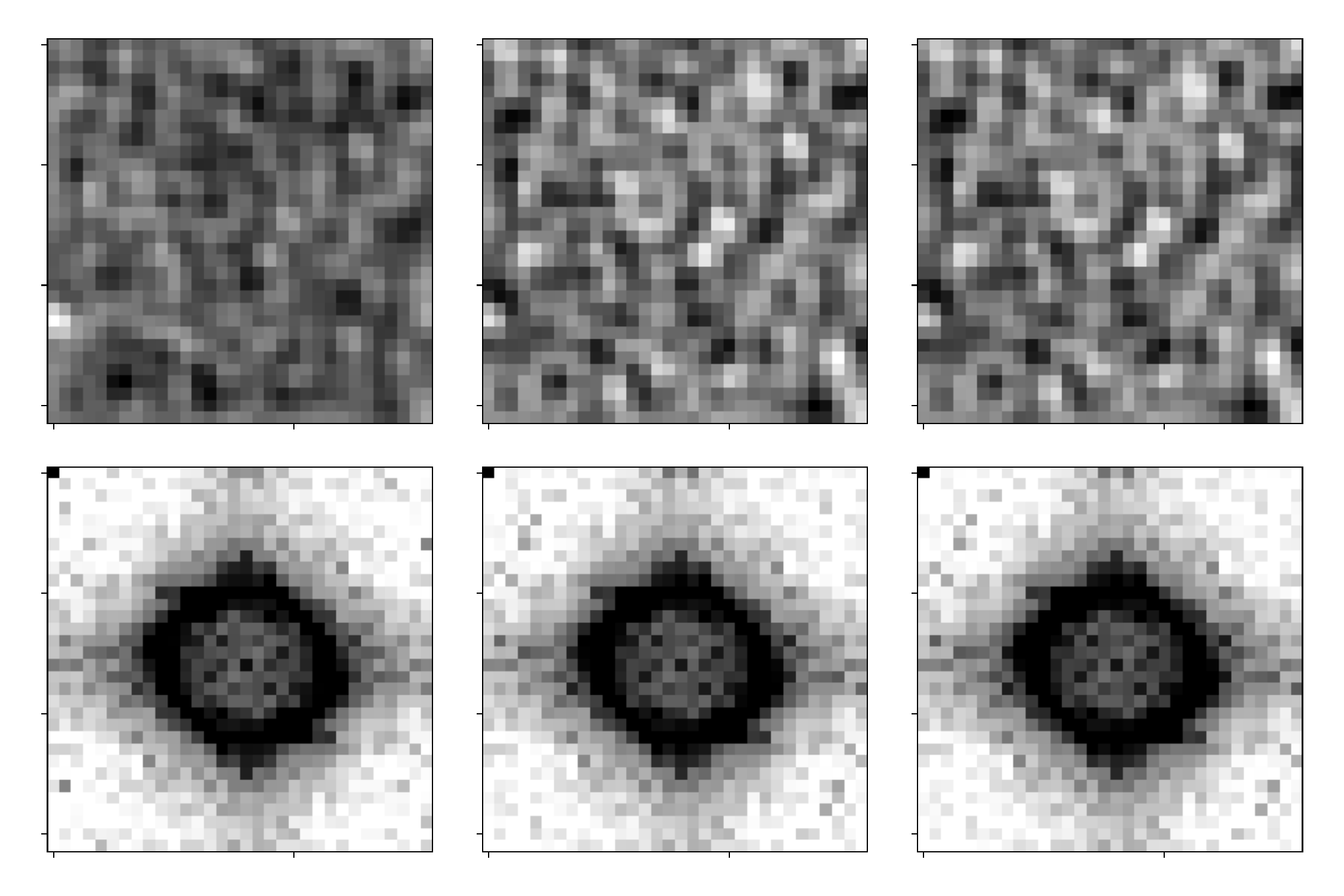}
\caption{\small $\COut = 3$}
\end{subfigure}
\begin{subfigure}[b]{0.49  \textwidth}
\centering
\includegraphics[width=\textwidth]{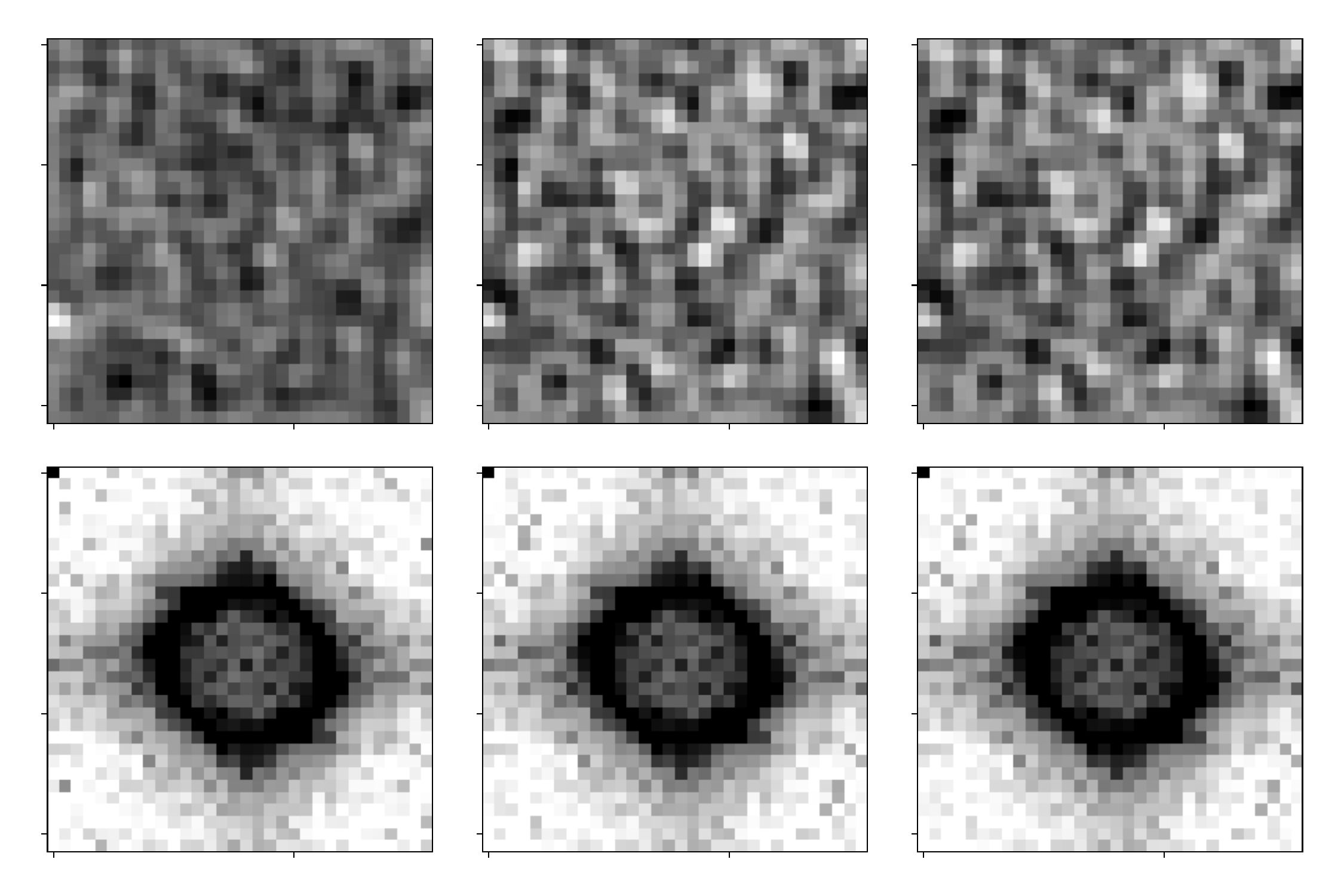}
\caption{\small $\COut = 4$}
\end{subfigure}
\begin{subfigure}[b]{ \textwidth}
\centering
\includegraphics[width=0.49 \textwidth]{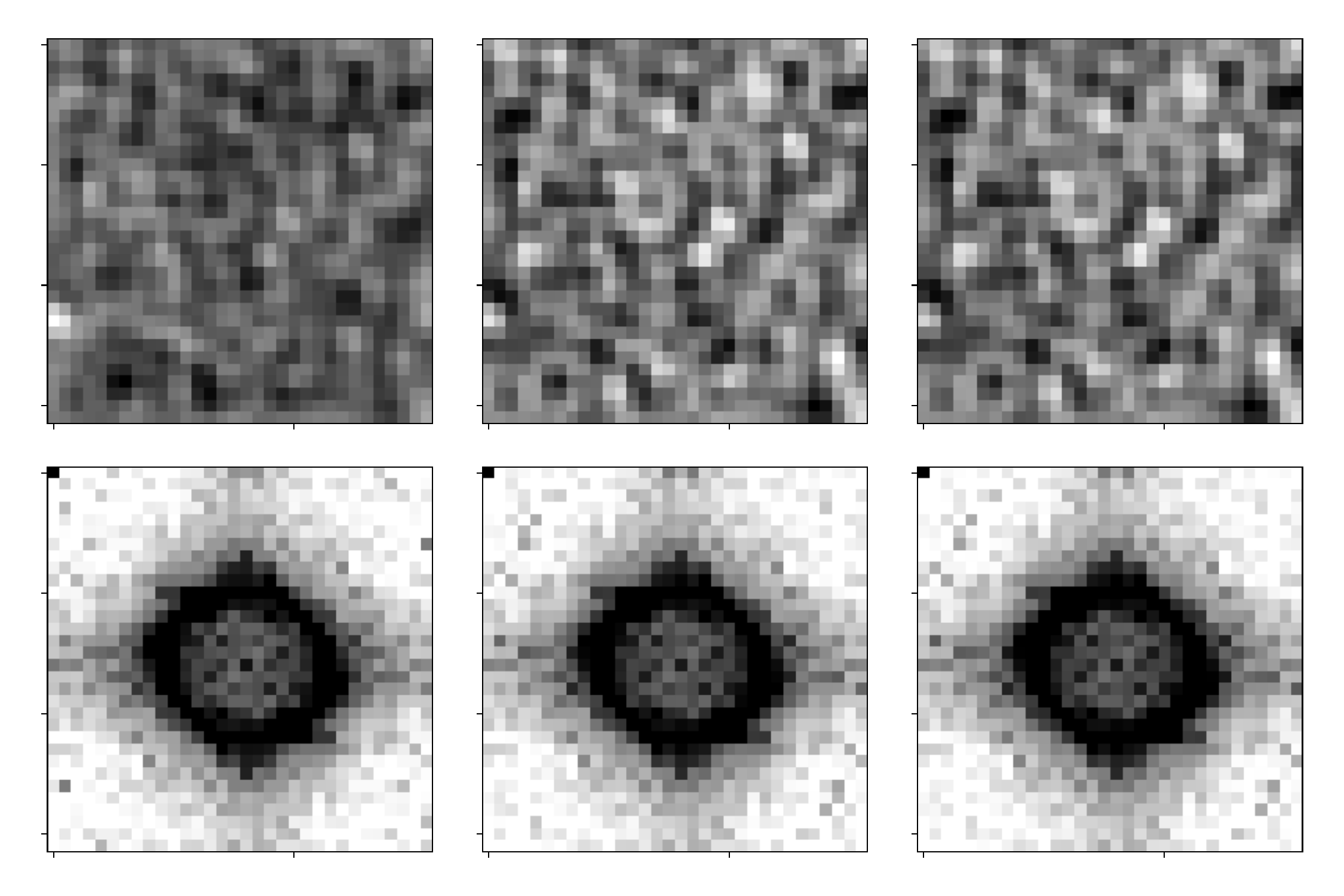}
\caption{\small $\COut = 8$}
\end{subfigure}
\caption{Linear predictors learned by  two layer linear convolutional network on CIFAR-10 task. The sub-figures depict predictors learned by using gradient descent on the  exponential loss for overparameterized networks with $\COut \in \left\{1, 2, 3, 4, 8\right\}$ and kernel size $\KerSize = (3, 3)$. The top row in each sub--figure  is the signal domain representation $\ParVecFn{(\U,\V)}$, and the bottom row is the Fourier domain representation $\hat{\ParVecFn}{(\U,\V)}$. }
\label{fig:outputchannelCIFAR3}
\end{figure}

\begin{figure}
\centering
\begin{subfigure}[b]{0.49 \textwidth}
\centering
\includegraphics[width=\textwidth]{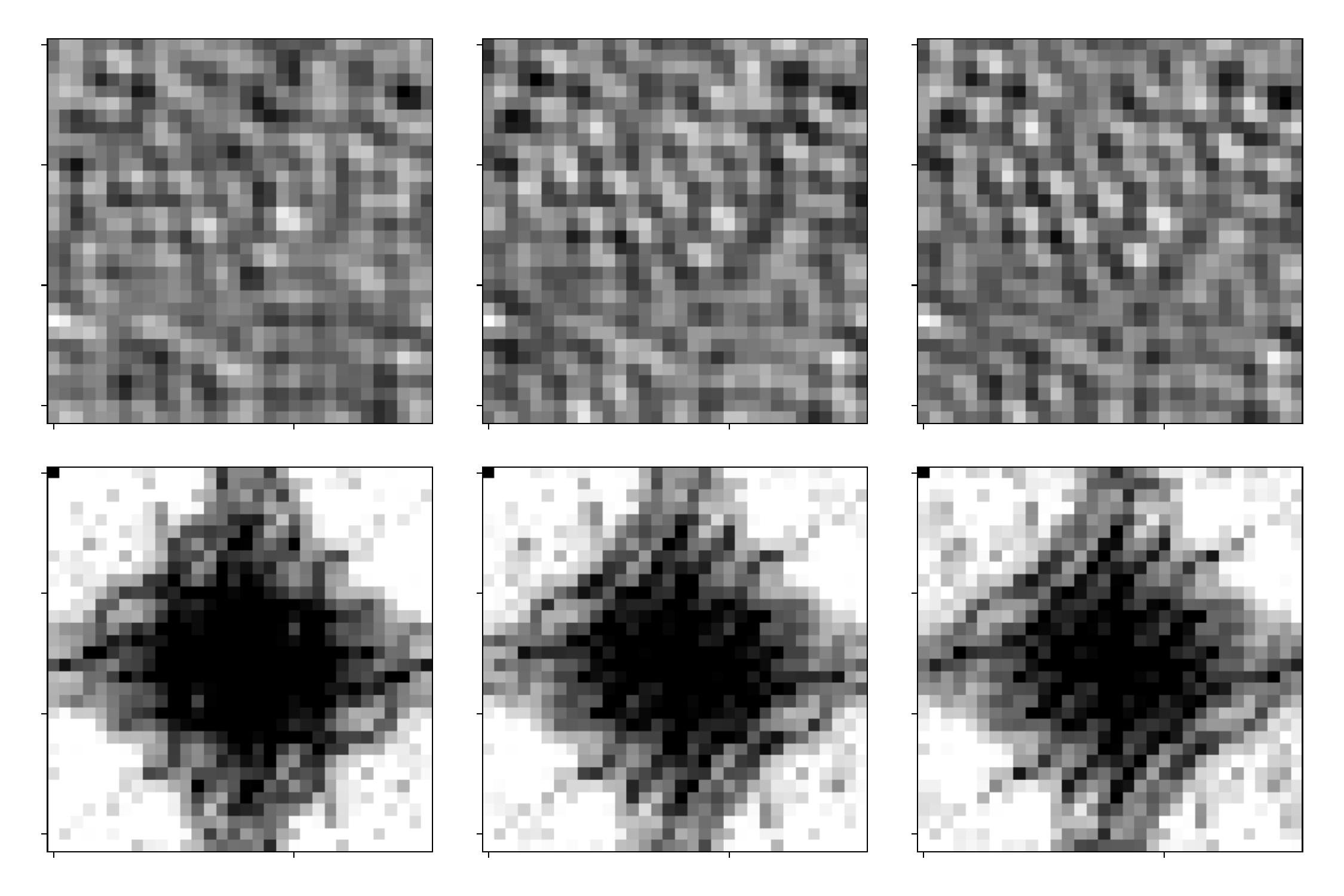}
\caption{\small $\COut = 1$}
\end{subfigure}
\begin{subfigure}[b]{0.49  \textwidth}
\centering
\includegraphics[width=\textwidth]{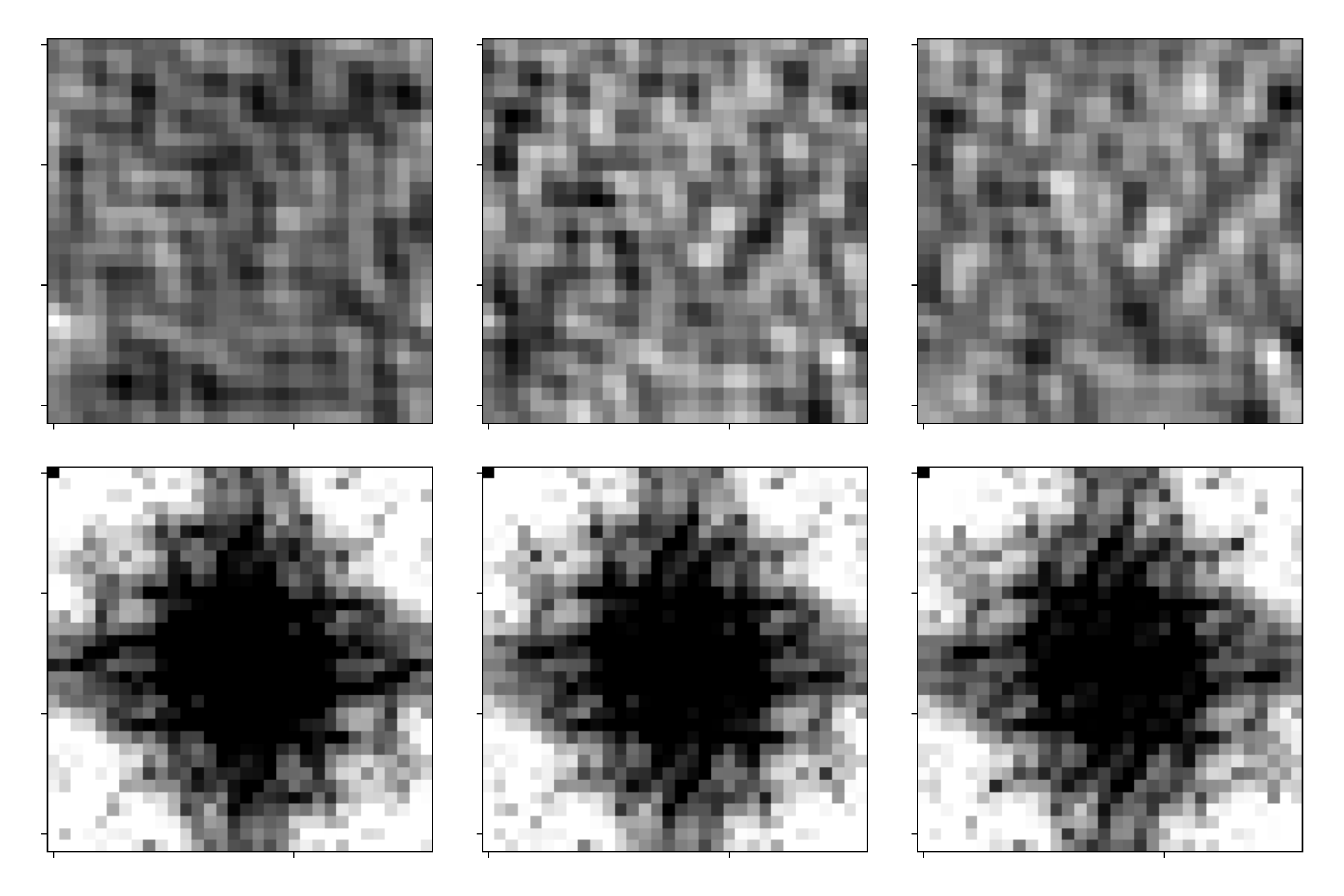}
\caption{\small $\COut = 2$}
\end{subfigure}
\begin{subfigure}[b]{0.49  \textwidth}
\centering
\includegraphics[width=\textwidth]{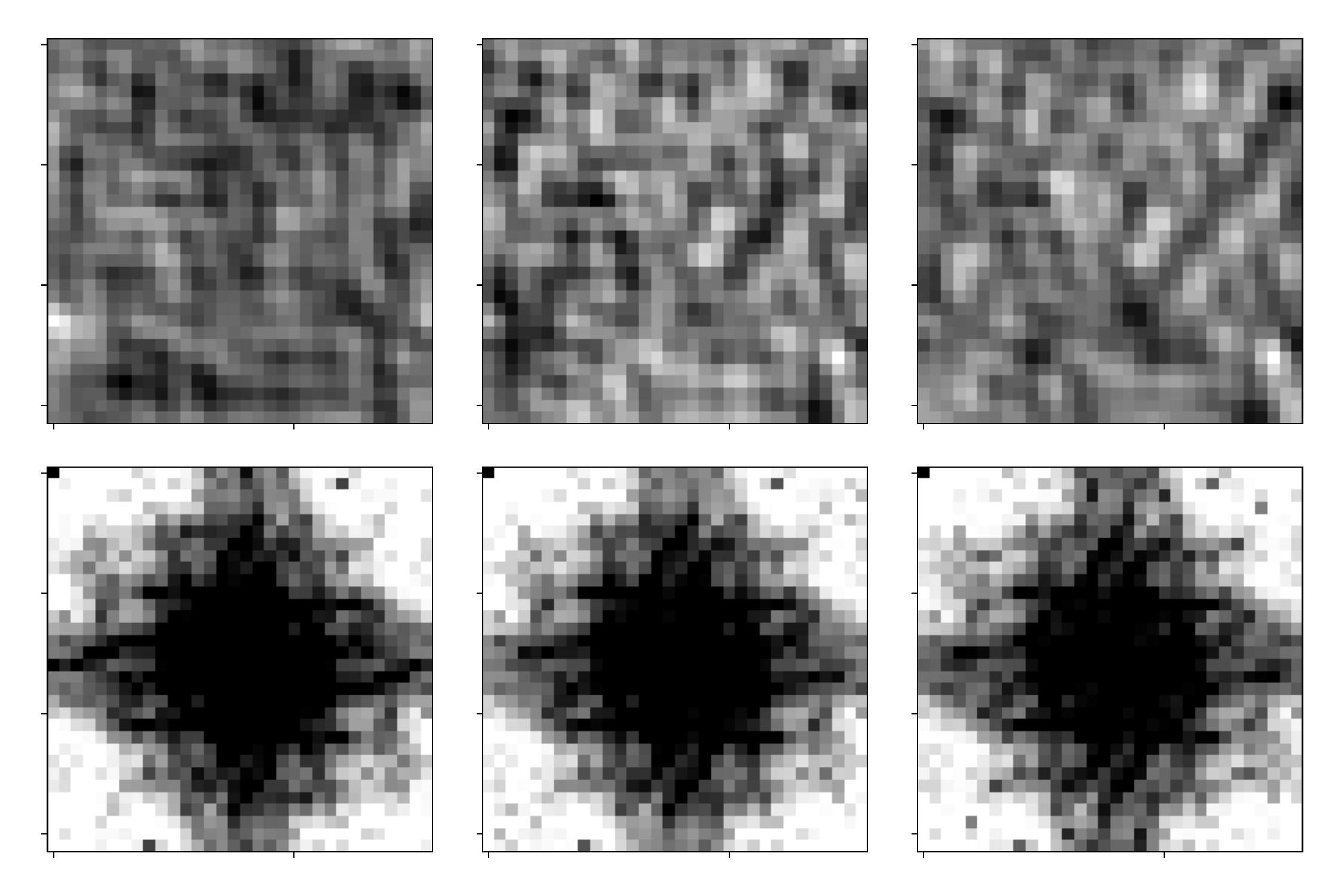}
\caption{\small $\COut = 3$}
\end{subfigure}
\begin{subfigure}[b]{0.49  \textwidth}
\centering
\includegraphics[width=\textwidth]{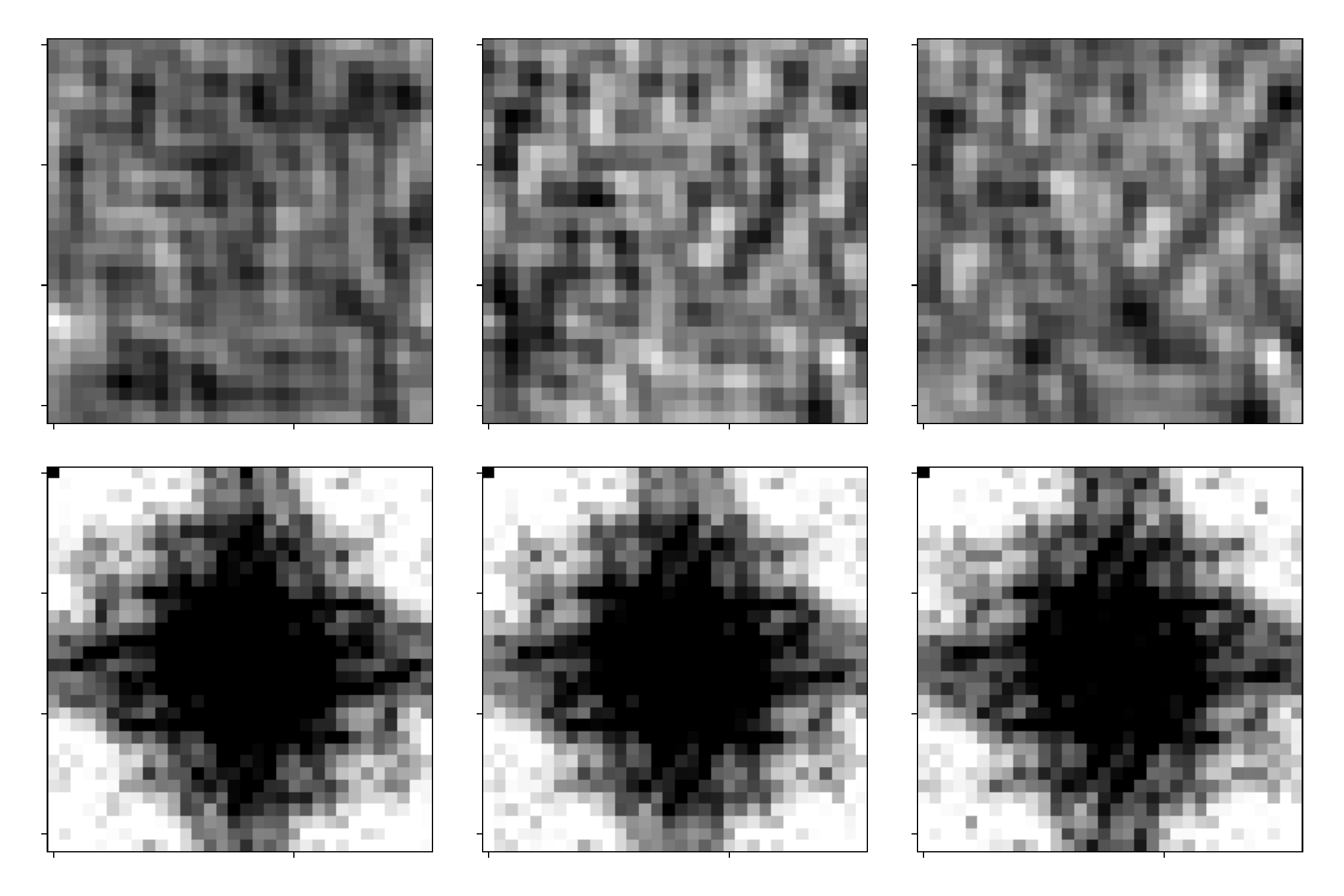}
\caption{\small $\COut = 4$}
\end{subfigure}
\begin{subfigure}[b]{ \textwidth}
\centering
\includegraphics[width=0.49 \textwidth]{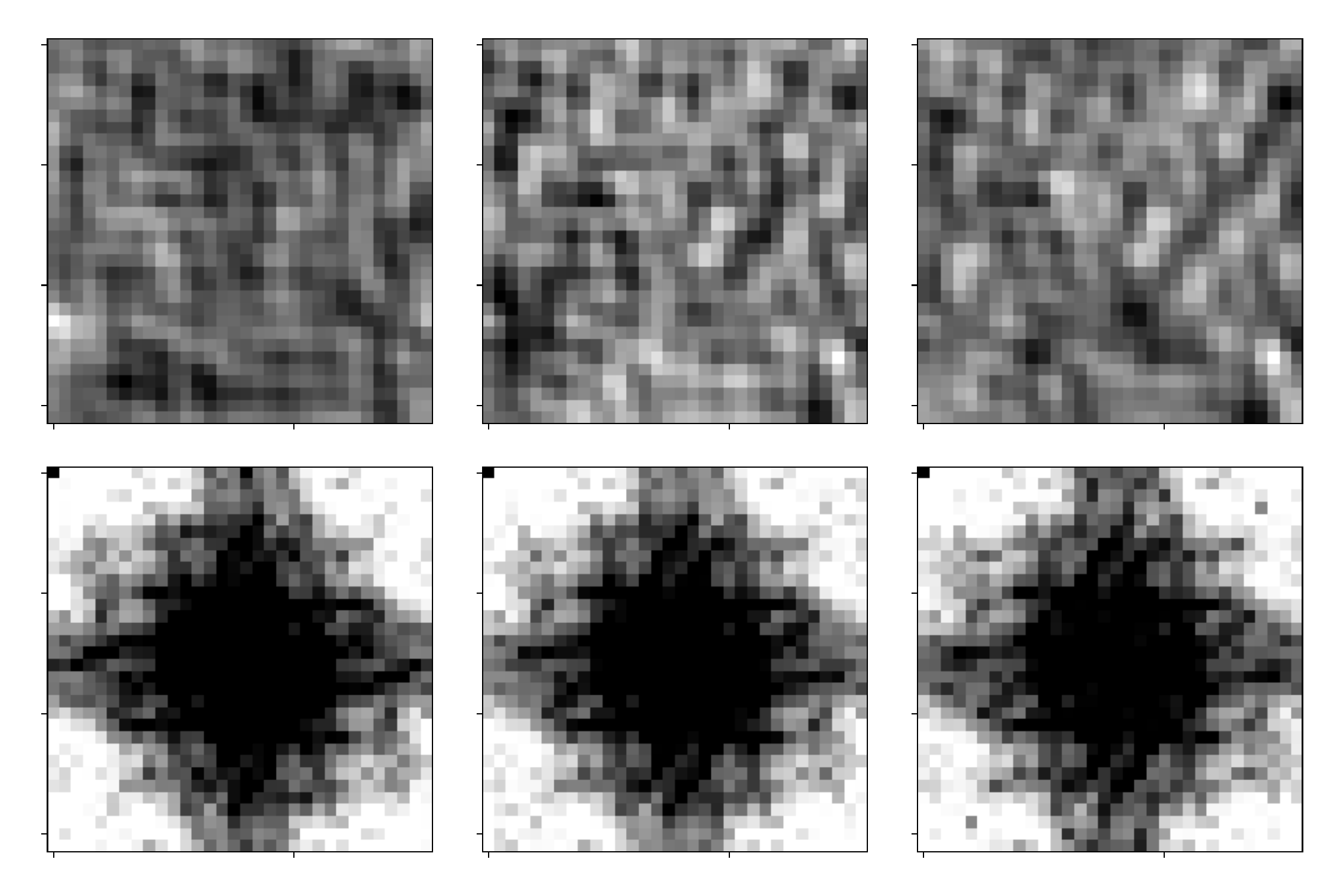}
\caption{\small $\COut = 8$}
\end{subfigure}
\caption{Linear predictors learned by  two layer linear convolutional network on CIFAR-10 task. The sub-figures depict predictors learned by using gradient descent on the  exponential loss for overparameterized networks with $\COut \in \left\{1, 2, 3, 4, 8\right\}$ and kernel size $\KerSize = (8, 8)$. The top row in each sub--figure  is the signal domain representation $\ParVecFn{(\U,\V)}$, and the bottom row is the Fourier domain representation $\hat{\ParVecFn}{(\U,\V)}$. }
\label{fig:outputchannelCIFAR8}
\end{figure}

\subsubsection{Convolutions with zero padding}

We additionally rerun the same setup as in the previous section with zero padding rather than circular padding. Table \ref{tab:outputchannelsCIFARzero} shows $\RkCinCoutHat{\ParVecFn(\U,\V)}{\KerSize}{\COut}{3}$ across different values of $\KerSize$ and $\COut$.  Although convolutions with zero padding goes beyond the scope of our theoretical results, we show that Hypothesis \ref{hyp:outputchannels} nonetheless still holds. 
\begin{table}[ht]
    \centering\small
 \begin{tabular}{c|c | c | c | c}
$\COut$ & $\KerSize = (1,1)$ &   $\KerSize = (3,3)$ &   $\KerSize = (8,8)$  &  $\KerSize = (20,20)$  \\ \hline
1 & 246.04 & 225.64 & 217.14 & 177.07 \\
2 & 245.98 & 191.28 & 186.32 & 152.58   \\
3 & 245.94 & 191.34 & 182.18 & 154.84  \\
4 & 245.98 & 191.20 & 182.42 & 152.48   \\
8 & 245.34 & 191.38 & 180.28 & 150.88  \\
    \end{tabular}
    \caption{$\RkCinCoutHat{\ParVecFn(\U,\V)}{\KerSize}{\COut}{3} = \|\U\|^2+\|\V\|^2$ of the predictor learned by gradient descent on linear convolutional networks \textit{with zero padding} with different number of output channels $\COut$ and kernel sizes $\KerSize$ on the CIFAR-10 task. } 
    \label{tab:outputchannelsCIFARzero}
\end{table}

\subsection{Varying kernel sizes}

While the number of output channels has little influence on the induced regularizer of the learned predictors, we show that the kernel size can have significant impact, which aligns with our theoretical findings.  

\subsubsection{Effect of kernel size on MNIST} 

Our theoretical findings in Section \ref{subsec:tightness} suggest that the induced regularizer interpolates between $\ell_2$ and $\ell_1$ norms in the Fourier domain. The 
$\ell_2$ regularization of $\RkCout{\ParVec}{1}{1}$ does not induce sparse solutions, while the $\ell_1$ regularization of $\RkCout{\ParVec}{\Dim}{1}$ promotes sparsity in the Fourier basis. This would suggest the following: \textit{larger kernel sizes induce sparsity in the frequency domain}. 

\paragraph{Explicit optimal solutions for $\KerSize = (1, 1)$ and $\KerSize = (\Dim, \Dim)$. }

First, to illustrate in the extreme cases of $\KerSize = 1$ and $\KerSize = \Dim$, we explicitly compute the $\RkCoutOp{\KerSize}{\COut}$ margin predictor: 
\[\min_{\w}\RkCout{\w}{\KerSize}{\COut}\st \forall_n y_n\innerprod{\w,\x_n}\ge1\] on the dataset using the closed-form solutions for the induced regularizer in these special cases. In Figure \ref{fig:mnistcvx}, we show resulting optimal solutions for $\KerSize = \Dim$ (a minimum $\ell_1$ solution) and $\KerSize = 1$  (a minimum $\ell_2$ solution). While the solution for $\KerSize = (1, 1)$ exhibits no sparsity in the frequency domain, the solution for $\KerSize = (\Dim, \Dim)$ exhibits significant sparsity.  

The corresponding values of $\RkCoutOp{\KerSize}{\COut}$ are 9.32 for $\KerSize = \Dim$ and 2.10 for $\KerSize = 1$. We note that the induced regularizer do not exactly match those computed on the $\ParVecFn(\Matrix{U}, \Matrix{V})$ from gradient descent---this is because the convergence of these values can be quite slow. Moreover, for $\KerSize = \Dim$, the difference in the predictors likely stems from the minimum $\ell_1$-norm solution being non-unique. We nonetheless show that the qualitative findings apply to gradient descent, despite the fact that the limiting values have not been reached.

\paragraph{Extension to gradient descent.}
Consider networks with one output channel $\COut=1$ and compute $\ParVecFn(\Matrix{U}, \Matrix{V})$ learned by gradient descent for networks with different kernel sizes in Figure \ref{fig:channels}-(a). Notice in the frequency domain plots that the  predictor learned with kernel size $\KerSize = (1,1)$ is not sparse, the predictor learned with $\KerSize = (3,3)$ already starts to exhibit some sparsity, and the linear predictor learned with $\KerSize = (28,28)$ is highly sparse in the frequency domain.

Since sparsity in the frequency domain promotes a patterned structure in the signal domain, we explore the qualitative behavior of large kernel sizes in the signal domain in more depth. To do this, we construct an augmented version of the dataset with $112 \times 112$ dimensional images where the top-left $28 \times 28$ region is the original image, while the remaining space is all $0$s. Figure \ref{fig:mnistaug} shows the  linear predictors learned by running gradient descent on single output channel networks with different kernel sizes. As $\KerSize$ increases, the nonzero region of the predictor becomes larger, eventually encompassing the full $112 \times 112$ dimensional space. For large kernel sizes, we can visually see that the predictors are composed of repetitions of a pattern. This is suggestive of a restricted form of periodic translation invariance, where the shift size aligns with the size of the patterns.

\begin{figure*}
\centering
\includegraphics[scale=0.38]{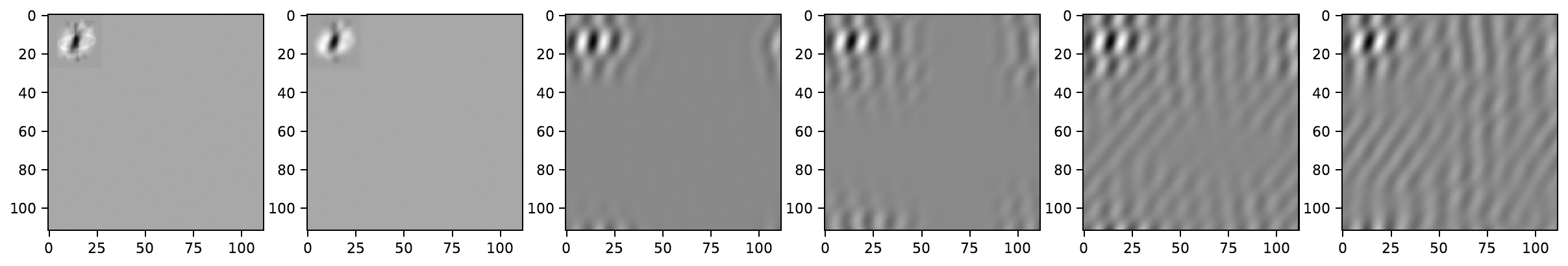}
\caption{Linear predictors  learned by gradient descent on single output channel networks  over an  augmented input space with kernel sizes $\KerSize\in\crl{(1, 1), (3, 3), (27, 27), (45, 45), (65, 65), (84, 84)}$ (left to right). The input images from the MNIST dataset  are augmented by padding with zeros to obtain an image of size $112\times 112$ with the signal present only in the top-left $28\times28$ block.}
\label{fig:mnistaug}
\end{figure*}

\subsubsection{Effect of kernel size on CIFAR-10} 
We now examine the role of kernel size for multi-channel networks on CIFAR-10. Our theoretical findings in Section \ref{sec:mult-input-channel} suggest that the induced regularizer interpolates between the nuclear norm for $\KerSize = 1$ and the $\ell_{2,1}$ norm for $\KerSize = \Dim$. This would again suggest the following: \textit{larger kernel sizes induce sparsity in the frequency domain}. The behavior across input channels, however, is more nuanced. Both of these norms favor similarities across different input channels (with the effect intuitively stronger for $\KerSize = 1$ since the nuclear norm is closely related to rank). We explore both of these effects in the following experiments. 

\paragraph{Explicit optimal solutions for $\KerSize = (\Dim, \Dim)$. }First, to illustrate in the extreme case of $\KerSize = \Dim$, we explicitly compute the $\RkCoutOp{\KerSize}{\COut}$ margin predictor: $\min_{\w}\RkCout{\w}{\KerSize}{\COut}\st \forall_n y_n\innerprod{\w,\x_n}\ge1$ on the dataset using the closed-form solutions for the induced regularizer in these special cases. In Figure \ref{fig:cifarcvx}, we show resulting optimal solutions for $\KerSize = \Dim$ (a minimum $\ell_{2,1}$ solution) along with the optimal $\ell_{1,1}$ solution. We visually see that the $\ell_{2,1}$ solution favors similarity across input channels at the expense of greater sparsity in the frequency domain.

The corresponding values of $\RkCoutOp{\Dim}{\COut}$ for the minimum $\ell_{2,1}$ norm solution is 82.85 for $\KerSize = (\Dim, \Dim)$.) As for the single-input channel case, we note that the induced regularizer do not exactly match the value of $114.85$ of the induced regularizer computed on the $\ParVecFn(\Matrix{U}, \Matrix{V})$ from gradient descent---this is because the convergence of the induced regularizer is known to be slow. We nonetheless show that the qualitative findings apply to gradient descent, despite the fact that the limiting value of induced regularizer has not been reached.

\begin{figure}
\centering
\begin{subfigure}[b]{0.49 \textwidth}
\centering
 \includegraphics[width=\textwidth]{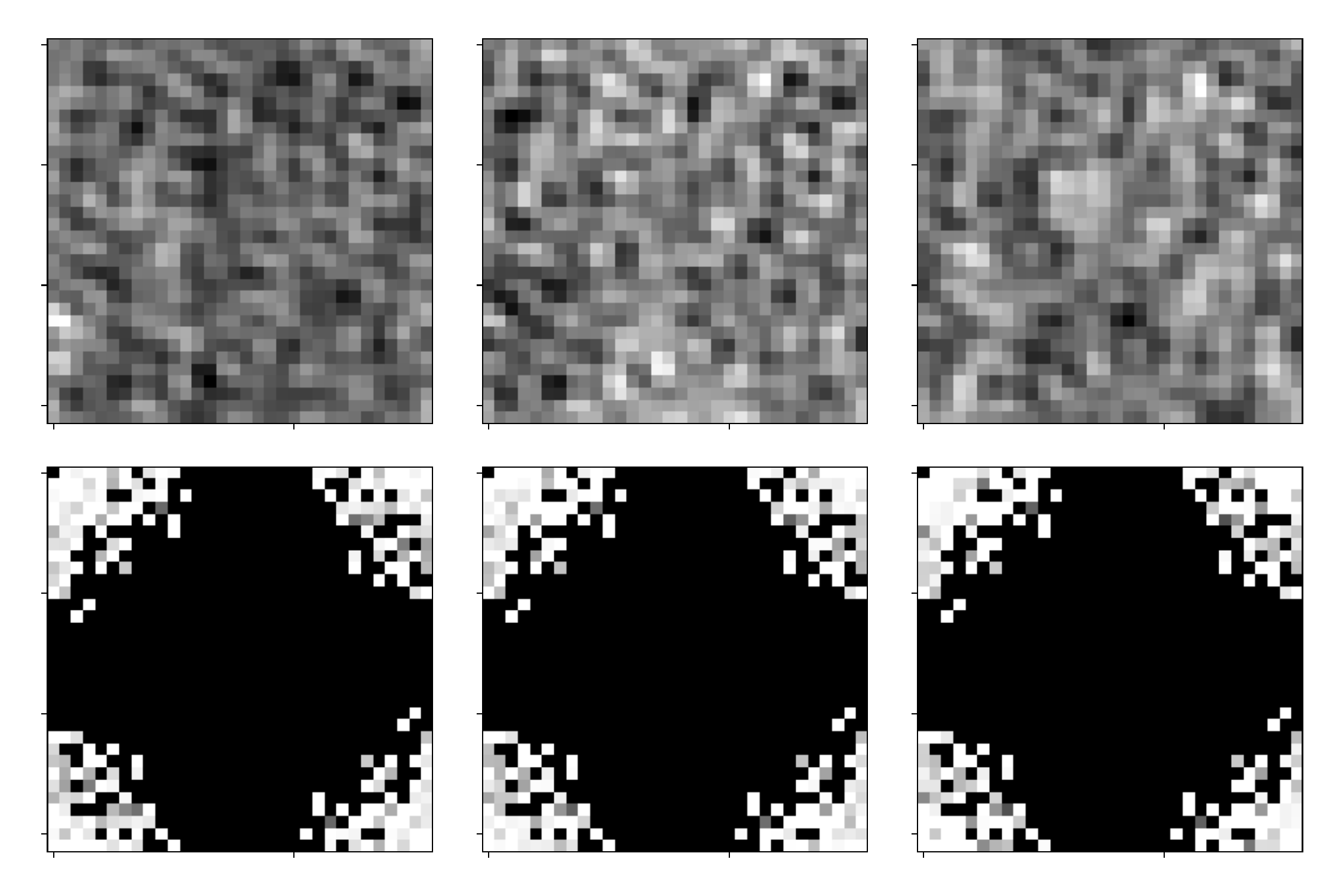}
\caption{\small $\ell_{2,1}$}
\end{subfigure}
\begin{subfigure}[b]{0.49  \textwidth}
\centering
 \includegraphics[width=\textwidth]{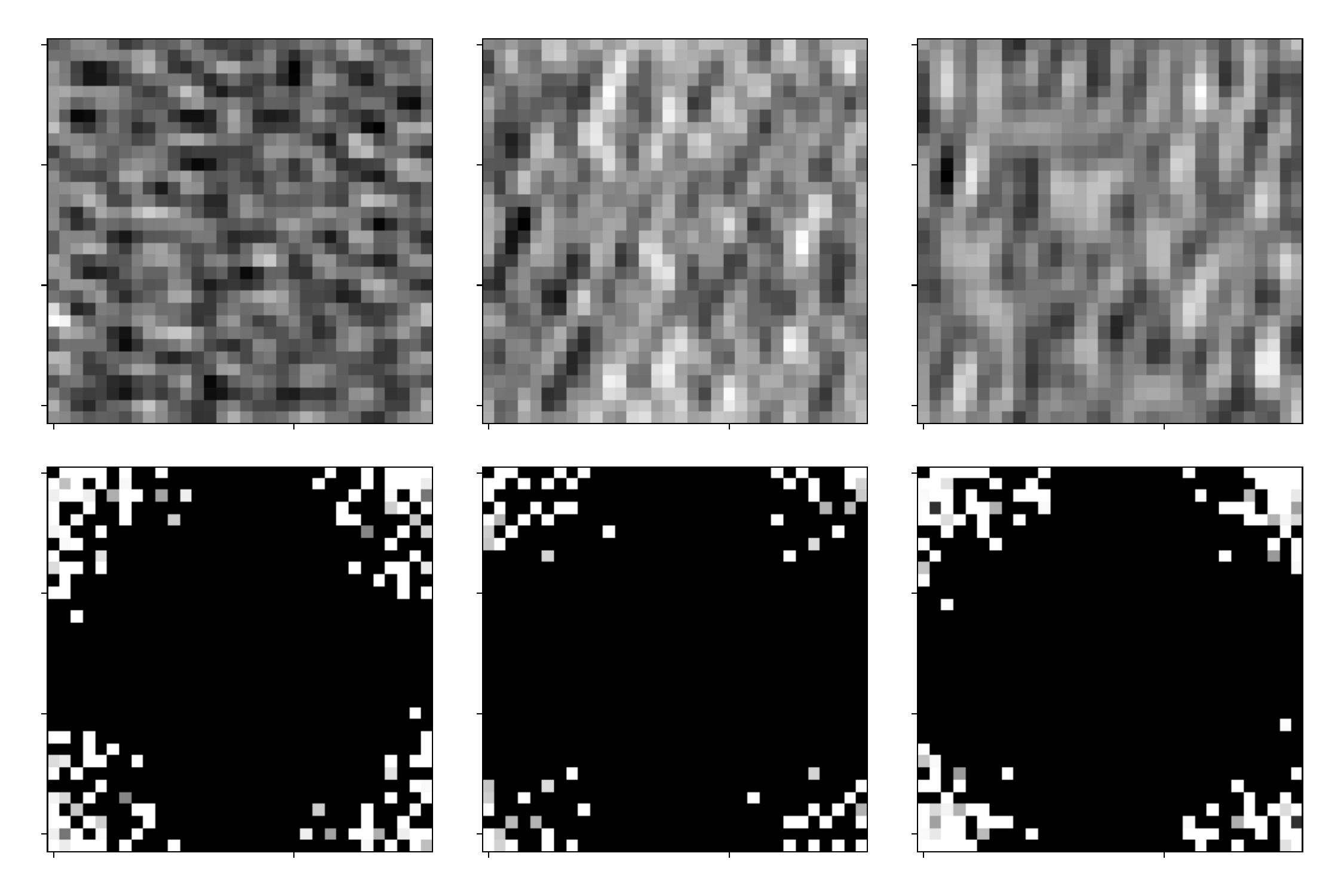}
\caption{\small $\ell_{1}$}
\end{subfigure}
\begin{subfigure}[b]{0.49  \textwidth}
\centering
\includegraphics[width=\textwidth]{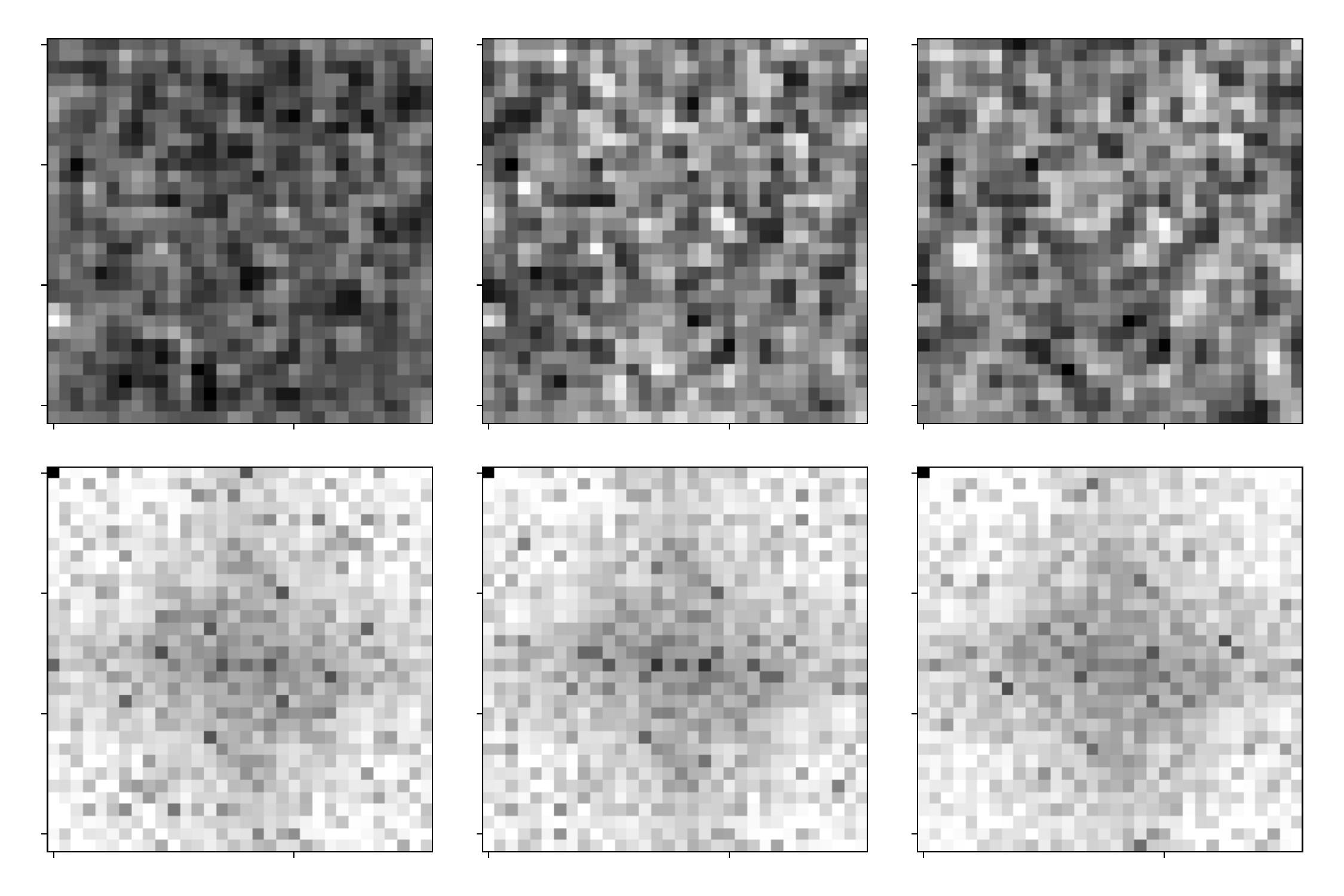}
\caption{\small $\ell_{2}$}
\end{subfigure}
\caption{Explicit $\ell_{2,1}$, $\ell_{1}$ and $\ell_{2}$ margin predictors on sampled CIFAR-10 dataset. The top row in each sub--figure  is the signal domain representation $\ParVecFn{(\U,\V)}$, and the bottom row is the Fourier domain representation $\hat{\ParVecFn}{(\U,\V)}$}
\label{fig:cifarcvx}
\end{figure}

\paragraph{Extension to gradient descent.}

\begin{figure}
\centering
\begin{subfigure}[b]{0.49 \textwidth}
\centering
\includegraphics[width=\textwidth]{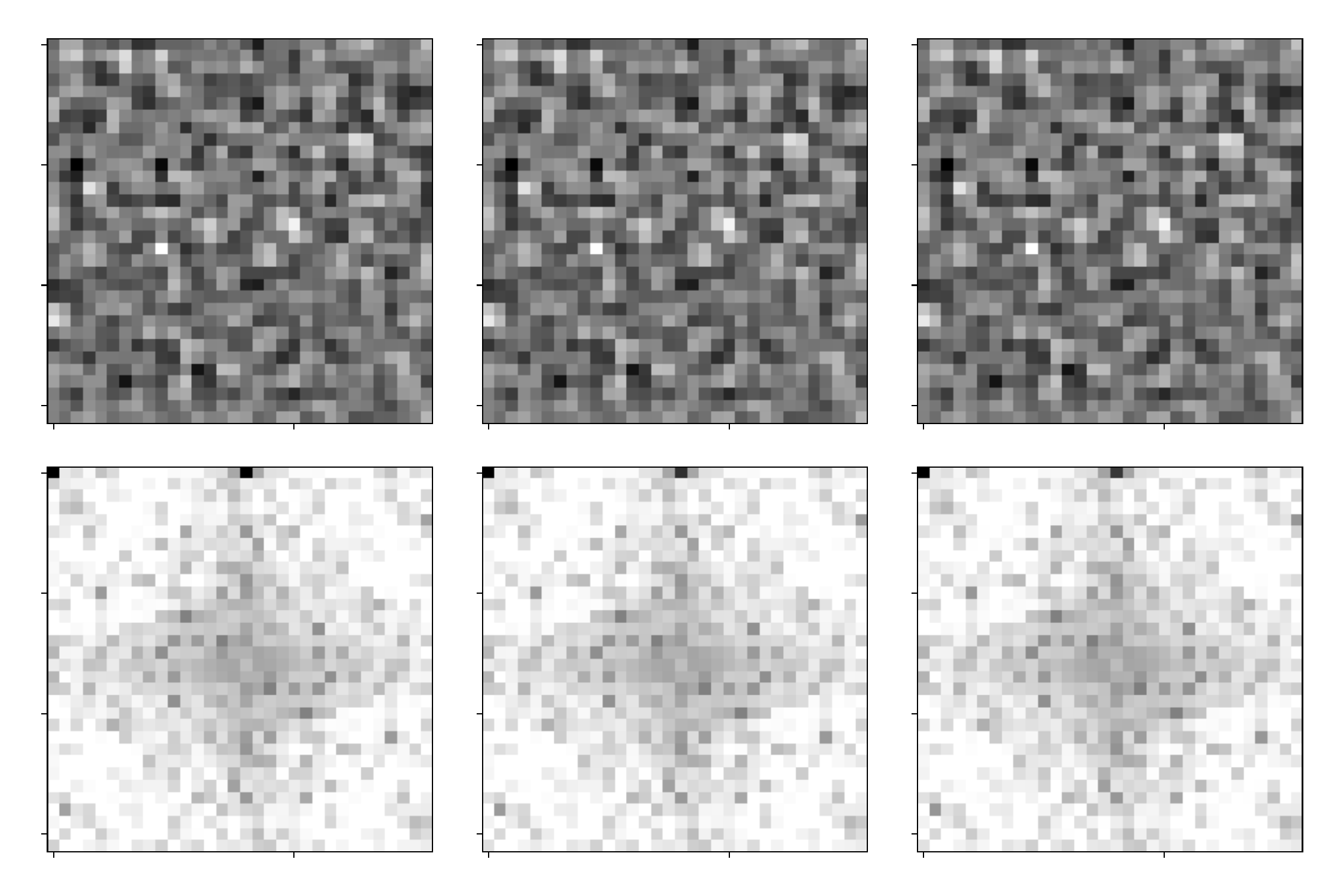}
\caption{\small $\KerSize = (1,1)$}
\end{subfigure}
\begin{subfigure}[b]{0.49  \textwidth}
\centering
\includegraphics[width=\textwidth]{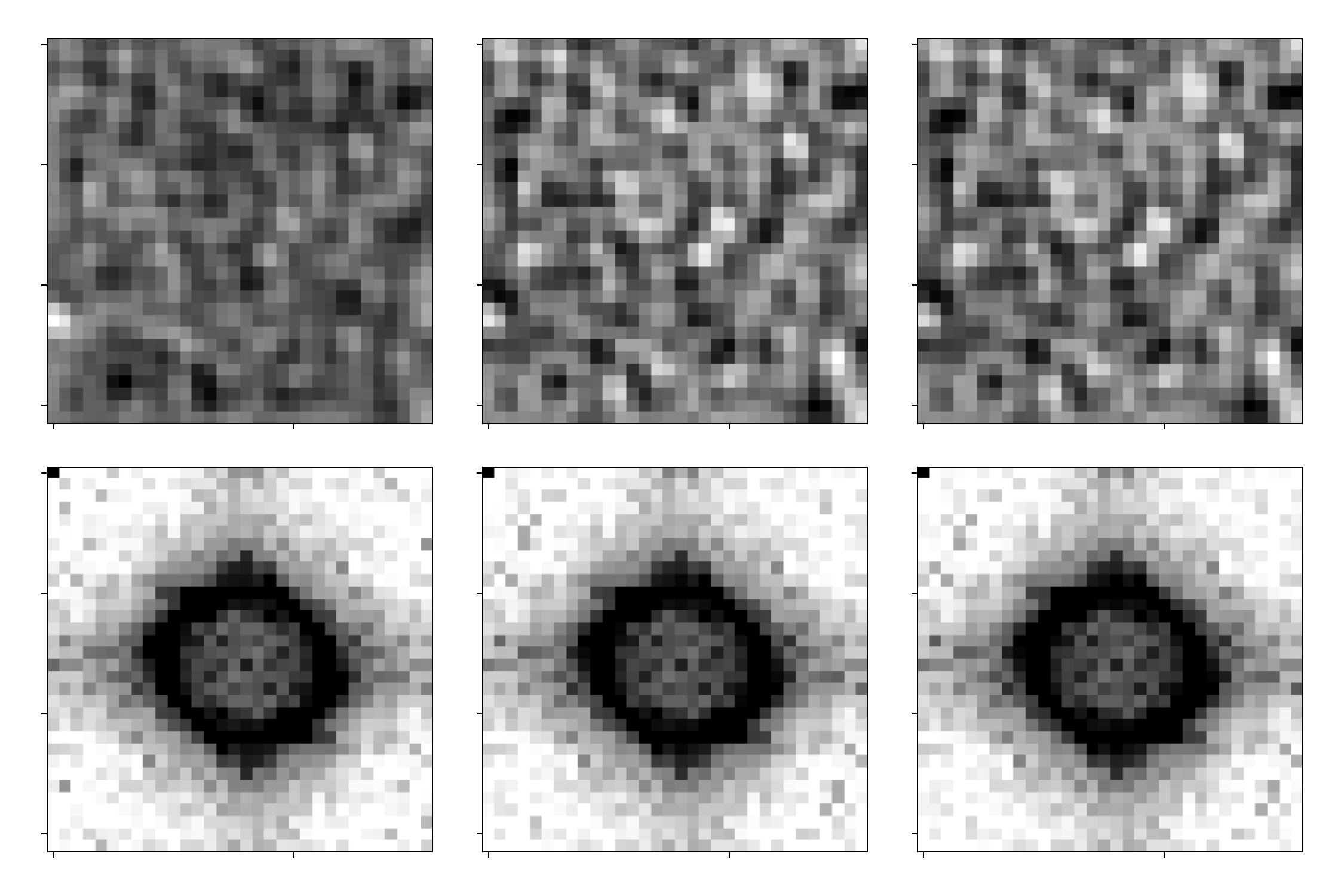}
\caption{\small $\KerSize = (3,3)$}
\end{subfigure}
\begin{subfigure}[b]{0.49  \textwidth}
\centering
\includegraphics[width=\textwidth]{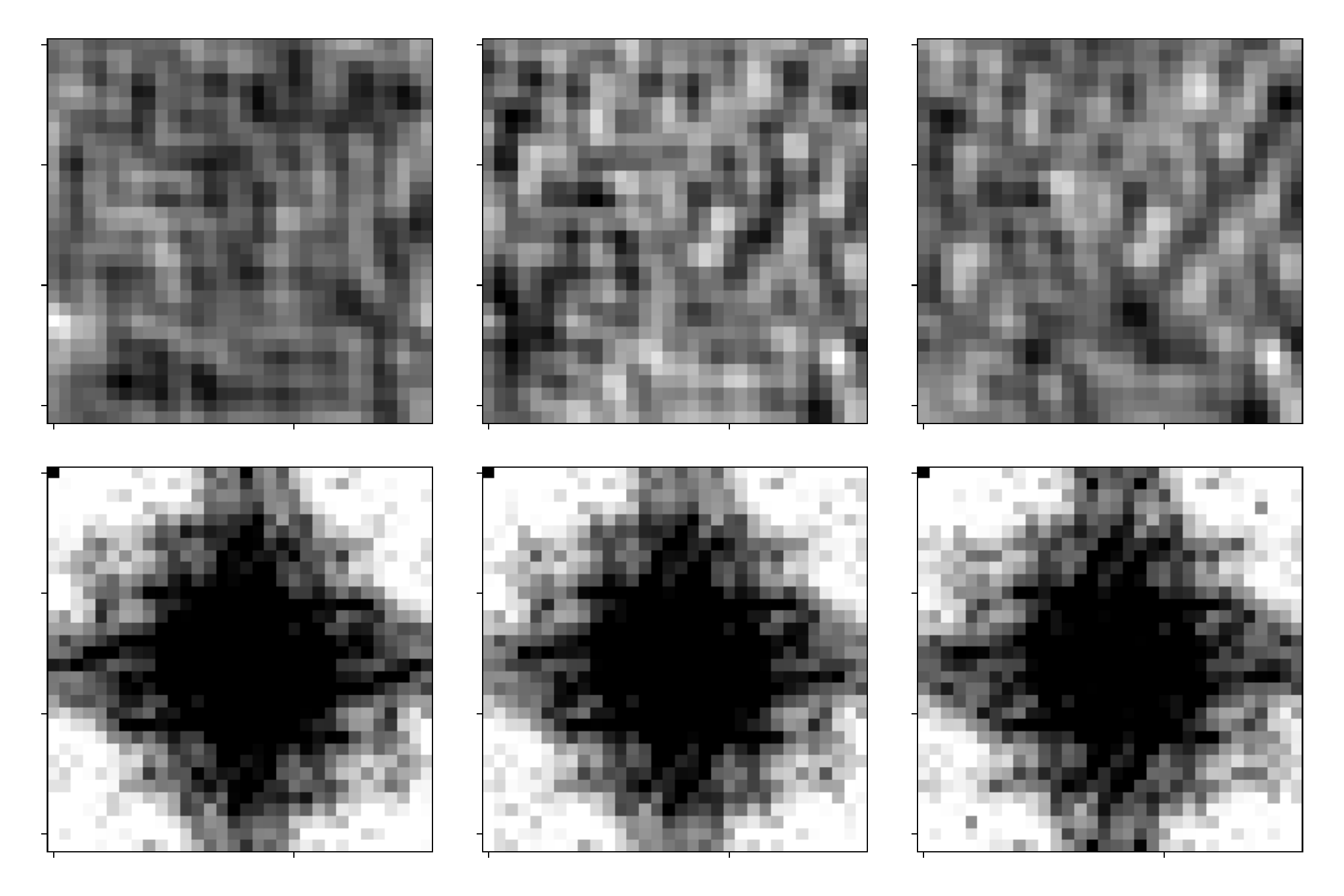}
\caption{\small $\KerSize = (8,8)$}
\end{subfigure}
\begin{subfigure}[b]{0.49  \textwidth}
\centering
\includegraphics[width=\textwidth]{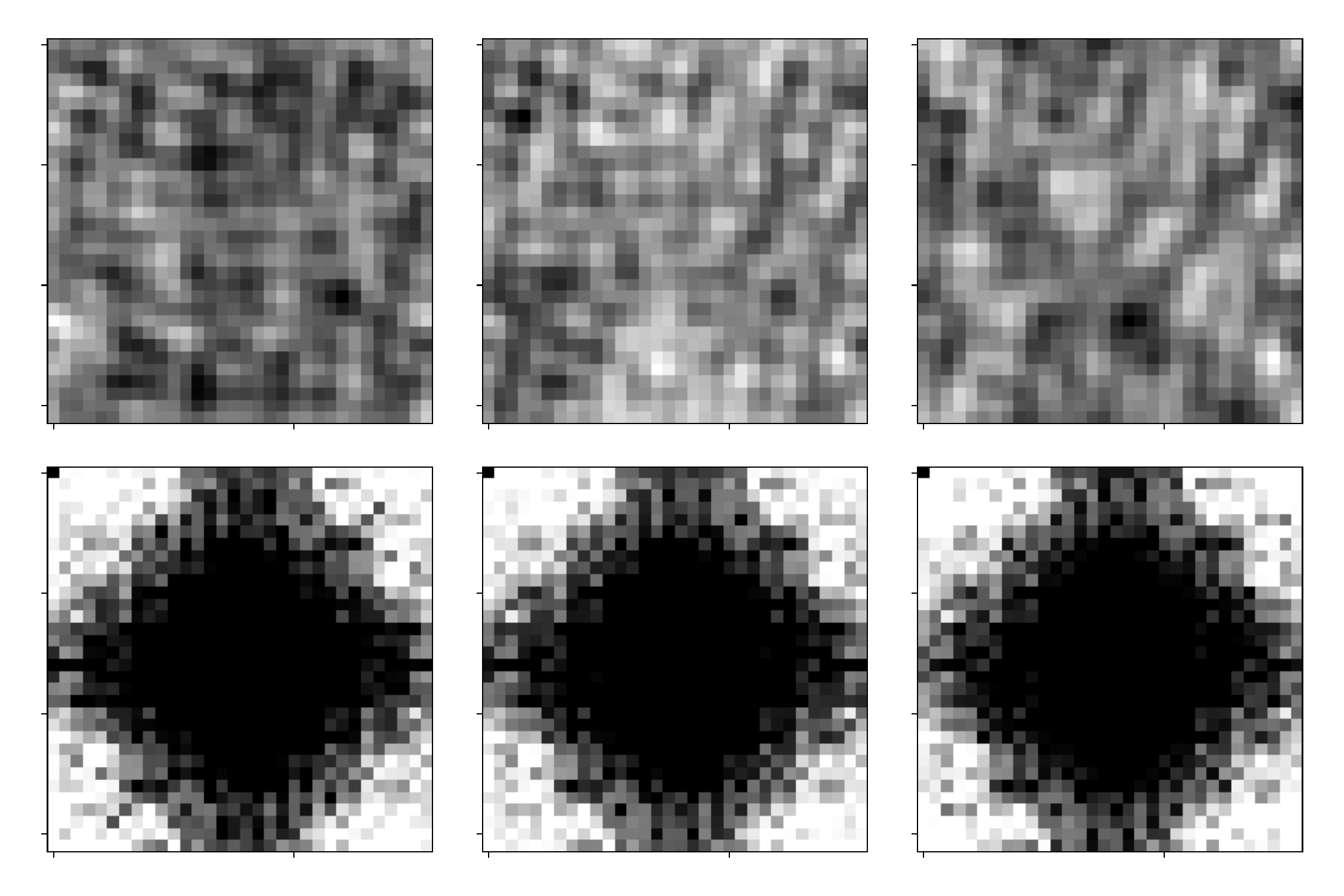}
\caption{\small $\KerSize = (20, 20)$}
\end{subfigure}
\begin{subfigure}[b]{ \textwidth}
\centering
\includegraphics[width=0.49 \textwidth]{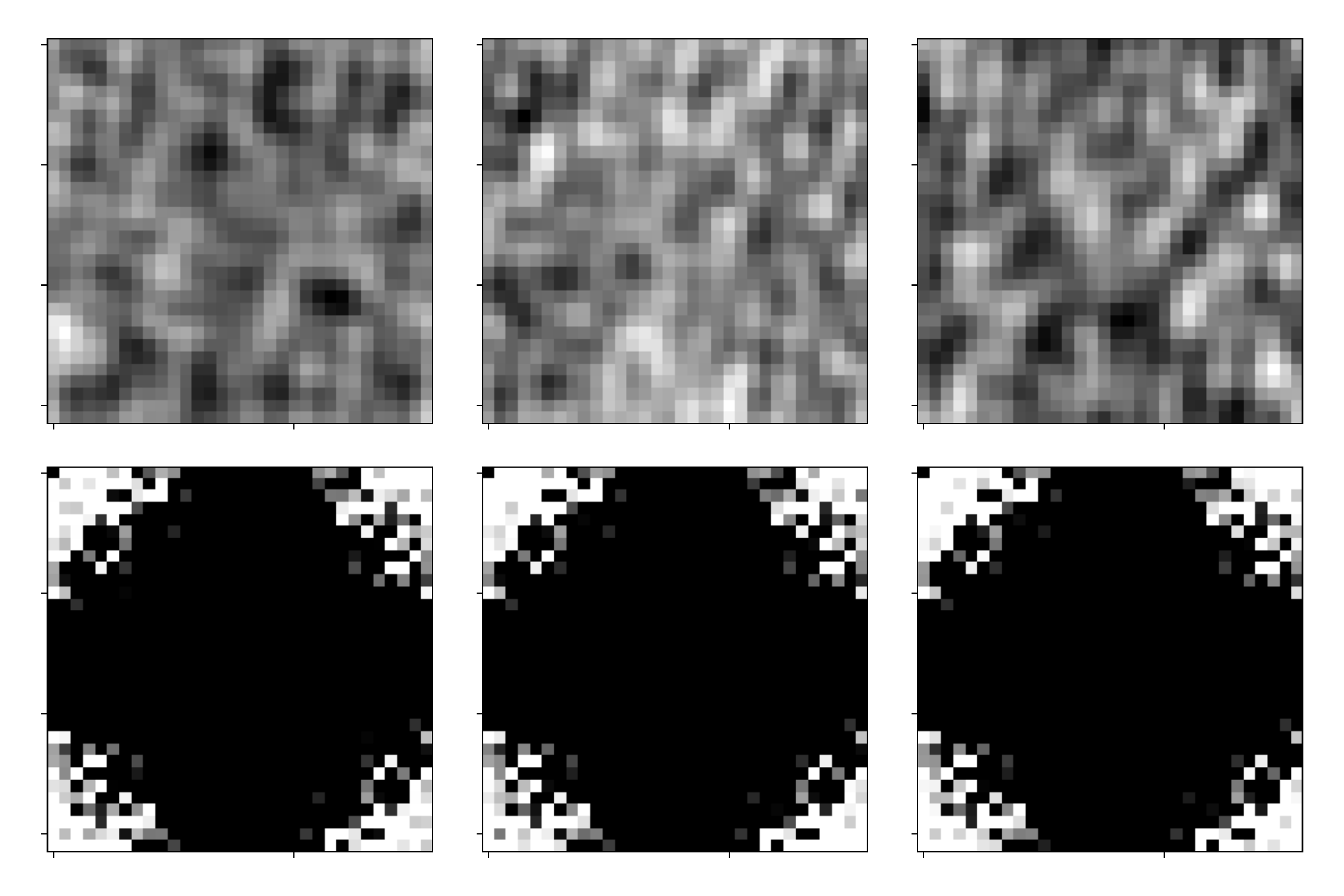}
\caption{\small $\KerSize = (32, 32)$}
\end{subfigure}
\caption{Linear predictors learned by  two layer linear convolutional network on CIFAR-10 task. The sub-figures depict predictors learned by using gradient descent on the  exponential loss for overparameterized networks with $\COut=3$ and  kernel sizes $\KerSize\in \left\{(1, 1), (3, 3), (8, 8), (20, 20), (32, 32)\right\}$. The top row in each sub--figure  is the signal domain representation $\ParVecFn{(\U,\V)}$, and the bottom row is the Fourier domain representation $\hat{\ParVecFn}{(\U,\V)}$. }
\label{fig:channelscifar}
\end{figure}

Consider networks with one output channel $\COut=3$ and compute $\ParVecFn(\Matrix{U}, \Matrix{V})$ learned by gradient descent for networks with different kernel sizes in Figure \ref{fig:channelscifar}. First, the higher kernel does indeed favor sparsity in the frequency domain, as in the single-input channel case. Nontrivial sparse structures can be observed even for $\KerSize = 3$. Next, we discuss the predictor across different input channels. Let's first focus on the extreme case of $\KerSize = 1$. For sake of comparison, we show the explicit minimum $\ell_{2}$ (Frobenius) predictor in Figure \ref{fig:cifarcvx} (note that this is \textit{not} the $\RkCoutOp{1}{\COut}$ margin predictor because the induced regularizer is related to the nuclear norm, not the $\ell_2$ norm). As expected, we see that the learned predictor has a greater degree of similarity across channels than the $\ell_2$ predictor. For other kernel sizes, Figure \ref{fig:channelscifar} also shows some degree of similarity across input channels, although the differences appear to grow as $\KerSize$ becomes larger. 

\end{document}